%% file: sequential_TT.tex
\RequirePackage{amsmath}
\RequirePackage{fix-cm}
\documentclass[smallextended,final]{svjour3}

\usepackage{mathptmx} 
\DeclareMathAlphabet{\mathcal}{OMS}{cmsy}{m}{n}

\usepackage{graphicx}
\usepackage{epstopdf}
\usepackage{breakcites}
\usepackage{enumitem}
\usepackage{amsfonts}
\usepackage{subdepth}

\usepackage{algorithm}
\usepackage{algpseudocode}
\ifpdf
\DeclareGraphicsExtensions{.pdf,.png,.jpg}
\else
\DeclareGraphicsExtensions{.eps}
\fi

\newtheorem{assumption}[theorem]{Assumption}
\usepackage[colorlinks=true, citecolor=blue, filecolor=black,  urlcolor=blue, hypertexnames=false]{hyperref}

\AtBeginDocument{%
  \setlength{\textwidth}{1.2\linewidth}
  \setlength{\linewidth}{\textwidth}
  \setlength{\hsize}{\textwidth}
  \setlength{\columnwidth}{\textwidth}
  \setlength{\textheight}{1.05\textheight}
  \setlength{\oddsidemargin}{\dimexpr(\paperwidth-\textwidth)/2-1in}%
  \setlength{\evensidemargin}{\oddsidemargin}%
  \setlength{\topmargin}{%
    \dimexpr(\paperheight-\textheight)/2-\headheight-\headsep-1.2in}%
}

\smartqed

\DeclareMathOperator*{\argmin}{arg\,min}

\newcommand{\N}{\mathbb{N}}
\newcommand{\R}{\mathbb{R}}

\newcommand{\bs}{\boldsymbol}
\newcommand{\bpsi}{\Psi}
\newcommand{\bphi}{\Phi}
\newcommand{\bx}{x}
\newcommand{\bX}{X}
\newcommand{\by}{y}
\newcommand{\bY}{Y}
\newcommand{\bze}{\zeta}
\newcommand{\bZ}{Z}
\newcommand{\bu}{u}
\newcommand{\bU}{U}

\newcommand{\mH}{\mathsf H}
\newcommand{\mF}{\mathsf F}
\newcommand{\mG}{\mathsf G}
\newcommand{\sqbasis}{\mathsf P}
\newcommand{\mM}{\mathsf M}
\newcommand{\mA}{\mathsf A}
\newcommand{\mB}{\mathsf B}
\newcommand{\mC}{\mathsf C}
\newcommand{\mD}{\mathsf D}
\newcommand{\mQ}{\mathsf Q}
\newcommand{\mR}{\mathsf R}
\newcommand{\mI}{\mathsf I}
\newcommand{\chol}{\mathsf L}
\newcommand{\tA}{\bs{\mathsf A}}
\newcommand{\tB}{\bs{\mathsf B}}
\newcommand{\tC}{\bs{\mathsf C}}
\newcommand{\tH}{\bs{\mathsf H}}

\newcommand{\mtar}{\nu_{\pi}}
\newcommand{\mref}{\mu }
\newcommand{\muni}{\mu_{\rm uni}}

\newcommand{\ptar}{\pi}
\newcommand{\pref}{f_{\bU}}
\newcommand{\puni}{f_{\rm uni}}

\newcommand{\km}{k-1}

\renewcommand{\ldots}{\makebox[1em][c]{.\hfil.\hfil.}}

\newcommand{\lowersup}[1]{ { \raisebox{-2pt}{$\scriptstyle #1$} } }

\newcommand{\lpx}[1]{L_\lambda^{#1}(\mathcal{X})}
\newcommand{\lpu}[1]{L_\omega^{#1}(\mathcal{U})}
\newcommand{\lpnormx}[2]{ \big\| {#1} \big\|_{\lpx{#2}} }
\newcommand{\lpnormu}[2]{ \big\| {#1} \big\|_{\lpu{#2}} }
\newcommand{\lpintx}[2]{ \left(\int_\mathcal{X} \left| #1\right|^{#2} \lambda(\bx) d\bx \right)^\frac1{#2} }
\newcommand{\lpintu}[2]{ \left(\int_\mathcal{U} \left| #1\right|^{#2} \omega(\bu) d\bu \right)^\frac1{#2} }

\newcommand{\chisq}{\chi^\lowersup{2}}
\usepackage{tikz}
\usetikzlibrary{positioning,shapes,calc,3d}
\usepackage{pgfplots}
\usepgfplotslibrary{fillbetween}

\pgfplotsset{%
      every axis/.append style={width=.45\textwidth,
                                axis x line=bottom, axis y line=left,
                                x axis line style={very thick,->}, y axis line style={very thick,->},
                                tick align=inside, tick style={thick},
                                every x tick label/.style={font=\small},
                                every y tick label/.style={font=\small},
                                },
      every axis legend/.append style={
                                legend columns=1,
                                font=\small,
                                draw=none,
                                fill=white,
                                },
      every axis x label/.style={at={(0.5,-0.12)},below,fill=none,fill opacity=1,text opacity=1},
      every axis y label/.style={at={(-0.14,0.5)},fill=none,fill opacity=1,text opacity=1,rotate=90},
      }

\allowdisplaybreaks

\begin{document}

\title{Deep composition of tensor-trains using squared inverse Rosenblatt transports}

\titlerunning{Deep inverse Rosenblatt transport} 

\author{Tiangang Cui \and Sergey Dolgov 
%
}

\institute {
T. Cui 
\at School of Mathematics, Monash University, Victoria 3800, Australia  \\ \email{tiangang.cui@monash.edu} \\ ORCID: \url{https://orcid.org/0000-0002-4840-8545}
\and
S. Dolgov (corresponding author)
\at Department of Mathematical Sciences, University of Bath, Bath, BA2 7AY, UK \\ \email{s.dolgov@bath.ac.uk} \\ ORCID: \url{https://orcid.org/0000-0002-1647-4214}
}

\date{\;}

\maketitle

\begin{abstract}
Characterising intractable high-dimensional random variables is one of the fundamental challenges in stochastic computation. The recent surge of transport maps offers a mathematical foundation and new insights for tackling this challenge by coupling intractable random variables with tractable reference random variables.
This paper generalises the functional tensor-train approximation of the inverse Rosenblatt transport recently developed by Dolgov et al. (Stat Comput 30:603--625, 2020) to a wide class of high-dimensional non-negative functions, such as unnormalised probability density functions.
First, we extend the inverse Rosenblatt transform to enable the transport to general reference measures other than the uniform measure.
We develop an efficient procedure to compute this transport from a squared tensor-train decomposition which preserves the monotonicity.
More crucially, we integrate the proposed order-preserving functional tensor-train transport into a nested variable transformation framework inspired by the layered structure of deep neural networks. The resulting deep inverse Rosenblatt transport significantly expands the capability of tensor approximations and transport maps to random variables with complicated nonlinear interactions and concentrated density functions.
We demonstrate the efficiency of the proposed approach on a range of applications in statistical learning and uncertainty quantification, including parameter estimation for dynamical systems and inverse problems constrained by partial differential equations.

\end{abstract}

\input{sec1_introduction}

\input{sec2_background}
\input{sec3_methods}
\input{sec4_numerics}

\input{appendix}

\begin{acknowledgements}
The authors would like to thank Y. Marzouk and R. Scheichl for for many insightful discussions. T. Cui acknowledges support from the Australian Research Council, under grant number CE140100049. S. Dolgov acknowledges support from the International Visitor Program of Sydney Mathematical Research Institute,
and from the EPSRC New Investigator Award EP/T031255/1.
\end{acknowledgements}

\providecommand{\bysame}{\leavevmode\hbox to3em{\hrulefill}\thinspace}
\providecommand{\MR}{\relax\ifhmode\unskip\space\fi MR }
\providecommand{\MRhref}[2]{%
  \href{http://www.ams.org/mathscinet-getitem?mr=#1}{#2}
}
\providecommand{\href}[2]{#2}

\end{document}

%% file: sec1_introduction.tex

\section{Introduction}

Exploration of high-dimensional probability distributions is a fundamental task in statistical physics, machine learning, uncertainty quantification, econometrics, and beyond.
In many practical scenarios, high-dimensional random variables of interest follow \emph{intractable} probability measures that exhibit nonlinear interactions and concentrate in some sub-manifolds. This way, one cannot directly simulate the random variables of interest but may be able to evaluate the unnormalised density function pointwise.

Suppose we have an intractable target probability measure $\mtar$ with the unnormalised density function $\pi(\bx)$ over a parameter space $\mathcal{X} \subseteq \R^d$, for example, the posterior measure in a Bayesian inference problem.
Various approaches have been proposed to characterise $\mtar$ using some reference probability measure $\mref$ defined over $\mathcal{U} \subseteq \R^d$, where independent and identically distributed random variables can be drawn from, e.g., a uniform or a Gaussian.
For example, Markov chain Monte Carlo (MCMC) methods \cite{MCMC:Liu_2001,robert2013monte} generate a Markov chain of random variables converging to $\mtar$ using $\mref$ as the proposal; and importance sampling and/or sequential Monte Carlo \cite{MCMC:KBJ_2014,mcbook} characterise  $\mtar$ using weighted samples drawn from $\mref$.
The recently developed \emph{transport map} idea, e.g., \cite{bigoni2019greedy,dafs-tt-bayes-2019,marzouk2016sampling,el2012bayesian,parno2018transport}, offers new insights for this task by identifying a measurable mapping, $T: \mathcal{U} \mapsto \mathcal{X}$, such that the pushforward of $\mref$, denoted by $T_\sharp\, \mref$, is a close approximation to $\mtar$.
Then, the mapping $T$ can be used to either accelerate classical sampling methods such as MCMC or to improve the efficiency of importance sampling.
In this work, we generalise the tensor-train (TT)  approach of \cite{dafs-tt-bayes-2019} to offer an order-preserving and multi-layered construction of transport maps that is suitable for high-dimensional random variables with nonlinear interactions and concentrated density functions.

\subsection{Outline and contributions}
The TT-based construction of~\cite{dafs-tt-bayes-2019} realises the mapping $T$ via a separable TT decomposition~\cite{oseledets2011tensor} of the target density function.
Since the separable tensor decomposition enables the marginalisation of the target density at a computational cost scaling linearly in the dimension of the random variables, it offers a computationally viable way to approximating marginal and conditional density functions of the target measure.
In turn, the cumulative distribution functions (CDFs) of the marginals and conditionals define the Rosenblatt transport\footnote{This is also referred to as the Knothe--Rosenblatt rearrangement. It was independently proposed by Rosenblatt \cite{rosenblatt1952remarks} for statistical purposes and by Knothe \cite{knothe1957contributions} for proving the isoperimetric inequality. The setup of the TT-based approach closely follows the work of Rosenblatt.} that can couple the target measure with the uniform reference measure.
Section \ref{sec:background} presents the relevant background of the Rosenblatt transport, the functional form of the TT decomposition of multivariate functions~\cite{bigoni2016spectral,gorodetsky2019continuous,griebel2019analysis,hackbusch2012tensor}, and the TT-based construction of the inverse Rosenblatt transport.

The TT-based construction faces several challenges. First, the TT decomposition of the non-negative target density function often cannot preserve the non-negativity after rank truncations. 
This way, the resulting Rosenblatt transport may not preserve the monotonicity. 
Second, TT decomposition works best when the correlations between random variables are \emph{local}, i.e., the correlation decays with the distance between the indices of the variables. In the extreme case of independent random variables, the joint density factorises into the product of marginal densities.
However, high-dimensional random variables of interest often have concentrated density functions and exhibit complicated nonlinear interactions. In such cases, one may need a TT with high ranks to approximate the target probability density with sufficient accuracy, which in turn requires a rather large number of target density evaluations during the TT construction.

In Section \ref{sec:sirt}, we overcome the first challenge by proposing a new construction of inverse Rosenblatt transport by approximating the squared root of the target density in the TT format, followed by constructing the marginal and conditional densities from the square of the TT approximation.
The resulting \emph{squared inverse Rosenblatt transport} (SIRT) is order- and smoothness-preserving.
In addition, utilising the squared structure of the approximation, we also establish error bounds of SIRT in terms of various statistical divergences within the $f$-divergence family. 
These bounds are useful for bounding the statistical efficiency of posterior characterisation algorithms such as MCMC and importance sampling.

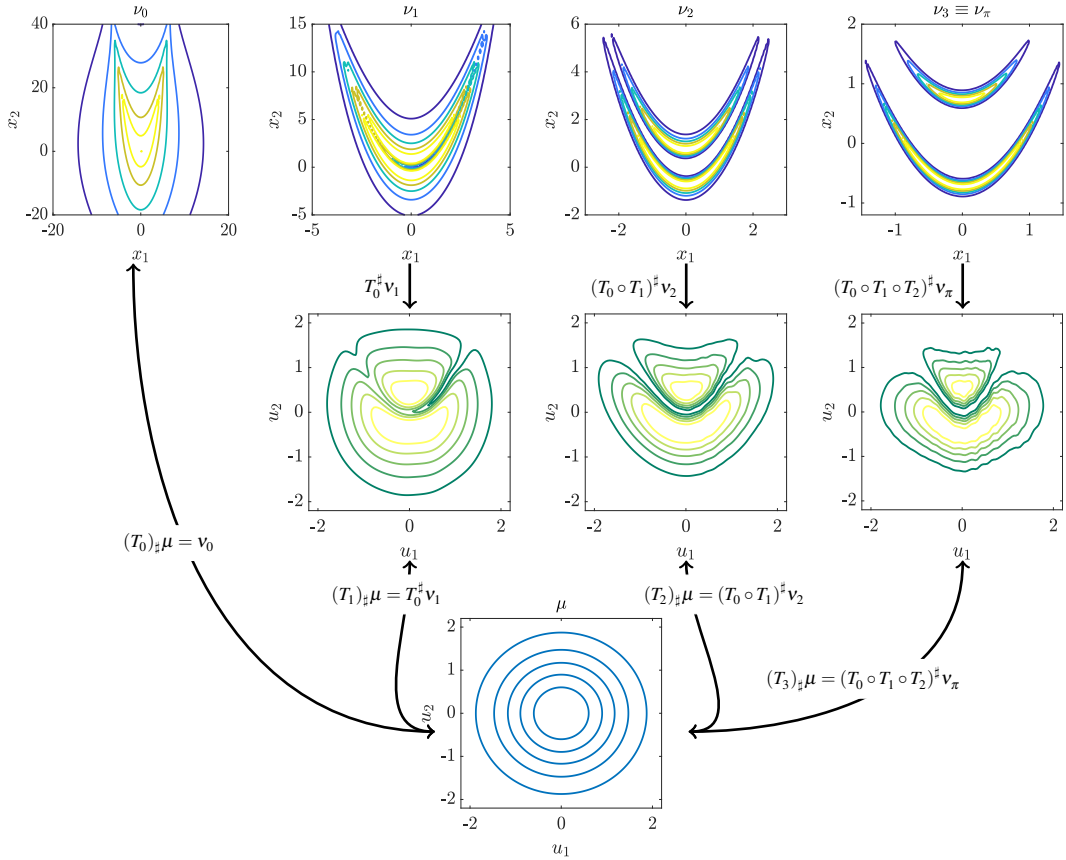
\begin{figure}[ht]
\centering
\begin{tikzpicture}
\node[] (p0) {\includegraphics[width=0.22\linewidth,height=0.22\linewidth,trim=2em 1.5em 0em 3em]{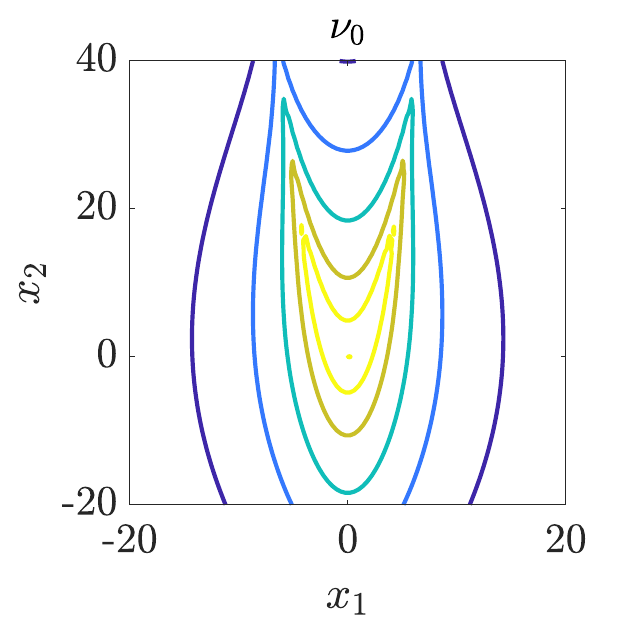}};

\node[anchor=west] (p1) at ($(p0.east)+(0.3,0)$) {\includegraphics[width=0.22\linewidth,height=0.22\linewidth,trim=4em 1.5em 1em 3em]{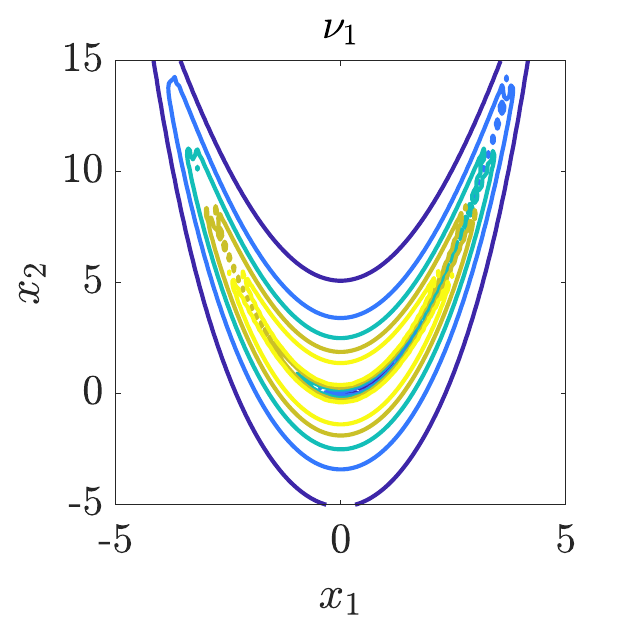}};

\node[anchor=north] (r1) at ($(p1.south)-(0,.6)$) {\includegraphics[width=0.22\linewidth,height=0.22\linewidth,trim=4em 1.5em 1em 4em]{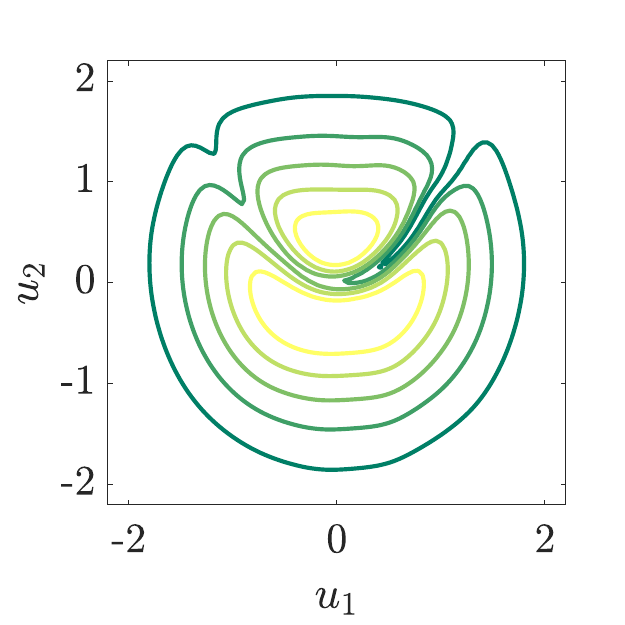}};

\node[anchor=west] (p2) at ($(p1.east)+(0.3,0)$) {\includegraphics[width=0.22\linewidth,height=0.22\linewidth,trim=4em 1.5em 1em 3em]{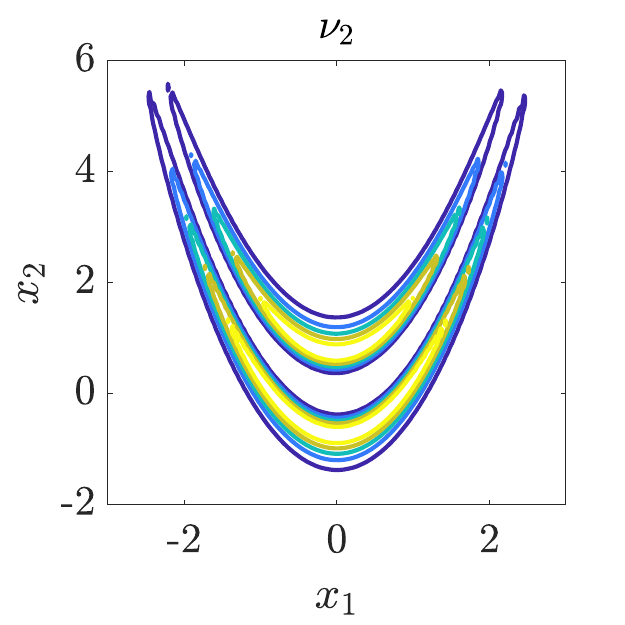}};
\node[anchor=north] (r2) at ($(p2.south)-(0,.6)$) {\includegraphics[width=0.22\linewidth,height=0.22\linewidth,trim=4em 1.5em 1em 4em]{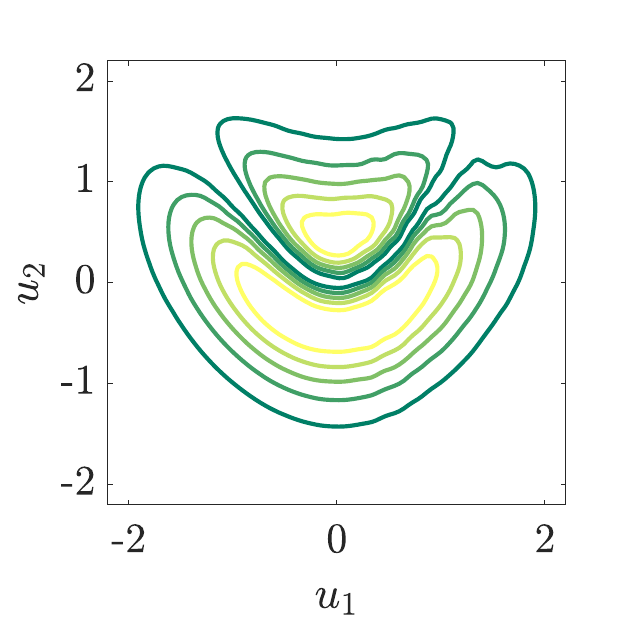}};

\node[anchor=west] (p3) at ($(p2.east)+(0.3,0)$) {\includegraphics[width=0.22\linewidth,height=0.22\linewidth,trim=4em 1.5em 1em 3em]{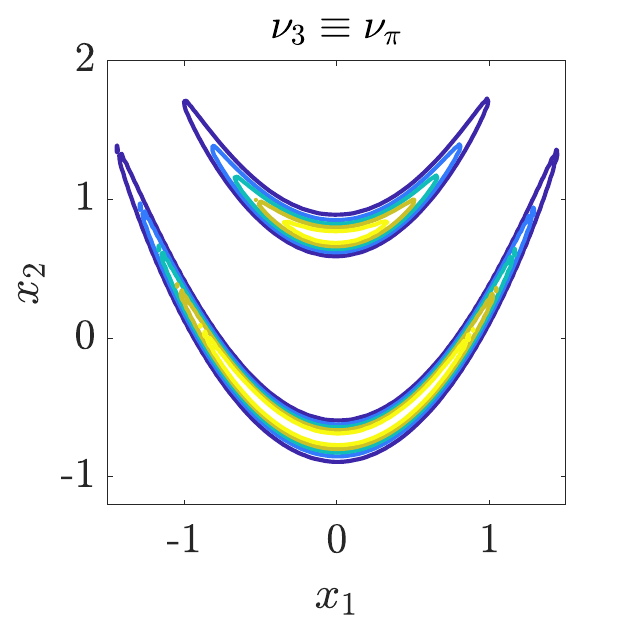}};
\node[anchor=north] (r3) at ($(p3.south)-(0,.6)$) {\includegraphics[width=0.22\linewidth,height=0.22\linewidth,trim=4em 1.5em 1em 4em]{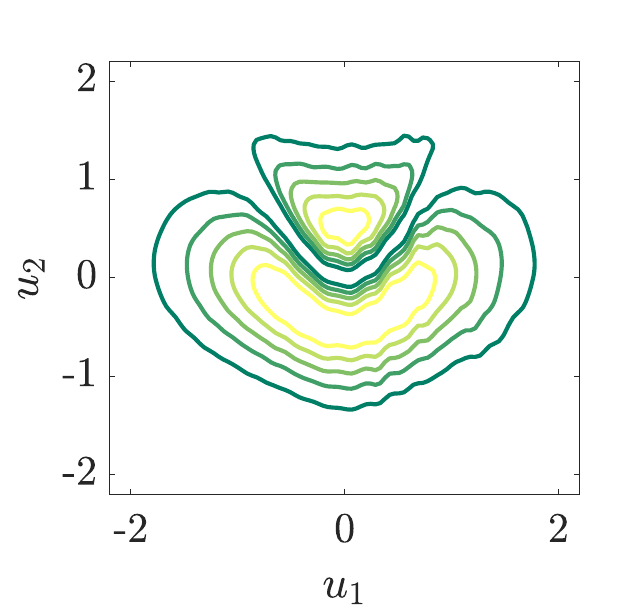}};

\node[anchor=north] (r) at ($(r1.south)+(2,-0.6)$) {\includegraphics[width=0.22\linewidth,height=0.22\linewidth,trim=4em 1.5em 1em 3em]{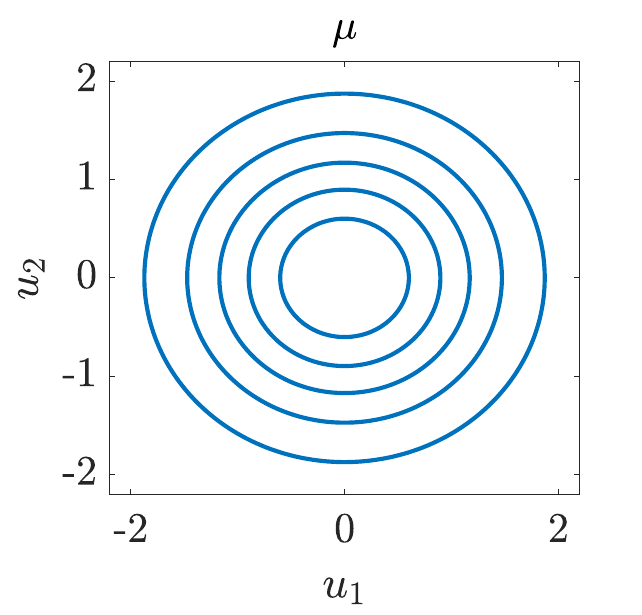}};

\draw[->,line width=1pt] (p1.south) -- (r1.north) node[midway,left] {\scriptsize$T_0^\sharp \nu_1$};
\draw[->,line width=1pt] (p2.south) -- (r2.north) node[midway,left] {\scriptsize$(T_0\circ T_1)^\sharp \nu_2$};
\draw[->,line width=1pt] (p3.south) -- (r3.north) node[midway,left] {\scriptsize$(T_0\circ T_1 \circ T_2)^\sharp \mtar$};

\draw[<->,line width=1pt] (p0.south) to[out=-90,in=180] (r.west) ;
\node[draw=white,fill=white] (n3) at ($(r.west)+(-3.5,2.5)$) {\scriptsize$(T_0)_\sharp \mref= \nu_0$};

\draw[<->,line width=1pt] (r1.south) to[out=-90,in=180]  (r.west) ;
\node[draw=white,fill=white] (n1) at ($(r1.south)+(-0.3,-0.45)$) {\scriptsize$(T_1)_\sharp \mref= T_0^\sharp \nu_1$};

\draw[<->,line width=1pt] (r2.south) to[out=-90,in=0] (r.east) ;
\node[draw=white,fill=white] (n2) at ($(r2.south)+(0.5,-0.45)$) {\scriptsize$(T_2)_\sharp \mref= (T_0\circ T_1)^\sharp \nu_2$};

\draw[<->,line width=1pt] (r3.south) to[out=-90,in=0] (r.east) ;
\node[draw=white,fill=white] (n3) at ($(r.east)+(2.3,0.7)$) {\scriptsize$(T_3)_\sharp \mref= (T_0\circ T_1 \circ T_2)^\sharp \mtar$};

\end{tikzpicture}
\caption{Illustration of DIRT. Top row shows a sequence of bridging measures towards the target measure $\mtar$. Each layer of DIRT identifies an incremental mapping that couples the reference measure (bottom row) with the pullback of the bridging measure under existing composition of mappings (middle row), which admits a simpler structure for constructing TT decomposition.}
\label{fig:dirt-cartoon}
\end{figure}

In Section \ref{sec:dirt}, we circumvent the second challenge by introducing a multi-layer \emph{deep inverse Rosenblatt transport} (DIRT) that builds a composition of SIRTs guided by a sequence of bridging measures with increasing complexity. We illustrate this idea in Figure \ref{fig:dirt-cartoon}.
At each layer of DIRT, we aim to obtain a composition of SIRTs, denoted by $T_0 \circ T_1 \circ \cdots \circ T_k$, such that the pushforward of the reference measure under this composition is a close approximation of the $k$-th bridging measure $\nu_k$. 
The existing composition $T_0 \circ T_1 \circ \cdots \circ T_k$ offers a nonlinear transformation of coordinates that can effectively capture the correlations and support of the next bridging measure $\nu_{k+1}$. 
As a result, the density of the pullback measure, $(T_0 \circ T_1 \circ \cdots \circ T_k)^\sharp \, \nu_{k+1}$, can have a much simpler structure for building the TT decomposition compared with the density of $\nu_{k+1}$. 
We can then factorise the density of $(T_0 \circ T_1 \circ \cdots \circ T_k)^\sharp \, \nu_{k+1}$ to define the incremental mapping $T_{k+1}$ such that $(T_{k+1})_\sharp \, \mref = (T_0 \circ T_1 \circ \cdots \circ T_k)^\sharp \, \nu_{k+1}$. 
This way, DIRT is capable of characterising random variables with concentrated density functions by factorising a sequence of less complicated density functions in transformed coordinates. 
To further improve the efficiency, we also present strategies that can embed \emph{general reference measures} rather than the uniform reference measure to avoid complicated boundary layers during DIRT construction.
Moreover, we can show that the DIRT construction is robust to TT approximation errors in various statistical divergences, in the sense that the error bounds on a range of divergences is a linear combination of errors of TT decompositions involved in the DIRT construction process.

In Section \ref{sec:sampling}, we integrate SIRT and DIRT into existing MCMC and importance sampling methods to further reduce the estimation and sampling bias due to TT approximation errors.
In Section \ref{sec:examples}, we demonstrate the efficiency and various aspects of DIRT on several Bayesian inverse problems governed by ordinary differential equations (ODEs) and partial differential equations (PDEs).
Using a predator-prey dynamical system (Section \ref{sec:pp}), we benchmark the impact of various tuning parameters of the functional TT decomposition such as the TT rank, the number of collocation points and the choice of the reference measure on the accuracy of the DIRT.
Using an elliptic PDE (Section \ref{sec:elliptic}), we are able to compare the single-layered SIRT with DIRT, in which DIRT shows a clear advantage in both the computational efficiency and the accuracy over the single-layered counterpart. In the same example, we also demonstrate the efficiency of TT with the Fourier basis compared to that with the piecewise-linear basis on concentrated posterior measures due to increasing number of measurements and decreasing measurement noises.
Furthermore, we can vary the discretisation of the underlying ODE or PDE models from layer to layer to accelerate the DIRT construction.
For an example involving a computationally expensive parabolic PDE (Section \ref{sec:heat}), we employ models with increasingly refined grids to construct DIRT that is otherwise computationally infeasible to build.

\subsection{Related work}

Apart from building transport maps using TT decompositions, most of other methods approximate the transport map $T$ by solving an optimisation problem such that $T$ minimises some statistical divergence between the target $\mtar$ and the pushforward $T_\sharp \, \mref$.
The mapping $T$ often has a triangular structure, which is computationally efficient for evaluating the Jacobian and the inverse of $T$, and can be represented using polynomials \cite{bigoni2019greedy,el2012bayesian,parno2018transport,peherstorfer2019transport}, kernel functions \cite{detommaso-SVN-2018,liu-stein-2016}, invertible neural networks~\cite{caterini2020variational,chen2019residualflows,pmlr-v119-cornish20a,Detommaso-HINT-2019,papamakarios2019normalizing,rezende2015variational}, etc.
In this setting, the objective function has to be approximated using a Monte Carlo average and minimised by some (stochastic) gradient-based method.
Depending on the objective function and how samples are obtained, the resulting methods may have very different structures. 

\begin{itemize}
\item \emph{Density approximation.} The work of \cite{bigoni2019greedy,el2012bayesian,peherstorfer2019transport} aims to minimise the Kullback--Leibler (KL) divergence of the pushforward $T_\sharp \, \mref$ from the target $\mtar$, in which the pushforward density naturally approximates the target density. 
In this case, the KL divergence is approximated using the Jacobian of $T$ and the target density function evaluated at samples drawn from the analytically tractable reference measure. 
The resulting optimisation problem may be highly nonlinear and non-convex. 
In each optimisation iteration, the target density function has to be re-evaluated as reference samples are transported by the updated map. 
Our TT-based methods also rely on approximations to the target density. However, TT approximations employ highly efficient deterministic sampling algorithms such as TT-Cross~\cite{oseledets2010tt}, which are free from either gradient or Monte Carlo. This way, TT-based methods may need less number of density evaluations to accurately approximate the target density.

\item \emph{Density estimation.} The strategy adopted by normalising flows (e.g.,~\cite{caterini2020variational,chen2019residualflows,pmlr-v119-cornish20a,Detommaso-HINT-2019,papamakarios2019normalizing,rezende2015variational}) and the work of \cite{parno2018transport,tabak2013family,trigila2016data} offer an alternative that can bypass evaluations of the target density. Instead, these methods assume availability of samples drawn from the target measure and construct objective functions using a given set of target samples. 
Many of these methods, particularly neural networks, were originally designed to approximate high-dimensional distributions of naturally available samples, such as images. 
However, in our context, the intractable target random variable $\bX$ cannot be simulated directly. One has to assume that there exists an auxiliary random variable $\bY$ such that the density function of $\bX$ is given by a conditional density $\pi(\bx|\by)$ and the pair of joint random variables $(\bX, \bY)$ can be simulated directly. 
This way, density estimations can be employed to first build a mapping from some higher dimensional reference measure to the joint measure of $(\bX, \bY)$, and then obtain the mapping $T$ by conditioning on a particular realisation of $\bY=\by$. We provide a concrete example of normalising flows and its comparisons with DIRT in Section \ref{sec:pp}.

\item \emph{Greedy methods.} In-between the fully data driven density estimation and the function driven density approximation is the greedy strategy, including the Stein variational gradient descent method \cite{liu-stein-2016}, its Newton variant \cite{detommaso-SVN-2018}, and the lazy maps \cite{bigoni2019greedy}. 
While greedy methods build transport maps sharing a similar composition structure with DIRT, they obtain the composition of mappings by iteratively minimising the KL divergence of the pushforward measure under the current composition of maps from the target $\mtar$. 
To relax the burden in optimisation, the class of mappings used in each layer of greedy methods is often restricted, for example, to reproducing kernel Hilbert space with Gaussian kernels \cite{detommaso-SVN-2018,liu-stein-2016} and sparse low-order polynomials  \cite{bigoni2019greedy}. As a result, greedy methods often need a rather large number of layers to accurately approximate concentrated target densities, and hence may lead to a large number of computationally costly target density evaluations.
Compared to the greedy strategy, the usage of bridging measures allows DIRT to construct TT decompositions in different layers with arbitrary accuracy.
The total error in DIRT is also accumulated linearly with the number of layers. 
We provide a numerical comparison on the performance of DIRT and the Stein variational Newton methods \cite{detommaso-SVN-2018} in Section \ref{sec:pp}.
\end{itemize}

%% file: sec2_background.tex

\section{Background}\label{sec:background}
In this section, we first introduce some notation and assumptions used throughout the paper. Then, we review the inverse Rosenblatt transport method that offers an algebraically exact transformation from the reference measure to the target measure. We will also discuss the role of the functional TT decomposition in the numerical construction of the (approximate) inverse Rosenblatt transport.

\subsection{Notation}\label{sec:notation}
We consider probability measures that are absolutely continuous with respect to the Lebesgue measure.
Suppose a mapping $S: \mathcal{X} \mapsto \mathcal{U}$ is a diffeomorphism and a probability measure $\nu$ has a density $p(\bx)$, the pushforward of $\nu$ under $S$, denoted by $S_\sharp\nu$, has the density:
\begin{equation}
S_\sharp p(\bu) = 
\big(p \circ S^{-1}\big)(\bu) \, \big| \nabla_{\bu} S^{-1}(\bu)\big|.
\end{equation}
Similarly, given a probability measure $\lambda$ with a density $q(\bu)$, the pullback of $\lambda$ under $S$, denoted by $S^\sharp \lambda$, has the density:
\begin{equation}
S^\sharp q(\bx) = 
\big(q \circ S \big)(\bx) \, \big| \nabla_{\bx} S(\bx)\big|.
\end{equation}
The short hand $\bX \sim \nu$ is used to refer a random variable $\bX$ with the law $\nu$. For a $\nu$-integrable function $q: \mathcal{X}\mapsto\mathbb{R}$, the expectation of $q$ is denoted by $\nu(q) = \int q(\bx) \nu(d\bx) $. 

We assume the parameter space $\mathcal{X}$ and the reference space $\mathcal{U}$ can be expressed as Cartesian products $\mathcal{X} = \mathcal{X}_1 \times \mathcal{X}_2 \times \cdots \times \mathcal{X}_d$  and $\mathcal{U} = \mathcal{U}_1 \times \mathcal{U}_2 \times \cdots \times \mathcal{U}_d$ respectively, where $\mathcal{X}_k \subseteq \mathbb{R}$ and $\mathcal{U}_k = [0, 1]$. 
Using product-form Lebesgue measurable weighting functions $\lambda(\bx) = \prod_{i = 1}^{d} \lambda_i(\bx_i) $ and $\omega(\bu)= \prod_{i = 1}^{d} \omega_i(\bx_i)$, the weighted $L^p$ norms on $\mathcal{X}$ and $\mathcal{U}$ can be expressed as
\[
\lpnormx{f}{p} = \lpintx{f(\bx)}{p} \quad \textrm{and} \quad \lpnormu{g}{p} = \lpintu{g(\bu)}{p},
\]
respectively. 
We define constants $\lambda_i(\mathcal{X}_i) = \int_{\mathcal{X}_i} \lambda_i(\bx_i) d\bx_i$ for $i = 1, \ldots, d$ and $\lambda(\mathcal{X}) = \prod_{i = 1}^d \lambda_i(\mathcal{X}_i)$. Likewise, we also define $\omega_i(\mathcal{U}_i)$ for $i = 1, \ldots, d$ and $\omega(\mathcal{U})$.

For a vector $\bx \in \R^d$ and an index $k \in \N$ such that $1 < k < d$, we express the first $k-1$ coordinates and the last $d-k$ coordinates of $\bx$ as
\begin{equation*}
\bx_{<k} \equiv [\bx_1, \ldots, \bx_{k-1}]^\top, \quad \textrm{and} \quad \bx_{>k} \equiv [\bx_{k+1}, \ldots, \bx_{d}]^\top,
\end{equation*}
respectively. Similarly, we write $\bx_{\leq k} = (\bx_{<k}, \bx_k)$, $\bx_{\geq k} = (\bx_{k}, \bx_{>k})$, $\bx_{\leq 1} = \bx_1$, and $\bx_{\geq d} = \bx_{d}$.
For any non-negative function $\ptar \in L^1_{\lambda(\mathcal{X})}$, we define its marginal functions as
\begin{equation}
\ptar_{\leq k}(\bx_{\leq k}) \equiv \int \ptar (\bx_{\leq k}, \bx_{>k}) \, \Big({\textstyle\prod_{i = k+1}^{d} \lambda_i(\bx_i)}\Big)\,d\bx_{>k}, \quad \textrm{for} \quad 1 \leq k < d,
\end{equation}
with $\ptar_{\leq d}(\bx_{\leq d}) = \ptar(\bx)$. The marginal functions should not be confounded with $\ptar_1(\bx)$, $\ptar_2(\bx)$, $\ldots$, $\ptar_k(\bx)$ where the subscript indexes a sequence of functions on $\mathcal{X}$.

\subsection{Inverse Rosenblatt transport} 

We start with a $d$-dimensional uniform reference probability measure, $\muni$, defined in a unit hypercube $\mathcal{U} = [0,1]^d$, which has the probability density function (PDF) $\puni(\bu) = 1$. 
We aim to characterise a target probability measure $\mtar$ with the PDF
\begin{equation}
f_{\bX}(\bx) = \frac1z\, \ptar(\bx)\,\lambda(\bx) \quad \textrm{and}\quad z = \int_{\mathcal{X}} \ptar(\bx) \lambda(\bx) d\bx.
\end{equation}
Here, $\ptar \in \lpx{1}$ is the unnormalised density function (with respect to  the weight $\lambda$) that is non-negative, i.e., $\ptar(\bx) \geq 0, \forall \bx \in \mathcal{X}$, and $z$ is the normalising constant that is often unknown.

Let $\bX:=(\bX_1,\ldots,\bX_d)$ be the target $d$-dimensional random variable with law $\mtar$ and $\bU$ be the reference $d$-dimensional random variable with law $\muni$. The {\it Rosenblatt transport} offers a viable way to constructing a map $F: \R^d \mapsto \R^d$ such that $F(\bX) = \bU$.
As explained in \cite{carlier2010knothe,spantini2018inference,villani2008optimal}, the principle of the Rosenblatt transport is the following. 
For $1 \leq k \leq d$, we denote the marginal PDF of the $k$-dimensional random variable $\bX_{\leq k}:=(\bX_1,\ldots,\bX_k)$ by
\[
f_{\bX_{\leq k}} ( \bx_{\leq k} )= \frac1z\,\ptar_{\leq k}(\bx_{\leq k}) \, \Big(\textstyle\prod_{i = 1}^{k} \lambda_i(\bx_i)\Big)
\] 
and the PDF of the conditional random variable $\bX_{k} | \bX_{< k}$ by
\begin{equation*}
f_{\bX_{k} | \bX_{< k}}( \bx_k | \bx_{<k}) = \frac{f_{\bX_{\le k}}( \bx_{<k}, \bx_k)}{f_{\bX_{< k}}( \bx_{<k})} = \frac{\ptar_{\le k}( \bx_{<k}, \bx_k)}{\ptar_{< k}( \bx_{<k})}\, \lambda_k(\bx_k).
\end{equation*}
This way, the CDF of $\bX_{1}$ and the conditional CDF of $\bX_{k} | \bX_{< k}$ can be expressed as
\begin{equation}\label{eq:marginal_CDF}
\!\!F_{\bX_1}(\bx_1 ) = \!\! \int_{-\infty}^{\bx_1} f_{\bX_{1}}( \bx_1)\,d\bx_{1} \;\; \text{and} \;\; F_{\bX_{k} | \bX_{< k}}(\bx_k | \bx_{<k}) = \!\!\int_{-\infty}^{\bx_k} f_{\bX_{k} | \bX_{< k}}( \bx_k | \bx_{<k}) \, d\bx_{k},\!
\end{equation}
respectively. 
Under mild assumptions \cite{carlier2010knothe}, the following sequence of transformations 
\begin{equation}\label{eq:forward_trans}
\left \{
\begin{array}{ll}
\bu_1 \;\;= \mathbb{P}[\bX_1 \leq \bx_1] & = F_{\bX_{1}}(\bx_1 ) \vspace{-6pt} \\ 
 \quad\;\;\;\;\,\vdots &  \vspace{-2pt}\\
\bu_k \;\;= \mathbb{P}[\bX_k \leq \bx_k | \bX_{<k} = \bx_{<k}] & = F_{\bX_{k} | \bX_{< k}}(\bx_k | \bx_{<k})  \vspace{-6pt}\\ 
 \quad\;\;\;\;\,\vdots & \vspace{-2pt}\\
\bu_d \;\;= \mathbb{P}[\bX_d \leq \bx_d | \bX_{<d} = \bx_{<d}] & = F_{\bX_{d} | \bX_{< d}}(\bx_d | \bx_{<d})
\end{array}
\right.
\end{equation}
defines uniquely a monotonically increasing map $F: \mathcal{X} \mapsto \mathcal{U}$ in the form of
\begin{equation}
F(\bx) = \left[F_{\bX_1}(\bx_1), \cdots, F_{\bX_{k} | \bX_{< k}}(\bx_k | \bx_{< k}), \cdots, F_{\bX_{d} | \bX_{< d}}(\bx_d | \bx_{<d})\right]^\top,
\label{eq:Rosenblatt}
\end{equation}
such that the random variable $\bU = F(\bX)$ is uniformly distributed in the unit hypercube $[0,1]^d$.
Since the $k$-th component of $F$ is a scalar-valued function depending on only the first $k$ variables, that is, $F_{\bX_{k} | \bX_{< k}}: \R^{k} \mapsto \R$, the map $F$ has a {\it lower-triangular} form.
Furthermore, the map $F$ (as well as its inverse) is almost surely differentiable and satisfies
\[
F^\sharp  \puni(\bx) = \big(\puni \circ F \big)(\bx)\,\big| \nabla_{\bx} F(\bx) \big|  = \big| \nabla_{\bx} F(\bx) \big| = f_{\bX}(\bx)
\]
$\mtar$-almost surely.

Suppose one can compute the Rosenblatt transport. Then, it provides a viable way to characterising the target measure. One can first generate uniform random variables $\bU \sim \muni$ and then applying the {\it inverse Rosenblatt transport} (IRT)
\[
\bX = F^{-1} \big(\bU \big)
\]
to obtain a corresponding target random variable $\bX \sim \mtar$. The inverse Rosenblatt transport $T \equiv F^{-1}: \mathcal{U} \mapsto \mathcal{X}$ is also {\it lower-triangular} and can be constructed by successively inverting the Rosenblatt transport for $k = 1, \ldots, d$:
\begin{equation}\label{eq:inverse_trans}
\bx = T(\bu) \equiv 
\left[ F_{\bX_1}^{-1}(\bu_1 ), \ldots, F_{\bX_{k} | \bX_{< k}}^{-1}(\bu_k | \bx_{<k}), \ldots, F_{\bX_{d} | \bX_{< d}}^{-1}(\bu_d | \bx_{<d}) \right]^\top.
\end{equation}
The evaluation of each $F_{\bX_{k} | \bX_{< k}}^{-1}(\bu_k | \bx_{<k})$ requires inverting only a scalar valued monotone function $\bu_k = F_{\bX_{k} | \bX_{< k}}(\bx_k | \bx_{<k})$, where $\bx_{<k}$ is already determined in the first $k-1$ steps.
Using the change-of-variables formula, the expectation of a  function $h : \mathcal{X} \mapsto \R$ can be expressed as 
\[
\mtar(h) = \muni(h\circ T).
\]
This way, the expectation over the intractable target probability measure can be expressed as the expectation over a reference uniform probability measure, and thus many efficient high-dimensional quadrature methods such as sparse grids~\cite{griebel-sparsegrids-2004} and quasi Monte Carlo~\cite{Kuo-QMC-2013} may apply.

\subsection{Functional tensor-train} \label{sec:ftt}
For high-dimensional target measures, it may be not computationally feasible to compute the marginal densities $\ptar_{\le k}$, and hence the marginal and the conditional CDFs in \eqref{eq:marginal_CDF} for building the inverse Rosenblatt transport.
To overcome this challenge, a recent work \cite{dafs-tt-bayes-2019} employed the TT decomposition~\cite{oseledets2011tensor} to factorise the density of the target measure in a separable form, which leads to a computationally scalable method for building the inverse Rosenblatt transport.
Here we first discuss the basics of the TT decomposition of a multivariate function.

\newcommand{\ttcorej}[3]{ {#3}^{(\alpha_{#1-1}, \alpha_{#1})}_{#1} (#2_{#1}) }
\newcommand{\ttcorejp}[3]{ {#3}^{(\alpha_{#1}, \alpha_{#1+1})}_{#1+1} (#2_{#1+1}) }
\newcommand{\ttcorejm}[3]{ {#3}^{(\alpha_{#1-2}, \alpha_{#1-1})}_{#1-1} (#2_{#1-1}) }
\newcommand{\ttcorez}[2]{ {#2}^{(\alpha_0, \alpha_1)}_{1} (#1_1) }

\newcommand{\ttcorejpr}[3]{ {#3}^{(\beta_{#1-1}, \beta_{#1})}_{#1} (#2_{#1}) }
\newcommand{\ttcorejppr}[3]{ {#3}^{(\beta_{#1}, \beta_{#1+1})}_{#1+1} (#2_{#1+1}) }
\newcommand{\ttcorezpr}[2]{ {#2}^{(\beta_0, \beta_1)}_{1} (#1_1) }

\newcommand{\tteigjp}[1]{ \lambda_{\alpha_{#1+1}} }
\newcommand{\tteigj}[1]{ \lambda_{ \alpha_{#1}} }
\newcommand{\tteigz}{ \lambda_{\alpha_1} }

Since multivariate functions can be viewed as continuous analogues of tensors \cite{hackbusch2012tensor}, one can factorise the unnormalised density function using functional-form of TT \cite{bigoni2016spectral,gorodetsky2019continuous,griebel2019analysis}.
Given a multivariate function $h : \mathcal{X} \mapsto \R$, where $\mathcal{X} = \mathcal{X}_1 \times \mathcal{X}_2 \times \ldots \times \mathcal{X}_d$, TT approximates $h(\bx)$ as
\begin{equation}
h(\bx) \approx \tilde{h}(\bx) \equiv
\sum_{\alpha_0=1}^{r_0}\sum_{\alpha_1 = 1}^{r_1} \cdots \sum_{\alpha_d = 1}^{r_d} \ttcorez{\bx}{\mH} \cdots \ttcorej{k}{\bx}{\mH} \cdots \ttcorej{d}{\bx}{\mH} , 
\end{equation}
with $r_0 = r_d = 1$, where the summation ranges $r_0, r_1, \ldots, r_d$ are called \emph{TT ranks}.
Each univariate function $\ttcorej{k}{\bx}{\mH}: \mathcal{X}_k \mapsto \R$ is represented as a linear combination of a set of
$n_k$ basis functions $\{\phi_{k}^{(1)}(\bx_k), \ldots,\phi_{k}^{(n_k)}(\bx_k)\}$. This way, we have
\begin{equation}
\ttcorej{k}{\bx}{\mH} = \sum_{i =1}^{n_k}\phi_{k}^{(i)}(\bx_k) \,\tA_k [\alpha_{k-1}, i, \alpha_k], 
\end{equation}
where $\tA_k \in \R^{r_{k-1}\times n_k \times r_k}$ is a coefficient tensor.
Examples of the basis functions include piecewise polynomials, orthogonal functions, radial basis functions, etc.
In general, the TT decomposition $\tilde{h}(\bx)$ is only an approximation to the original function $h(\bx)$ because of truncated TT ranks and sets of basis functions used for representing each $\ttcorej{k}{\bx}{\mH}$.

\begin{remark}
For each $k$, grouping all the univariate functions $\ttcorej{k}{\bx}{\mH}$, we have a matrix valued function $\mH_{k}(\bx_k) : \mathcal{X}_k \mapsto \R^{r_{k-1} \times r_k}$ that is commonly referred to as the {\it $k$-th TT core}.
This way, the TT decomposition can also be expressed in the matrix form
\begin{equation}
\tilde{h}(\bx) = \mH_{1}(\bx_1) \cdots \mH_{k}(\bx_k) \cdots \mH_{d}(\bx_d).
\end{equation}
We follow the MATLAB notation to denote vector-valued functions consisting of the $\alpha_k$-th column and $\alpha_{k-1}$-th row of $\mH_{k}(\bx_k)$ by $\mH_{k}^{(\,:\,, \alpha_k)}(\bx_k) : \mathcal{X}_k \mapsto \R^{r_{k-1} \times 1}$ and $\mH_{k}^{(\alpha_{k-1}, \,:\,)}(\bx_k) : \mathcal{X}_k \mapsto \R^{1 \times r_{k}}$, respectively.
In some situations, it is convenient to represent the TT decomposition with grouped coordinates. For example, we can write  TT as the functional analogue of the compact singular value decomposition (SVD):
\begin{equation}
\tilde{h}(\bx) = \sum_{\alpha_k=1}^{r_k} \mH_{\leq k}^{(\alpha_k)}(\bx_{\leq k}) \, \mH_{> k}^{(\alpha_k)}(\bx_{>k}),
\end{equation}
where
\begin{align}
\mH_{\leq k}^{(\alpha_k)}(\bx_{\leq k})  & =  \mH_{1}(\bx_1) \cdots \mH_{k-1}(\bx_{k-1})\,\mH_{k}^{(\,:\,,\alpha_k )}(\bx_k) : \mathcal{X}_1 \times \cdots \times \mathcal{X}_k \mapsto \R, \label{eq:ftt_groupL} \\
\mH_{> k}^{(\alpha_k)}(\bx_{>k}) & = \mH_{k+1}^{(\alpha_k,\,:\,)}(\bx_{k+1})\,\mH_{k+2}(\bx_{k+2}) \cdots \mH_{d}(\bx_{d}) : \mathcal{X}_{k+1} \times \cdots \times \mathcal{X}_d \mapsto \R.\label{eq:ftt_groupR}
\end{align}
\end{remark}

Given a multivariate function, its TT decomposition can be computed using alternating linear schemes such as the classical alternating least squares method (e.g.,~\cite{kolda2009tensor,oseledets2010tt}), density matrix renormalization group methods \cite{holtz2012alternating,oseledets2011tensor,white1993density}, and the alternating minimal energy method \cite{dolgov2014alternating} together with the cross approximation \cite{goreinov2010find,goreinov1997theory,goreinov1997pseudo,mahoney2009cur} or the empirical interpolation \cite{barrault2004empirical,chaturantabut2010nonlinear}.
In Appendix \ref{sec:tt_cross}, we detail the cross algorithm used for constructing the functional TT decomposition.

\subsection{A TT-based inverse Rosenblatt transport} \label{sec:ttirt}
Using the functional TT previously discussed, now we review the TT-based construction of the inverse Rosenblatt transport \cite{dafs-tt-bayes-2019}.
Suppose one has the (approximate) TT decomposition $\tilde{\pi}(\bx)$ of the unnormalised target density $\pi(\bx)$ in the form of 
\[
\tilde{\pi}(\bx_1, \bx_2, \ldots, \bx_d) =
\mF_{1}(\bx_1) \cdots \mF_{k}(\bx_k) \cdots \mF_{d}(\bx_d),
\]
where $\mF_{k}(\bx_k) : \mathcal{X}_k \mapsto \R^{r_{k-1} \times r_k}$ is the $k$-th TT core. 
Then, we can approximate the target PDF by
\begin{equation}
\tilde{f}_{\bX}(\bx) = \frac{1}{\tilde c}\, \tilde \pi(\bx) \, \lambda(\bx), \quad \textrm{where}\quad {\tilde c} = \int_{\mathcal{X}} \tilde \pi(\bx) \lambda(\bx) d\bx.
\end{equation}
\begin{proposition}
For $k < d$, the $k$-th marginal PDF is given by
\[
\tilde{f}_{\bX_{\leq k}}(\bx_{\leq k}) = \frac{1}{\tilde c} \, \tilde{\pi}_{\leq k}(\bx_{\leq k}) \, \Big({\textstyle\prod_{i = 1}^{k}} \lambda_i(\bx_i)\Big),
\]
where $\tilde{\ptar}_{\leq k}(\bx_{\leq k}) = \mF_{1}(\bx_{1}) \cdots \mF_{k}(\bx_{k}) \bar{\mF}_{k+1} \cdots \bar{\mF}_{d}$, ${\tilde c} = \bar{\mF}_{1} \cdots \bar{\mF}_{d}$,
and the matrices $\bar{\mF}_{k}$ are
the integrated TT cores
\[
\bar{\mF}_{k} = \int_{\mathcal{X}_k} {\mF}_{k} (\bx_k) \, \lambda_{k}(\bx_{k})\, d\bx_{k} \in \R^{r_{k-1} \times r_k}, \quad \textrm{for} \quad k = 1, \ldots, d.
\] 
\end{proposition}
\begin{proof}
The marginal function of $\tilde{\pi}(\bx)$ can be expressed by
\begin{align*}
\tilde{\ptar}_{\leq k}(\bx_{\leq k}) 
= \int_{\mathcal{X}_{>k}} \mF_{1}(\bx_{1}) \cdots \mF_{d}(\bx_{d}) \,\left(\textstyle{\prod_{i = k+1}^{d}} \lambda_i(\bx_i)\right)\,d\bx_{>k} .
\end{align*}
Using the separable form of the tensor--train, the marginal density then has the form 
\begin{align*}
& \hspace{-6pt} \tilde{\ptar}_{\leq k}(\bx_{\leq k}) \\
& = \mF_{1}(\bx_{1}) \cdots \mF_{k}(\bx_{k}) \bigg(\int_{\mathcal{X}_{k+1}} \!\!\!\!\!\!\mF_{k+1}(\bx_{k+1}) \lambda_{k+1}(\bx_{k+1})d\bx_{k+1}\bigg) \cdots  \left( \int_{\mathcal{X}_d} \!\!\!\! \mF_{d}(\bx_{d}) \lambda_{d}(\bx_d)d\bx_{d} \right) \\
& = \mF_{1}(\bx_{1}) \cdots \mF_{k}(\bx_{k}) \bar{\mF}_{k+1} \cdots \bar{\mF}_{d}.
\end{align*}
Since ${\tilde c} = \int_{\mathcal{X}} \tilde{f}(\bx) \lambda(\bx)d\bx$, we have ${\tilde c} = \bar{\mF}_{1} \cdots \bar{\mF}_{d}$ using a similar argument.
\end{proof}

The above proposition leads to the marginal PDF $\tilde{f}_{\bX_1} = \frac1{\tilde c}\, \tilde{\pi}_{\leq 1}(\bx_{1})\,\lambda_1(\bx_1)$ 
and the sequence of conditional probability densities
\begin{equation}
\tilde{f}_{\bX_k | \bX_{< k}}(\bx_k | \bx_{<k}) = \frac{\tilde{\ptar}_{\le k}( \bx_{<k}, \bx_k)}{\tilde{\ptar}_{< k}( \bx_{<k})}\,\lambda_k(\bx_k), \quad k=2, \ldots, d.
\end{equation}
This leads to the CDF and the sequence of conditional CDFs
\begin{equation}
\tilde{F}_{\bX_1}(\bx_1 ) = \!\! \int_{-\infty}^{\bx_1} \tilde{f}_{\bX_1}\,d\bx_1^\prime \;\; \text{and}\;\;
\tilde{F}_{\bX_k | \bX_{<k} }(\bx_k | \bx_{<k}) = \!\! \int_{-\infty}^{\bx_k}  \tilde{f}_{\bX_k | \bX_{< k}}(\bx_k | \bx_{<k})\, d\bx_k^\prime, 
\end{equation}
for $k=2, \ldots, d$, and hence the Rosenblatt transport $\bU = \tilde{F}(\bX)$. 
This equivalently defines the inverse Rosenblatt transport $\tilde{T} = \tilde{F}^{-1}$. 
This way, by drawing a reference random variable $\bU \sim \muni$ and evaluating $\bX = \tilde{T}(\bU)$, we obtain an approximate target random variable $\bX \sim \tilde{T}_\sharp \muni$. Note that the pushforward measure $\tilde{T}_\sharp \muni$ has the density $\tilde{f}_{\bX}(\bx)$.

To estimate the numerical complexity, let us introduce the maximal number of basis functions $n=\max_{k=1,\ldots,d} n_k$, TT rank $r = \max_{k=0,\ldots,d} r_k$, and suppose we need to draw $N$ samples from $\tilde \pi(\bx)$.
Note that we can precompute $\bar{\mF}_{k+1} \cdots \bar{\mF}_{d}$ with the total cost of $\mathcal{O}(dnr^2)$ operations, before any sampling starts.
Similarly, the conditioning requires the interpolation of $\mF_1(\bx_1) \cdots \mF_{k-1}(\bx_{k-1})$ at the current sample coordinates, which can be built up sequentially.
Each univariate interpolation needs $\mathcal{O}(nr^2)$ operations in general, but for a piecewise interpolation this can be reduced to $\mathcal{O}(r^2)$ operations per sample per coordinate.
Finally, the assembling of the conditional density requires the multiplication of $N$ vectors $\mF_1(\bx_1) \cdots \mF_{k-1}(\bx_{k-1}) \in \R^{r_{k-1}}$ with a vector-valued function $\mF_k(\bx_k)\bar{\mF}_{k+1} \cdots \bar{\mF}_{d} \in \R^{r_{k-1}}$.
The total complexity is therefore $\mathcal{O}(dnr^2 + Ndr^2 + Ndnr)$~\cite{dafs-tt-bayes-2019}.

Constructing the inverse Rosenblatt transport using the TT decomposition of the target density faces several challenges. First, the density function $\pi(\bx)$ is non-negative, however, its truncated TT decomposition  $\tilde{\pi}(\bx)$ can have negative values---a discrete analogue is that the truncated SVD of a matrix filled with non-negative entries can be negative.
The leads to a critical issue: if the set $\{ \bx \in \mathcal{X} \,|\, \tilde{\pi}(\bx) < 0\}$ has nonzero measure under $\mtar$, then the Rosenblatt transport constructed from $\tilde{\pi}(\bx)$ loses monotonicity.
A simple way to circumvent this is to take the modulus of each univariate conditional density $\tilde{f}_{\bX_k|\bX_{<k}}(\bx_k | \bx_{<k})$ and then renormalise the modulus before computing the CDF~\cite{dafs-tt-bayes-2019}.
However, the use of moduli and renormalisations may degrade the smoothness of marginal PDFs and conditional PDFs.
This way, the resulting inverse Rosenblatt transport and its induced PDF can lose accuracy and smoothness.
More importantly, the construction of the TT decomposition (see Section~\ref{sec:tt_cross} for details) requires evaluating the target density at parameter points where the target density is significant.
In practice, the high probability region of a high-dimensional target density, e.g., the posterior in the Bayesian inference context, can be hard to characterise.
Thus, it can be challenging to construct the TT decomposition for approximating the target density directly.
In the next section, we generalise the TT-based construction of the inverse Rosenblatt transport by tackling the aforementioned challenges.

%% file: sec3_methods.tex

\section{Squared inverse Rosenblatt transport}\label{sec:sirt}
We first introduce the SIRT to overcome the negativity issue outlined above. 
Instead of directly decomposing the unnormalised target density $\ptar(\bx)$, we first obtain the (approximate) functional TT decomposition $\tilde{g}(\bx)$ of the square root of $\ptar(\bx)$ in the form of 
\begin{equation}\label{eq:ftt_sqrt}
\sqrt{\ptar}(\bx) \approx \tilde{g}(\bx) =
\mG_{1}(\bx_1) \cdots \mG_{k}(\bx_k) \cdots \mG_{d}(\bx_d),
\end{equation}
where $\mG_{k}(\bx_k) : \mathcal{X}_k \mapsto \R^{r_{k-1} \times r_k}$ is the $k$-th TT core. 
This leads to an alternative approximation to the target PDF:
\begin{equation}\label{eq:pdf_sirt}
\hat{f}_{\hat{\bX}}(\bx)  = \frac{1}{\hat z} \, \hat{\pi}(\bx)  \,\lambda(\bx), \;\; \textrm{with}\;\;  \hat \pi(\bx) = \gamma  + \tilde g(\bx)^2\;\; \textrm{and}\;\; \hat z = \gamma \,\lambda(\mathcal{X}) + \int_{\mathcal{X}} \tilde g(\bx)^2 \,\lambda(\bx)\,d\bx,
\end{equation}
where $\gamma > 0$ is a constant chosen according to the $L^2$ error of $\tilde g(\bx)$.
Similar to the process discussed in Section \ref{sec:ttirt}, we can obtain the SIRT, $\hat{\bX} = \hat{T}(\bU)$, by constructing the sequence of marginal functions $\hat{\pi}_{ \leq k}(\bx_{\leq k}) = \int_{\mathcal{X}} \hat{\pi}(\bx) \prod_{i=k+1}^{d}\lambda_i(\bx_i) d\bx_{> k}$ for  $k=1, \ldots, d-1$ and computing the normalising constant $\hat z$.
Given a reference random variable $\bU \sim \muni$, we can evaluate $\hat{\bX} = \hat{T}(\bU)$ to obtain an approximate target random variable $\hat{\bX} \sim \hat{T}_\sharp  \muni$, which has exactly the PDF $\hat{f}_{\hat{\bX}}(\bx)$.
Since the function $\hat{\pi}(\bx)$ is positive by construction, we can \emph{preserve the smoothness and monotonicity} in the resulting SIRT $\hat{T}$. 

\begin{remark}\label{remark:sirt_bound_ratio}
For a target density $\pi(\bx)$ satisfying $\sup_{\bx\in \mathcal{X}} \pi(\bx) < \infty$, the ratio between $\pi(\bx)$ and the approximate density $\hat\pi(\bx)$ satisfies
\begin{equation}\label{eq:ratio_bound}
\sup_{x\in \mathcal{X}}\,\frac{\pi(\bx)}{\hat\pi(\bx)} = \sup_{x\in \mathcal{X}}\,\frac{\pi(\bx)}{\gamma + \tilde{g}(\bx)^2} = \hat{c} < \infty.
\end{equation}
This bound is essential to ensure the uniform ergodicity of the Metropolis independent algorithm and the rate of convergence of importance sampling schemes defined by SIRT. See Section \ref{sec:sampling} for further details.
\end{remark} 

\subsection{Marginal functions and conditional PDFs}%
We represent each TT core of the decomposition in \eqref{eq:ftt_sqrt} as
\begin{equation}
\ttcorej{k}{\bx}{\mG} = \sum_{i =1}^{n_k}\phi_{k}^{(i)}(\bx_k) \tA_k [\alpha_{k-1}, i, \alpha_k], \quad k = 1, \ldots, d,
\end{equation}
where $\{\phi_{k}^{(i)}(\bx_k)\}_{i = 1}^{n_k}$ is the set of basis functions for the $k$-th coordinate and $\tA_k \in \R^{r_{k-1}\times n_k \times r_k}$ is the associated $k$-th coefficient tensor. 
For the $k$-th set of basis functions, we define the mass matrix $\mM_k \in \R^{n_k \times n_k}$ by
\begin{equation}\label{eq:mass-k}
\mM_k[i,j] = \int_{\mathcal{X}_k} \phi_{k}^{(i)}(\bx_k)\phi_{k}^{(j)}(\bx_k) \,\lambda(\bx_k)\,d\bx_k, \quad \text{for}\quad  i = 1,\ldots, n_k,\,j = 1,\ldots, n_k.
\end{equation}
Then,  we can represent the marginal functions by
\begin{align}
\hat{\ptar}_{1}(\bx_{1}) & = \gamma \prod_{i = 2}^d \lambda_i(\mathcal{X}_i) + \sum_{\ell_{1}=1}^{r_{1}}  \Big(\sqbasis_{1}^{(\,\alpha_0\,,\ell_{1})}(\bx_{1}) \Big)^2,  \label{eq:marginal_sq1} \\
\hat{\ptar}_{ \leq k}(\bx_{\leq k}) & = \gamma \!\!\! \prod_{i = k+1}^d \!\!\! \lambda_i(\mathcal{X}_i) + \sum_{\ell_{k}=1}^{r_{k}} \Big(\sum_{\alpha_{k-1}=1}^{r_{k-1}} \mG^{(\alpha_{k-1}) }_{<k}(\bx_{<k}) \, \sqbasis_{k}^{(\alpha_{k-1},\ell_{k})}(\bx_{k}) \Big)^2 , \quad k = 2, \ldots, d, \label{eq:marginal_sqk}
\end{align}
where $\alpha_0 = 1$ and
\begin{align}\label{eq:G_<k}
\mG^{(\alpha_{k-1})}_{<k}(\bx_{<k}) &= \mG_{1}(\bx_{1}) \cdots \mG^{(\,:\,,\alpha_{k-1})}_{\km}(\bx_{\km}): \mathcal{X}_{<k} \mapsto \mathbb{R}, \\
\sqbasis_{k}^{(\alpha_{k-1}, \ell_{k})}(\bx_{k}) & = \sum_{i =1}^{n_k}\phi_{k}^{(i)}(\bx_k) \, \tB_k [\alpha_{k-1}, i, \ell_k] : \mathcal{X}_{k} \mapsto \mathbb{R}\label{eq:L_k} ,
\end{align}
for a coefficient tensor $\tB_k \in \R^{r_{k-1}\times n_k \times r_k}$ that is recursively defined as follows.

\begin{proposition}\label{prop:sirt_recur}
Starting with the last coordinate $k = d$, we set $\tB_d = \tA_d$. 
Suppose for the first $k$ dimensions ($k>1$), we have a coefficient tensor $\tB_k \in \R^{r_{k-1}\times n_k \times r_k}$ that defines a marginal function $\hat\pi_{\leq k}(\bx_{\leq k})$ as in \eqref{eq:marginal_sqk}.
The following procedure can be used to obtain the coefficient tensor $\tB_{k-1}\in \R^{r_{k-2}\times n_{k-1} \times r_{k-1}}$ for defining the next marginal function $\hat{\ptar}_{ < k}(\bx_{< k})$:
\begin{enumerate}[leftmargin=14pt]
\item Use the Cholesky decomposition of the mass matrix, $\chol_k \chol_k^\top = \mM_k \in \R^{n_k \times n_k}$, to construct a tensor $\tC_k \in \R^{r_{k-1}\times n_k \times r_k}$:
\begin{align}
\tC_k[\alpha_{k-1}, \tau, \ell_{k}] = \sum_{i = 1}^{n_k} \tB_k[\alpha_{k-1}, i, \ell_{k}] \, \chol_k[i, \tau].
\end{align}
\item  Unfold $\tC_k$ along the first coordinate \cite{kolda2009tensor} to obtain a matrix $\mC_k^{(\rm R)} \in \R^{r_{k-1} \times (n_k  r_k)} $ and compute the thin QR decomposition
\begin{align}\label{eq:SIRT-QR}
\mQ_k \mR_k = \big( \mC_k^{(\rm R)} \big)^\top,
\end{align}
where $\mQ_k \in \R^{(n_k  r_k) \times r_{k-1}} $ is semi-orthogonal and $\mR_k \in \R^{r_{k-1} \times r_{k-1}}$ is upper-triangular.
\item Compute the new coefficient tensor
\begin{align}\label{eq:B_recur}
\tB_{k-1}[\alpha_{k-2},i, \ell_{k-1}] =  \sum_{\alpha_{k-1} = 1}^{r_{k-1}} \tA_{k-1}[\alpha_{k-2},i, \alpha_{k-1}]\, \mR_k[\ell_{k-1},\alpha_{k-1}].
\end{align}
\end{enumerate}
Furthermore, at index $k = 1$, the unfolded $\tC_1$ along the first coordinate is a row vector $\mC_1^{(\rm R)} \in \R^{1 \times (n_1  r_1)} $. Thus, the thin QR decomposition $\mQ_1 \mR_1 = \big( \mC_1^{(\rm R)} \big)^\top$ produces a scalar $\mR_1$ such that $\mR_1^2 = \|\mC_1^{(\rm R)}\|^2$, and then the normalising constant $\hat{z} = \int_{\mathcal{X}_1} \hat{\ptar}_{ \leq 1}(\bx_{1}) \lambda_1(\bx_1) d\bx_1$ can be obtained by $ \hat{z} = \gamma \prod_{i = 1}^d \lambda_i(\mathcal{X}_i) + \mR_1^2$.
\end{proposition}
\begin{proof}
See Appendix \ref{appen:prop:sirt_recur}.
\end{proof}

\begin{proposition}\label{prop:sirt_cond_cdf}
The marginal PDF of $\hat{\bX}_1$ can be expressed as
\begin{equation}
\hat{f}_{\hat{\bX}_1}(\bx_1) = \frac{1}{\hat{z}} \bigg( \gamma \prod_{i = 2}^d \lambda_i(\mathcal{X}_i) + \sum_{\ell_{1}=1}^{r_{1}} \Big(\sum_{i =1}^{n_1}\phi_{1}^{(i)}(\bx_1) \, \mD_1 [i, \ell_1] \Big)^2 \bigg) \lambda_1(\bx_1),
\end{equation}
where $\mD_1[i, \ell_1]=\tB_1 [\alpha_{0}, i, \ell_1]$ for $i = 1, \ldots, n_1$ and $\alpha_0 = 1$. For $k>1$ and a given $\bx_{<k}$, the conditional PDF of $\hat{\bX}_k | \hat{\bX}_{<k}$ can be expressed as
\begin{equation}
\hat{f}_{\hat{\bX}_k | \hat{\bX}_{< k}}(\bx_k | \bx_{<k}) = \frac{1}{\hat{\ptar}_{ < k}(\bx_{< k})} \bigg( \gamma \!\!\! \prod_{i = k+1}^d \!\!\! \lambda_i(\mathcal{X}_i) + \sum_{\ell_{k}=1}^{r_{k}} \Big(\sum_{i =1}^{n_k}\phi_{k}^{(i)}(\bx_k) \, \mD_k [i, \ell_k] \Big)^2  \bigg) \lambda_k(\bx_k),
\end{equation}
where $\mD_k \in \R^{n_k \times r_k}$ is given by 
\[
\mD_k[i, \ell_k] = \sum_{\alpha_{k-1} = 1}^{r_{k-1}} \mG^{(\alpha_{k-1})}_{<k}(\bx_{<k})\tB_k [\alpha_{k-1}, i, \ell_k].
\]
\end{proposition}
\begin{proof}
The above results directly follow from the definition of conditional PDF and the marginal functions in \eqref{eq:marginal_sq1} and \eqref{eq:marginal_sqk}. 
\end{proof}

Note that the product $\mG_{1}(\bx_{1}) \cdots \mG_{\km}(\bx_{\km})$ requires $k-1$ univariate interpolations and $k-2$ products of matrices per sample,
that is the same operations as in the standard inverse Rosenblatt transport.
The QR decomposition~\eqref{eq:SIRT-QR} and the construction of the coefficient tensors~\eqref{eq:B_recur} need $\mathcal{O}(dnr^3)$ operations, but these are pre-processing steps that are independent of the number of samples.
However, in contrast to the vector-valued function $\mathcal{F}_k(\bx_k)\bar{\mathcal{F}}_{k+1} \cdots \bar{\mathcal{F}}_{d} \in \R^{r_{k-1}}$, in evaluating the PDF $\hat{f}_{\hat{\bX}}$, we need to multiply the matrix-valued function $\sqbasis_{k}(\bx_{k}) \in \R^{ r_{k-1} \times r_{k} }$ for each sample.
Thus, the leading term of the complexity becomes $\mathcal{O}(Ndnr^2)$, one order of $r$ or $n$ higher than the complexity of the standard inverse Rosenblatt transport.
However, for small $r$ and $n$ this is well compensated by a smoother map, which will be crucial in Section~\ref{sec:dirt}.

\subsection{Implementation of CDFs}\label{sec:sirt_cdf}
To evaluate SIRT, one has to first construct the marginal CDF of $\hat{\bX}_1$ and the conditional CDFs of $\hat{\bX}_k | \hat{\bX}_{<k}$ for $k > 1$, and then inverts the CDFs (see \eqref{eq:inverse_trans}). 
Here we discuss the computation and the inversion of CDFs, which are based on pseudo-spectral methods, for problems with bounded domains and extensions to problems with unbounded domains. 
We refer the readers to \cite{boyd2001chebyshev,shen2011spectral,trefethen2019approximation} and references therein for a more details.

\subsubsection{Bounded domain with polynomial basis}
For a bounded parameter space $\mathcal{X}\subset\R^d$, we consider the weighting function $\lambda(\bx) = 1$. 
Since $\mathcal{X}$ can be expressed as a Cartesian product, without loss of generality, here we discuss the CDF of a one-dimensional random variable $Z$ with the PDF
\begin{equation}\label{eq:one_pdf}
\hat{f}_{Z}(\zeta) = C + \sum_{\ell=1}^{r} \Big(\sum_{i =1}^{n}\phi^{(i)}(\zeta) \, \mD [i, \ell] \Big)^2, 
\end{equation}
where $C>0$ is some constant, $\{\phi^{(i)}(\zeta)\}_{i = 1}^{n}$ are the basis functions, $\mD \in \R^{n \times r}$ is a coefficient matrix, and $\zeta \in [-1, 1]$.
Here $\hat{f}_{Z}(\zeta)$ can be either the marginal PDF or the conditional PDFs defined in Proposition \ref{prop:sirt_cond_cdf} with a suitable linear change of coordinate.

We first consider a polynomial basis, $\phi^{(i)}(z) \in \mathbb{P}_{n-1}$ for $i = 1, \ldots, n$, where $\mathbb{P}_{n-1}$ is a vector space of polynomials of degree at most $n-1$ defined on $ [-1, 1]$. Thus, the PDF $\hat{f}_{Z}(\zeta)$ can be represented exactly in $\mathbb{P}_{2n-2}$.
To enable fast computation of the CDF, we choose the Chebyshev polynomials of the second kind
\[
p_m(\zeta) = \frac{\sin\big( (m+1)\cos^{-1}(\zeta) \big)}{\sin\big(\cos^{-1}(\zeta)\big)}, \quad m = 0, 1, \ldots, 2n-2,
\]
as the basis of $\mathbb{P}_{2n-2}$. Using the roots of $p_{2n-1}(\zeta)$, we can define the set of collocation points 
\[
\big\{\zeta_m \big\}_{m=1}^{2n-1}, \quad \textrm{where} \quad \zeta_m = \cos\Big(\frac{m \pi} {2n}\Big). 
\] 
This way, by evaluating $\hat{f}_{Z}(\zeta)$ on the collocation points, which needs $\mathcal{O}(nr)$ operations, one can apply the collocation method \cite[Chapter 4]{boyd2001chebyshev} to represent $\hat{f}_{Z}(\zeta)$ using the Chebyshev basis:
\begin{equation}\label{eq:cheby_pdf}
\hat{f}_{Z}(\zeta) = \sum_{m=0}^{2n-2} a_m \, p_m(\zeta) ,
\end{equation}
where the coefficients $\{a_m\}_{m=1}^{2n-2}$ can be computed by the fast Fourier transform 
with $\mathcal{O}(n\log(n))$ operations. Then, one can express the CDF of $Z$ as
\begin{equation}\label{eq:cheby_cdf}
F_Z(\zeta) = \int_{-1}^\zeta \hat{f}_{Z}(\zeta^\prime) d\zeta^\prime = \sum_{m=0}^{2n-2} \frac{a_m}{m+1} \big(t_{m+1}(\zeta) - t_{m+1}(-1) \big),
\end{equation}
where $t_m(\zeta) = \cos\big( m \cos^{-1}(\zeta) \big)$ is the Chebyshev polynomial of the first kind of degree $m$. 
A random variable $Z$ can be generated by drawing a uniform random variable $U$ and evaluating $Z = F_Z^{-1}(U)$ by solving the root finding problem $F_Z(Z) = U$.

\begin{remark}
The PDF in \eqref{eq:one_pdf} is positive for all $\zeta \in [-1, 1]$ by construction and can be represented exactly in $\mathbb{P}_{2n-2}$ with the polynomial basis. Thus, its Chebyshev representation in \eqref{eq:cheby_pdf} is also positive. This way, the resulting CDF in \eqref{eq:cheby_cdf} is monotone, and thus the solution to the inverse CDF equation, $F_Z(Z) = U$, admits a unique solution. 
\end{remark}

\begin{remark}
One can also employ piecewise Lagrange polynomials as a basis to enable hp-adaptivity. With piecewise Lagrange polynomials, the above-mentioned technique can also be used to obtain the piecewise definition of the CDF.
\end{remark}

Since $F_Z(Z) = U$ has a unique solution and $F_Z$ is monotone and bounded between $[0,1]$, it requires usually only a few iterations to apply the root finding methods, such as the regula falsi method and the Newton's method, to solve $F_Z(Z) = U$ with an accuracy close to machine precision.
Overall, the construction of the CDF needs $\mathcal{O}(nr + n\log(n))$ operations, and the inversion of the CDF function needs $\mathcal{O}(c n )$ operations, where $\mathcal{O}(n)$ is the cost of evaluating the CDF and $c$ is the number of iterations required by the root finding method. In comparison, building the matrix $\mD$ requires $\mathcal{O}(nr^2)$ operations (cf. Proposition \ref{prop:sirt_cond_cdf}).

\subsubsection{Bounded domain with Fourier basis}
If the Fourier transform of the PDF of $Z$, which is the characteristic function, is band-limited in the frequency domain, then one may choose the sine and cosine Fourier series as the basis for representing the PDF in \eqref{eq:one_pdf}. 
In this case, the above strategy can also be applied. Recall the Fourier basis with an even cardinality $n$,
\[
\big\{ 1, \ldots, \sin(m \pi \zeta), \cos (m \pi \zeta), \ldots, \cos(n \pi \zeta / 2) \big\}, \quad m = 1, \ldots, n/2-1,
\]
which consists of $n/2-1$ sine functions and $n/2+1$ cosine functions. The PDF $\hat{f}_{Z}(\zeta)$ defined in \eqref{eq:one_pdf} yields an exact representation using the Fourier basis with cardinality $2n$.
This way, one can represent $\hat{f}_{Z}(\zeta)$ as
\[
\hat{f}_{Z}(\zeta) = a_0 + \sum_{m=1}^{n} a_m \cos(m \pi \zeta)  + \sum_{m=1}^{n-1} b_m \sin(m \pi \zeta),
\]
where the coefficients, $a_m$ and $b_m$, are obtained by evaluating $\hat{f}_{Z}(\zeta)$ on the collocation points
\[
\big\{\zeta_m \big\}_{m=1}^{2n}, \quad \textrm{where} \quad \zeta_m =  \frac{m}{n} - 1,
\]
and applying the rectangular rule. This leads to the CDF
\begin{align*}
F_Z(\zeta) & = \int_{-1}^\zeta \hat{f}_{Z}(\zeta^\prime) d\zeta^\prime \\
& = a_0(\zeta + 1) + \sum_{m=1}^{n} \frac{a_m}{m\pi} \sin(m \pi \zeta)  - \sum_{m=1}^{n-1} \frac{b_m}{m\pi} \big( \cos(m \pi \zeta) - \cos(m \pi ) \big).
\end{align*}
The construction and the inversion of the CDF using the Fourier basis cost a similar amount of operations compared to the polynomial basis.

\subsubsection{Unbounded domain}

Given an unbounded domain, the simplest approach is to truncate the domain at the tail of the PDF.  With the domain truncation, the above-mentioned implementations based on Chebyshev and Fourier basis can be applied directly. Although the function approximation error induced by the domain truncation can be bounded, using the resulting SIRT for computing expectations may lead to a biased estimator.

One can also consider basis functions that are intrinsic to an unbounded domain. For the domain $\mathcal{X}_k = (0, \infty)$, one can employ the Laguerre polynomials as the basis. This equips $\mathcal{X}_k$ with a natural exponential weighting function $\lambda_k(\bx_k) = \exp(-\bx_k)$. The collocation method using higher order Laguerre polynomials can be applied again to obtain the exact representation of the CDF.
Similarly, for $\mathcal{X}_k = (-\infty, \infty)$, the Hermite polynomials can be used as a basis, which equips $\mathcal{X}_k$ with a Gaussian weighting function $\lambda_k(\bx_k) = \exp(- \frac12 \bx_k^2)$. Although one can apply the collocation method to obtain an algebraically exact representation of the CDF, the resulting CDF involves error functions, complementary error functions, and imaginary error functions. Those functions have to be approximated numerically. Thus, the computational cost of computing the CDF can be high, and it may be hard to guarantee the monotonicity and uniqueness of the inverse CDF solution at the tails.
Using other bases such as the Whittaker cardinal functions for $\mathcal{X}_k = (-\infty, \infty)$ may face a similar challenge. 

\begin{remark}
In a situation where the squared form of the PDF in \eqref{eq:one_pdf} can be computed but it is challenging to invert the CDF function, one can employ the rejection sampling \cite{robert2013monte} to generate random variables. In this situation, our TT approximation can still be used to draw conditional samples. However, this approach may not lead to the deterministic inverse Rosenblatt transport. 
\end{remark}

\subsubsection{Change of coordinate}
One can also apply a diffeomorphic mapping to change the coordinate of an unbounded domain $\mathcal{X}$ to a bounded one, e.g., $\mathcal{Z} = [-1, 1]^d$, followed by application of the Chebyshev polynomials or Fourier series.
Given a PDF $f_{\bX}(\bx)$ of a random variable $\bX \in \mathcal{X}$, suppose we have a diffeomorphic mapping $R: \mathcal{X} \mapsto \mathcal{Z}$ and let $q(\bx) = |\nabla R(\bx)| \geq 0$. For any Borel set $\mathcal{B}_X \subseteq \mathcal{X}$, we have
\begin{align*}
\mathbb{P} [\bX \in \mathcal{B}_X] & = \int_{\mathcal{B}_X} f_{\bX}(\bx) d\bx \\
& = \int_{\mathcal{B}_Z}  (f_{\bX}\circ R^{-1})(\bze)\,\big|\nabla R^{-1}\!(\bze)\big|  d\bze = \int_{\mathcal{B}_Z}  \frac{ (f_{\bX}\circ R^{-1})(\bze) }{(q \circ R^{-1})(\bze)}  d\bze,
\end{align*}
where $\mathcal{B}_Z = R(\mathcal{B}_X)$. 
Thus, one can draw a random variable $\bZ$ with PDF 
\[
f_{\bZ}(\bze) = \frac{ (f_{\bX}\circ R^{-1})(\bze) }{(q \circ R^{-1})(\bze)},
\]
and apply the mapping $\bX = R(\bZ)$ to obtain a random variable $X$. With the change of the coordinate $\bze = R(\bx)$, one needs to build a TT to approximate $\sqrt{f_{\bZ}(\bze)}$ and construct the corresponding SIRT to simulate  the random variable $\bZ$. To avoid singularities at the boundary of $\mathcal{Z}$, one can choose a mapping $R$ such that the function $q(\bx)$ decays slower than $f_{\bX}(\bx)$.

\subsection{SIRT error}
Since SIRT enables us to generate i.i.d. samples from the probability measure $\hat{T}_\sharp  \muni$, it can be used to define either Metropolis independence samplers or importance sampling schemes.
Based on certain assumptions on the TT approximation $\tilde{g}$, here we establish error bounds for the TV distance, the Hellinger distance, and the $\chisq$-divergence of the target measure $\mtar$ from $\hat{T}_\sharp  \muni$. These divergences play a vital role in analysing the convergence of Metropolis--Hastings methods and the efficiency of importance sampling. 
The error analysis also provides a heuristic for choosing the constant $\gamma$ in the approximate target density \eqref{eq:pdf_sirt}.

\begin{proposition}\label{prop:sirt_l2}
Recall the approximate target density $\hat \pi(\bx) = \gamma  + \tilde g(\bx)^2$, where $\tilde g(\bx)$ is a TT approximation to the square root of the unnormalised target density $\sqrt{\pi}$.
Suppose the error of $\tilde{g}$ and the constant $\gamma$ satisfy
\begin{equation}\label{eq:l2_error}
\lpnormx{\tilde{g} - \sqrt{\pi}}{2} \leq \epsilon \quad {\rm and} \quad \gamma \leq \frac{1}{\lambda(\mathcal{X})}\lpnormx{\tilde{g} - \sqrt{\pi}}{2}^2,
\end{equation}
respectively.
Then, the error of $\sqrt{\hat \pi}$ satisfies $\lpnormx{\sqrt{\pi} - \sqrt{\hat{\pi}}}{2} \leq  \sqrt{2}\epsilon$.
\end{proposition}
\begin{proof}
Applying the identity
\(
\big( \sqrt{\pi} - \sqrt{\gamma + \tilde{g}^2} \big)^2 \leq \big( \sqrt{\pi} - \tilde{g} \big)^2 + \gamma,
\)
we have
\begin{align*}
\lpnormx{\sqrt{\pi} - \textstyle\sqrt{\hat{\pi}}}{2} & \leq \bigg(\int_\mathcal{X} \Big( \sqrt{\pi(\bx)} - \tilde{g}(\bx) \Big)^2 \lambda(\bx)\,d\bx + \gamma \,\lambda(\mathcal{X}) \bigg)^\frac12 \leq \sqrt{2} \epsilon.
\end{align*}
\end{proof}

\begin{proposition}\label{prop:sirt_z}
Suppose the conditions in \eqref{eq:l2_error} hold. Then, the approximate normalising constant $\hat{z}$ satisfies $\big| \sqrt{z} -  \sqrt{\hat z} \big| \leq  \sqrt{2} \epsilon$.
\end{proposition}
\begin{proof}
The normalising constants $z$ and $\hat{z}$ satisfy
\begin{align}
\big| z -  \hat{z} \big| = \Big| \int_\mathcal{X}\! \big( \pi(\bx) - \hat{\pi}(\bx) \big) \,\lambda(\bx)\,d\bx \Big| \leq \lpnormx{\pi -\hat{\pi} }{1}.\label{eq:z_ineq1}
\end{align}
Applying the H\"{o}lder's inequality (with $p = q = 2$) and the Minkowski inequality, the right hand side of the above inequality also satisfies
\begin{align}
\lpnormx{\pi - \hat{\pi} }{1} & =  \lpnormx{ \big( \sqrt{\pi} - \textstyle\sqrt{\hat{\pi}}\big) \big( \sqrt{\pi} + \textstyle \sqrt{\hat{\pi}}\big)}{1} \nonumber \\
& \leq \lpnormx{ \sqrt{\pi} - \textstyle\sqrt{\hat{\pi}}}{2} \lpnormx{\sqrt{\pi} + \textstyle\sqrt{\hat{\pi}}}{2}  \nonumber \\
& \leq \lpnormx{\sqrt{\pi} - \textstyle\sqrt{\hat{\pi}}}{2} \; \Big( \lpnormx{\sqrt{\pi} }{2} + \lpnormx{\textstyle\sqrt{\hat{\pi}}}{2} \Big) \nonumber \\
& = \lpnormx{\sqrt{\pi} - \textstyle\sqrt{\hat{\pi}}}{2} \; \big( \sqrt{z} + \sqrt{\hat z} \big) .\label{eq:z_ineq2}
\end{align}
Since both $z$ and $\hat{z}$ are positive, we have $\big| z -  \hat{z} \big| = \big| \sqrt{\hat z} -\sqrt{z}  \big|\; \big( \sqrt{z} + \sqrt{\hat z}  \big)$. Substituting this identity and the inequality in \eqref{eq:z_ineq2} into \eqref{eq:z_ineq1}, we have
\(
\big| \sqrt{z} -  \sqrt{\hat z} \big| \leq \lpnormx{\sqrt{\pi} - \sqrt{\hat{\pi}}}{2} .
\)
Thus, the result follows from Proposition \ref{prop:sirt_l2}.
\end{proof}


\begin{theorem}\label{thm:sirt_h}
Suppose the conditions in \eqref{eq:l2_error} hold. 
The Hellinger distance between $\mtar$ and $\hat{T}_\sharp  \muni$ satisfies $D_{\rm H}(\nu_\ptar\| \hat{T}_\sharp  \muni ) \leq 2\epsilon \big/ \sqrt{z}$.
\end{theorem}

\begin{proof}
Since the target measure $\nu_\ptar$ and the approximate measure $\hat{T}_\sharp  \muni$ respectively have the densities $\frac1z \, \pi(\bx)\,\lambda(\bx)$ and $\frac1{\hat z}\,\hat{\pi}(\bx)\,\lambda(\bx)$, the squared Hellinger distance satisfies
\begin{align*}
D_{\rm H}^2\big(\nu_\ptar\| \hat{T}_\sharp  \muni \big) = \frac12 \int_\mathcal{X} \bigg(\sqrt{ \frac{\pi(\bx) }{z} } - \sqrt{\frac{\hat\pi(\bx)}{\hat z}} \bigg)^2 \, \lambda(\bx)\,d\bx .
\end{align*}
This leads to the inequality
\begin{align*}
D_{\rm H}^2\big(\nu_\ptar\| \hat{T}_\sharp  \muni \big) &=  \frac1{2z}\lpnormx{ \sqrt{\pi} - \textstyle\sqrt{\hat\pi} + \sqrt{\hat\pi} - \textstyle\sqrt{\hat\pi} \, \sqrt{z/\hat{z}} }{2}^2 \nonumber \\
& \leq \frac{1}{2z} \Big( \lpnormx{ \sqrt{\pi} - \textstyle\sqrt{\hat\pi} }{2} + \lpnormx{ \textstyle\sqrt{\hat\pi} }{2}\,\big|1 - \sqrt{z/\hat{z}}\,\big| \Big)^2 .
\end{align*}
Applying $\lpnormx{ \sqrt{\hat\pi} }{2}^2 = \hat{z}$ and Propositions \ref{prop:sirt_l2} and \ref{prop:sirt_z}, the above inequality can be further reduced to
\[
D_{\rm H}^2\big(\nu_\ptar\| \hat{T}_\sharp  \muni \big) \leq \frac{1}{2z} \Big( \sqrt{2} \epsilon + \big|\sqrt{\hat z} - \sqrt{z}\big| \Big)^2 \leq \frac{4\epsilon^2}{z} .
\]
Thus, we have $D_{\rm H}\big(\nu_\ptar\| \hat{T}_\sharp  \muni \big) \leq  2\epsilon \big/ \sqrt{z}$.
\end{proof}


\begin{corollary}\label{coro:sirt_tv}
Suppose the conditions in \eqref{eq:l2_error} hold. The total variation distance between $\mtar$ and $\hat{T}_\sharp  \muni$ satisfies $D_{\rm TV}(\mtar \| \hat{T}_\sharp  \muni ) \leq  2 \sqrt{2} \epsilon \big/ \sqrt{z}$.
\end{corollary}
\begin{proof}
The result directly follows from the inequality $D_{\rm TV} \leq \sqrt{2} D_{\rm H}$ and Theorem \ref{thm:sirt_h}.
\end{proof}

\begin{proposition}\label{prop:exp_hell}
Given two probability measures $\mtar$ and $\hat{\nu}_{\hat{\pi}}$ and a function $h$ with finite second moments with respect to  $\mtar$ and $\hat{\nu}_{\hat{\pi}}$. Then
\[
\big| \mtar(h) - \hat{\nu}_{\hat{\pi}}(h) \big| \leq \sqrt{2}\,  \big( \mtar(h^2)^\frac12 + \hat{\nu}_{\hat{\pi}} (h^2)^\frac12 \big) \, D_{\rm H}\big(\nu_\ptar \| \hat{\nu}_{\hat{\pi}} \big).
\]
\end{proposition}

\begin{proof}
Suppose $\mtar$ and $\hat{\nu}_{\hat{\pi}}$ respectively have density functions $f_{\bX}(\bx)$ and $\hat{f}_{\hat \bX}(\bx)$ with respect to  the Lebesgue measure. We have the following inequality
\begin{align*}
& \hspace{-24pt} \big| \mtar(h) - \hat{\nu}_{\hat{\pi}} (h) \big| \\
& = \bigg| \int_\mathcal{X} h(\bx) \big( f_{\bX}(\bx) -\hat{f}_{\hat \bX}(\bx) \big) d\bx \, \bigg| \\
& = \bigg| \int_\mathcal{X} h(\bx) \Big( f_{\bX}(\bx)^\frac12 + \hat{f}_{\hat \bX}(\bx)^\frac12 \Big) \Big(f_{\bX}(\bx)^\frac12 - \hat{f}_{\hat \bX}(\bx)^\frac12 \Big) d\bx \, \bigg| \\
& \leq \bigg( \!\int_\mathcal{X}\! \Big(h(\bx)f_{\bX}(\bx)^\frac12 + h(\bx)\hat{f}_{\hat \bX}(\bx)^\frac12 \Big)^2 d\bx \bigg)^\frac12
\bigg( \!\int_\mathcal{X\!} \Big( f_{\bX}(\bx)^\frac12 - \hat{f}_{\hat \bX}(\bx)^\frac12 \Big)^2 d\bx\bigg)^\frac12 \\
& \leq \sqrt{2}\,\bigg(\Big( \int_\mathcal{X}  h(\bx)^2 \, f_{\bX}(\bx)\, d\bx \Big)^\frac12 + \Big(\int_\mathcal{X} h(\bx)^2 \, \hat{f}_{\hat \bX}(\bx) \, d\bx \Big)^\frac12\bigg)
D_{\rm H}\big(\nu_\ptar \| \hat{\nu}_{\hat{\pi}} \big).
\end{align*}
Thus, the result follows. 
\end{proof}

\begin{corollary}\label{coro:sirt_chi}
Suppose the conditions in \eqref{eq:l2_error} hold. Suppose further the bound in \eqref{eq:ratio_bound} holds.
Then, the $\chisq$-divergence of $\nu_\ptar$ from $\hat{T}_\sharp  \muni$ satisfies 
\begin{equation*}%
D_{\chisq}\big(\nu_\ptar \| \hat{T}_\sharp  \muni\big) \leq \Big( \mtar\big(\pi^2\big/ {\hat\pi}^2\big)^\frac12 + \big(\hat{T}_\sharp \muni\big)\big(\pi^2 \big/{\hat\pi}^2\big)^\frac12 \Big)  \, \frac{2 \sqrt{2} \hat{z}}{z\sqrt{z}}\, \epsilon.
\end{equation*}
\end{corollary}

\begin{proof}
Given the bound in \eqref{eq:ratio_bound}, the ratio between $f_{\bX}(\bx)$ and $\hat{f}_{\hat{\bX}}(\bx)$ satisfies
\begin{equation*}
\sup_{\bx \in \mathcal{X}} \frac{f_{\bX}(\bx)}{\hat{f}_{\hat{\bX}}(\bx)} = \frac{\hat{z}}{z} \sup_{\bx \in \mathcal{X}} \frac{\pi(\bx)}{\hat\pi(\bx)}  =  \frac{\hat{z}}{z} \hat{c} < \infty.
\end{equation*}
This way, $\nu_\ptar$ is absolutely continuous with respect to $\hat{T}_\sharp  \muni$. Thus, the $\chisq$-divergence of $\nu_\ptar$ from $\hat{T}_\sharp  \muni$ can be expressed as
\begin{align*}
D_{\chisq}\big(\nu_\ptar \| \hat{T}_\sharp  \muni\big) & = \int_\mathcal{X} \bigg( \frac{f_{\bX}(\bx)}{\hat{f}_{\hat{\bX}}(\bx)} \bigg)^2  \hat{f}_{\hat{\bX}}(\bx) d\bx - 1 \\
& = \mtar\big(f_{\bX}\big/\hat{f}_{\hat{\bX}} \big) - \big(\hat{T}_\sharp \muni\big)\big(f_{\bX} \big/ \hat{f}_{\hat{\bX}}\big) \\
& = \frac{\hat{z}}{z}\;\Big( \mtar\big(\pi \big/ {\hat\pi} \big) - \big(\hat{T}_\sharp \muni\big)\big(\pi\big/ {\hat\pi} \big) \Big).
\end{align*}
The bound in \eqref{eq:ratio_bound} also implies that the expectations $\mtar({\pi^2}/{\hat\pi}^2)$ and $(\hat{T}_\sharp \muni)({\pi^2}/{\hat\pi}^2)$ are finite. Therefore, the result follows from Proposition \ref{prop:exp_hell}. 
\end{proof}

\begin{remark}
The TV distance, Hellinger distance, and $\chisq$-divergence of the target measure $\mtar$ from the pushforward of the reference $\muni$ under the SIRT $\hat{T}$ are linear in the approximation error of the TT.
Note that $\sqrt{z} = \lpnormx{\sqrt{\ptar}}{2}$. Suppose the relative error of $\tilde{g}$ satisfies
\[
\lpnormx{\tilde{g} - \sqrt{\pi}}{2} \Big/ \lpnormx{\sqrt{\ptar}}{2}  \leq \tau, 
\] 
the TV distance, Hellinger distance, and $\chisq$-divergence of $\mtar$ from $\hat{T}_\sharp \muni$ are bounded by $\mathcal{O}(\tau)$.
\end{remark}


\newcommand{\map}[1]{T_{#1}}
\newcommand{\amap}[1]{\hat{T}_{#1}}
\newcommand{\cmap}[1]{\bar{T}_{#1}}
\newcommand{\tmap}[1]{\tilde{T}_{#1}}
\newcommand{\gk}[1]{\tilde{g}_{#1}}
\newcommand{\rk}[2]{r_{#1,#2}}

\section{Deep inverse Rosenblatt transport}\label{sec:dirt}

In many practical applications, probability densities can be concentrated to a small region of the parameter space or have complicated correlation structures. For example, posterior densities in Bayesian inference problems with informative data often occupy a relatively small region of the parameter space and demonstrate complicated nonlinear interactions in some sub-manifold, see \cite {ROM:CMW_2014, ROM:CMW_2016, MCMC:GiCal_2011, parno2018transport} for detailed examples.
In this situation, straightforward approximation of a complicated density function in a TT decomposition may require rather large ranks. As the number of function evaluations needed in constructing TT decompositions grows quadratically with the ranks, such direct factorisation of the target densities with complicated structures may become infeasible.

\begin{example}
Consider a $d$-dimensional multivariate normal distribution with the unnormalised density $\pi(\bx) = \exp(-\frac{1}{2}\bx^\top \mC^{-1} \bx)$.
If the covariance matrix $\mC\in\R^{d \times d}$ is diagonal, the joint density factorises into a product of marginal densities, that is a TT decomposition with ranks~$1$.
This corresponds to zero-rank \emph{off-diagonal} blocks $\mC[1:k,~(k+1):d]$, $k=1,\ldots,d-1$.
However, the TT ranks of a correlated normal density $\pi(\bx)$
may grow \emph{exponentially} in the rank of the off-diagonal blocks of $\mC$~\cite{rdgs-tt-gauss-2020}.
\end{example}

We design a DIRT framework to construct a \emph{composition} of order-preserving mappings in the SIRT format that can characterise concentrated probability densities with complicated correlation structures.
The construction of DIRT is guided by a sequence of bridging probability measures ${\nu}_{0}$, ${\nu}_{1}, \ldots$, ${\nu}_{L}$, where $\nu_L = \mtar$ is the target measure. Each bridging measure $\nu_k$ has the corresponding PDF
\begin{equation}\label{eq:densities}
f_{\bX^k}(\bx) = \frac{1}{z_k} \,\pi_k(\bx)\,\lambda(\bx), \quad \textrm{where} \quad z_k = \int_\mathcal{X} \pi_k(\bx) \,\lambda(\bx)\,d\bx.
\end{equation}
Here $\pi_0(\bx)$ is the unnormalised initial density such that $\sup_{\bx \in \mathcal{X}} \pi_0(\bx) < \infty$, $\pi_L(\bx) = \pi(\bx)$ is the unnormalised target density, and the superscript $k$ indexes the random variable $\bX^k \in \mathcal{X}$, $\bX^k \sim \nu_k$.
Our goal is to construct a composition of mappings $T_0 \circ T_1 \circ \cdots \circ T_k$ such that the pushforward of the reference measure under this composition matches the $k$-th bridging measure, i.e, $(T_0 \circ T_1 \circ \cdots \circ T_k)_\sharp \, \mref = \nu_k$.
This way, by gradually increasing the complexity in the geometry and/or the computational cost of the densities of the bridging measures, it becomes computationally feasible to construct TT and the corresponding SIRT at each layer of the composition.

\begin{assumption}\label{assum:ratio}
Denoting the ratio between two unnormalised densities by 
\begin{equation}\label{eq:ratio}
\rk{k}{j}(\bx)=\frac{\pi_k( \bx )}{\pi_j( \bx )},
\end{equation}
we assume that for each pair of $j < k$, the ratio $\rk{k}{j}(\bx)$ is finite such that 
\begin{equation}\label{eq:ratio_b}
\sup_{x\in\mathcal{X}}\rk{k}{j}(\bx) = c_{k,j} < \infty, \quad \forall j < k.
\end{equation}
\end{assumption}

In practice, there are many ways to choose the bridging measures. For example, one can consider \emph{tempered} distributions \cite {gelman1998simulating,hukushima1996exchange,meng1996simulating,neal1996sampling,swendsen1986replica} where $\pi_k(\bx) = \pi(\bx)^{\beta_k}$ for a suitable chosen set of powers (\emph{reciprocal temperatures}) $0 \leq \beta_0 < \cdots < \beta_L=1$; and for problems involving computationally expensive PDE models, one can employ a hierarchy of models with different grid resolutions to reduce the computational cost for building the DIRT.
In the rest of this section, we will present the recursive construction of DIRT and provide  error analysis.

\subsection{Recursive construction}
In the initial step ($k = 0$), we compute a TT $\gk{0}(\bx)$ that approximates $\sqrt{\pi_0}$ and construct the corresponding SIRT $\hat{\bX}^\lowersup{0} = \amap{0}(\bU)$ so that the reference random variable $\bU \sim \muni$ and $\hat{\bX}^\lowersup{0} \sim (\amap{0})_\sharp  \muni$ with the PDF
\begin{equation}\label{eq:density0}
\hat{f}_{\hat{\bX}^0}(\bx) = \frac{1}{\hat{z}_0} \big(\gamma_0 + \gk{0}(\bx)^2\big)\,\lambda(\bx), \quad \textrm{with} \quad \hat{z}_0 = \gamma_0 \,\lambda(\mathcal{X}) + \int_\mathcal{X} \gk{0}(\bx)^2 \, \lambda(\bx)\, d\bx,
\end{equation}
where the constant $\gamma_0$ is chosen such that $0< \gamma_0 \leq \lambda(\mathcal{X})^{-1}\lpnormx{\sqrt{\pi_0} - \tilde{g}_0}{2}^2$. Note that $\hat{f}_{\hat{\bX}^0}(\bx)$ is an approximation to $f_{\bX^0}(\bx)$. 

\begin{remark}
We can replace the uniform reference measure $\muni$ with a general product-form probability measure $\mref$ that has the PDF $\pref(\bu) = \prod_{k = 1}^d  f_{\bU_k}(\bu_k)$ with support in $\mathcal{U} = \mathcal{U}_1 \times \mathcal{U}_2 \times \cdots \times \mathcal{U}_d$. One can construct a mapping 
\[
R(\bu) = \big[F_{\bU_1}(\bu_1), \ldots, F_{\bU_k}(\bu_k),\ldots, F_{\bU_d}(\bu_d)\big]^\top,
\] 
where $F_{\bU_k}(\bu_k)$ is the CDF of $\bU_k$, such that $R_\sharp \,\mref = \muni$. Then, the composition of mappings  $\amap{0} \circ R: \mathcal{U} \mapsto \mathcal{X}$ is lower-triangular and $(\amap{0} \circ R)_\sharp \, \mref$ has the density $\hat{f}_{\hat{\bX}^0}(\bx)$. 
We initialise the DIRT by $\cmap{0} = \amap{0}\circ R$, where $R$ is the identity map if $\mref = \muni$.
\end{remark}

After $k>1$ steps, suppose we have the $k$-th DIRT given as a composition of mappings
\[
\cmap{k} = (\amap{0} \circ R) \circ (\amap{1}\circ R) \circ \cdots \circ (\amap{k} \circ R),
\]
where each $\amap{j}$ is a SIRT. 
Denoting the pushforward of the reference probability measure $\mref$ under $\cmap{k}$ by $\hat{\nu}_k$, i.e., $\hat{\nu}_k \equiv (\cmap{k})_\sharp \, \mref$, and the density function of $\hat{\nu}_k$ by  $\hat{f}_{\hat{\bX}^k}(\bx)$, the pullback density of $\hat{\nu}_k$ under $\cmap{k}$ satisfies
\begin{align}\label{eq:pullback_k0}
\cmap{k}^\sharp\hat{f}_{\hat{\bX}^k}(\bu) & = (\hat{f}_{\hat{\bX}^k} \circ \cmap{k} ) (\bu)  \, \big|\nabla_{\bu} \cmap{k}(\bu) \big| = \pref(\bu).
\end{align}
The density function of the pullback probability measure $\cmap{k}^\sharp \hat{\nu}_k$ is the reference product density $\pref(\bu)$.
Suppose the corresponding approximate PDF $\hat{f}_{\hat{\bX}^k}(\bx)$ can capture the range of variation and the correlation structure of the next PDF $f_{\bX^{k+1}}(\bx)$, then the density function of the pullback probability measure $\cmap{k}^\sharp \nu_{k+1}$,
\begin{equation}\label{eq:pullback_k1}
\cmap{k}^\sharp f_{\bX^{k+1}}(\bu) = (f_{\bX^{k+1}} \circ \cmap{k} ) (\bu)\, \big|\nabla_{\bu} \cmap{k}(\bu) \big|,
\end{equation}
may become easier to factorise in the TT format compared to the direct factorisation of the original target density function $f_{\bX^{k+1}}(\bx)$.
This way, for step $k+1$, the existing composition $\cmap{k}$ can be used to precondition the construction of the coupling between $\mref$ and $\nu_{k+1}$: by building a coupling between the pullback measure $\cmap{k}^\sharp\nu_{k+1}$ and the reference measure
\[
\bU^{k+1} = (\map{k+1}\circ R)(\bU), \quad \textrm{where} \quad \bU \sim \mref \quad \textrm{and} \quad \bU^{k+1} \sim \cmap{k}^\sharp\nu_{k+1},
\]
one can obtain a new composition of maps $( \cmap{k} \circ \map{k+1} \circ R)$ such that
\[
( \cmap{k} \circ \map{k+1} \circ R )^\sharp \nu_{k+1} = \mref \quad \textrm{or} \quad ( \cmap{k} \circ \map{k+1} \circ R)_\sharp  \mref = \nu_{k+1}.
\]

We use SIRT to approximate $\map{k+1}$. Using \eqref{eq:pullback_k0}, we have $\big|\nabla_{\bu} \cmap{k}(\bu) \big| = \big( (\hat{f}_{\hat{\bX}^k} \circ \cmap{k} ) (\bu) \big)^{-1} \, \pref(\bu)$. Thus, the pullback density in \eqref{eq:pullback_k1} can be expressed as a \emph{ratio function}
\begin{align} \label{eq:pullback_ratio}
(\cmap{k})^\sharp f_{\bX^{k+1}}(\bu) = \frac{(f_{\bX^{k+1}} \circ \cmap{k}) (\bu)}{({\hat f}_{\hat{\bX}^{k}} \circ \cmap{k}) (\bu)} \pref(\bu).
\end{align}
This way, we can compute a TT $\gk{k+1}(\bu)$ to approximate the function
\begin{equation}\label{eq:ratio_fttk}
q_{k+1}(\bu) \propto \bigg(\frac{(f_{\bX^{k+1}} \circ \cmap{k}) (\bu)}{({\hat f}_{\hat{\bX}^{k}} \circ \cmap{k}) (\bu) } \frac{\pref(\bu)}{\omega(\bu)} \bigg)^{\frac12},
\end{equation}
where $\omega(\bu)$ is the weighting function associated with the reference domain $\mathcal{U}$ and $a \propto b$ denotes that $a$ is proportional to $b$. Since $z_k$, the normalising constant of $f_{\bX^{k+1}}$, is unknown, here we only need to factorise an unnormalised version of $q_{k+1}(\bu)$ into a TT. 
The normalising constant is computed automatically during
the marginalisation process of SIRT (see Proposition \ref{prop:sirt_recur}).

The SIRT $\bU^{k+1} = \amap{k+1}(\bU')$ built from the TT $\gk{k+1}(\bu)$ couples the uniform reference random variable $\bU' \sim \muni$ with $\bU^{k+1} \sim (\amap{k+1})_\sharp \, \muni$.
Thus, the composition of transformations $\bU^{k+1} = (\amap{k+1}\circ R)(\bU)$ couples the general reference random variable $\bU \sim \mref$ with $\bU^{k+1} \sim (\amap{k+1}\circ R)_\sharp \, \mref$, where $(\amap{k+1}\circ R)_\sharp \, \mref$ is an approximation to the pullback measure $\cmap{k}^\sharp\,\nu_{k+1}$. Thus we have
\[
(\amap{k+1}\circ R)^\sharp (\cmap{k}^\sharp \, \nu_{k+1} ) = ( \cmap{k} \circ \amap{k+1} \circ R )^\sharp \, \nu_{k+1} \approx \mref \quad \textrm{or} \quad ( \cmap{k} \circ \amap{k+1} \circ R )_\sharp  \, \mref \approx \nu_{k+1} .
\]
The next DIRT is therefore defined by the new composition of mappings
\[
\cmap{k+1} := \cmap{k} \circ (\amap{k+1}\circ R).
\]
The recursion is completed by obtaining $\amap{L}$ and $\cmap{L}$.


\begin{proposition}\label{prop:inc_densityj}
At the $j$-th DIRT step, the Jacobian of the incremental mapping $\amap{j}$ is given by
\begin{equation}\label{eq:density_iterj}
\big|\nabla_{\bu} \amap{j}^{-1}(\bu) \big| = \hat{p}_{\bU^j} (\bu) = \frac{1}{\hat{z}_j} \big(\gamma_j + \gk{j}(\bu)^2 \big) \omega(\bu) , 
\end{equation}
with
\begin{equation}\label{eq:density_iterj_z}
\hat{z}_j = \gamma_j\, \omega(\mathcal{U})+ \int_\mathcal{U} \gk{j}(\bu)^2 \omega(\bu) d\bu,
\end{equation}
where the constant $\gamma_j > 0$ is chosen according to the $L^2$ error of $\gk{j}$. 
\end{proposition}
\begin{proof}
The SIRT $\bU^{j} = \amap{j}(\bU^\prime)$, which is constructed by integrating $\big(\gamma_j + \gk{j}(\bu)^2 \big)  \omega(\bu)$, maps the uniform random variable $\bU^\prime \sim \muni$ to $\bU^{j} \sim (\amap{j})_\sharp \muni$. Thus, the pushforward measure $(\amap{j})_\sharp \,\muni$ has the density function $\hat{p}_{\bU^{j}}(\bu)$ defined in \eqref{eq:density_iterj}, which yields 
\[
\hat{p}_{\bU^{j}} (\bu) \equiv (\amap{j})_\sharp \puni (\bu) = \big|\nabla_{\bu} \amap{j}^{-1}(\bu) \big|.
\]
\end{proof}

\begin{lemma}\label{lemma:densityk}
At step $k$ of the DIRT construction, suppose we have the DIRTs
\[
\cmap{j} = (\amap{0} \circ R) \circ (\amap{1}  \circ R) \circ \cdots \circ (\amap{j} \circ R), \quad \textrm{for} \quad j \leq k.
\]
Suppose further we have a normalised density function over the domain $\mathcal{X}$ defined in \eqref{eq:density0} for $j = 0$, and normalised density functions over the reference domain $\mathcal{U}$ defined in \eqref{eq:density_iterj} for $j = 1, \ldots, k$. 
Then, the pushforward measure $(\cmap{k})_\sharp\,\mref$ has the PDF
\begin{equation}\label{eq:densityk}
\hat{f}_{\hat{\bX}^k}(\bx) \equiv (\cmap{k})_\sharp \pref (\bx) = \hat{f}_{\hat{\bX}^0} (\bx) \prod_{j=1}^{k}\frac{(\hat{p}_{\bU^j}\circ \cmap{j-1}^{-1})(\bx)}{(\pref \circ \cmap{j-1}^{-1})(\bx)}.
\end{equation}
\end{lemma}

\begin{proof}
The result can be shown using induction. For the case $k=0$, the result follows directly from \eqref{eq:density0}.
Suppose \eqref{eq:densityk} holds for $k>0$. We define the composition of mappings 
\[
\tmap{k} = \cmap{k} \circ R^{-1} = \amap{0}\circ ( R \circ \amap{1}) \circ \cdots \circ ( R \circ \amap{k}).
\]
Since $R_\sharp \,\mref = \muni$, we have the identity $(\cmap{k})_\sharp \,\mref = (\tmap{k})_\sharp \,\muni$, which leads to
\[
\hat{f}_{\hat{\bX}^k}(\bx) = (\cmap{k})_\sharp \pref (\bx)= (\tmap{k})_\sharp \puni (\bx).
\]
At step $k+1$, we have the new composition of mappings $\tmap{k+1} = \tmap{k}\circ R\circ\amap{k+1}$, the pushforward measures, $(\tmap{k+1})_\sharp \, \mref$ and $(\tmap{k})_\sharp \, \mref$, have the density functions
\begin{align*}
\hat{f}_{\hat{\bX}^{k+1}}(\bx) = (\tmap{k+1})_\sharp \puni (\bx) & = (\puni \circ \tmap{k+1}^{-1}) (\bx) \, \big|\nabla_{\bx} \tmap{k+1}^{-1}(\bx) \big| = \big|\nabla_{\bx} \tmap{k+1}^{-1}(\bx) \big|,\\
\hat{f}_{\hat{\bX}^k}(\bx) = (\tmap{k})_\sharp\puni(\bx) & = (\puni \circ \tmap{k}^{-1}) (\bx) \, \big|\nabla_{\bx} \tmap{k}^{-1}(\bx) \big| = \big|\nabla_{\bx} \tmap{k}^{-1}(\bx) \big|,
\end{align*}
respectively. Applying the change of variables $\bu^\prime = \tmap{k}^{-1}(\bx)$ and $\bu = R^{-1}(\bu^\prime)$, the determinant of $\nabla_{\bx} \tmap{k+1}^{-1}(\bx)$ can be expressed as 
\begin{equation}\label{eq:jac_k1}
\big|\nabla_{\bx} \tmap{k+1}^{-1}(\bx) \big| = \big|\nabla_{\bu} \amap{k+1}^{-1}(\bu) \big| \, \big|\nabla_{\bu'} R^{-1}(\bu') \big| \, \big|\nabla_{\bx} \tmap{k}^{-1}(\bx) \big| .
\end{equation}
Note that the above change of variables implies also that $\big|\nabla_{\bu'} R^{-1}(\bu') \big| = \pref( \bu )^{-1}$ and $\bu = (\tmap{k} \circ R)^{-1}(\bx) = \cmap{k}^{-1}(\bx)$.
Together with Proposition \ref{prop:inc_densityj}, we have 
\begin{align*}
\big|\nabla_{\bx} \tmap{k+1}^{-1}(\bx) \big| & = \frac{( \hat{p}_{\bU^{k+1}} \circ \cmap{k}^{-1}) (\bx)}{ ( \pref \circ \cmap{k}^{-1} ) ( \bx )}\, \big|\nabla_{\bx} \tmap{k}^{-1}(\bx) \big|.
\end{align*}
Thus, the result follows. 
\end{proof}

\begin{corollary}\label{coro:changek}
At step $k$ of the DIRT construction, the composition of mappings, $\cmap{k}$, satisfies
\[
\big|\nabla_{\bx} \cmap{k}^{-1}(\bx)\big| = \frac{\hat{f}_{\hat{\bX}^k}(\bx)}{( \pref \circ \cmap{k}^{-1} ) ( \bx )}.
\]
\end{corollary}
\begin{proof}
The result direct follows from $\cmap{k}^{-1} = R^{-1} \circ \tmap{k}^{-1}$ and the proof of Lemma \ref{lemma:densityk}. 
\end{proof}

\begin{remark}
The normalised PDFs of the $k$-th DIRT step can be expressed as
\begin{equation}\label{eq:pdf_dirt}
\hat{f}_{\hat{\bX}^k}(\bx) 
= \frac{1}{ \bar{z}_k } \hat{\pi}_k(\bx) \lambda(\bx),
\end{equation}
where $\bar{z}_k = \prod_{j=0}^{k} \hat{z}_j$ is an approximation to the normalising constant $z_k$, and 
\begin{equation}\label{eq:un_densityk}
\hat{\pi}_k(\bx) = \big(\gamma_0 + \gk{0}(\bx)^2 \big) \prod_{j=1}^{k} \frac{\big(\gamma_j + (\gk{j}\circ \cmap{j-1}^{-1})(\bx)^2\big) \, (\omega \circ \cmap{j-1}^{-1}) (\bx)}{(\pref \circ \cmap{j-1}^{-1}) (\bx)}
\end{equation}
is an approximation to the unnormalised bridging density $\pi_k(\bx)$.%
\end{remark}

\subsection{Ratio functions and error analysis}
We will first discuss the ratio function \eqref{eq:ratio_fttk} and its approximation and then present the corresponding error analysis. 

\subsubsection{Ratio functions} 
Given the unnormalised PDF in \eqref{eq:un_densityk}, the pullback density in \eqref{eq:pullback_ratio} can be expressed as
\begin{align*} 
(\cmap{k})^\sharp f_{\bX^{k+1}}(\bu) = 
\frac{(f_{\bX^{k+1}} \circ \cmap{k}) (\bu)}{({\hat f}_{\hat{\bX}^{k}} \circ \cmap{k}) (\bu)} \, \pref(\bu) 
\propto \frac{(\pi_{k+1}\circ \cmap{k})(\bu) }{(\hat{\pi}_k\circ\cmap{k})(\bu)}\, \pref(\bu).
\end{align*}
This way, we need to compute a TT $\gk{k+1}(\bu)$ to approximate the function
\begin{equation}\label{eq:qk_exact}
q_{k+1}(\bu) =   \bigg(\frac{ (\pi_{k+1}\circ \cmap{k})(\bu) } {(\hat{\pi}_k\circ\cmap{k})(\bu)} \frac{ \pref(\bu)} {\omega(\bu)} \bigg)^\frac12,
\end{equation}
to build the SIRT $\amap{k+1}$. We call this strategy the \emph{exact ratio approach}.

Alternatively, the pullback density in \eqref{eq:pullback_ratio} can be expressed as
\begin{align} \label{eq:qk_exact_ratio}
(\cmap{k})^\sharp f_{\bX^{k+1}}(\bu) \propto \frac{(\pi_{k+1}\circ \cmap{k})(\bu) }{(\pi_{k}\circ \cmap{k})(\bu) } \frac{(\pi_{k}\circ \cmap{k})(\bu) }{(\hat{\pi}_k\circ\cmap{k})(\bu)}\, \pref(\bu)
\end{align}
Since the DIRT density function $\hat{\pi}_k$ approximates the $k$-th unnormalised bridging density function $\pi_k$,
the pullback density in \eqref{eq:qk_exact_ratio} can be approximated as
\begin{align*} 
(\cmap{k})^\sharp f_{\bX^{k+1}}(\bu) \propto (\rk{k+1}{k} \circ \cmap{k}) (\bu) \frac{(\pi_{k}\circ \cmap{k})(\bu)}{(\hat{\pi}_k\circ\cmap{k})(\bu)} \,\pref(\bu)
\approx (\rk{k+1}{k} \circ \cmap{k}) (\bu)\, \pref(\bu).
\end{align*}
This way, we need to compute a TT $\gk{k+1}(\bu)$ that approximates the function
\begin{equation}\label{eq:qk_approx}
\tilde{q}_{k+1}(\bu) = \bigg(\frac{(\rk{k+1}{k} \circ \cmap{k}) (\bu)\, \pref(\bu) }{\omega(\bu)} \bigg)^\frac12 ,
\end{equation}
to build an alternative SIRT $\amap{k+1}$.
We call this strategy the \emph{approximate ratio approach}. 

\begin{remark}\label{remark:dirt_bound_ratio}
For all $k \geq 0$, we want the ratio $\pi_k / \hat{\pi}_k$ to be finite in $\mathcal{U}$. Otherwise, it may cause large errors in the TT decomposition and may deteriorate the convergence of the resulting sampling schemes for characterising $\mtar$. Given Assumption \ref{assum:ratio} and $\gamma_j > 0$ for all $j = 0, \ldots, k$, we have 
\begin{align}
\sup_{\bx \in \mathcal{X}} \frac{\pi_0(\bx)}{\gamma_0 + \gk{0}(\bx)^2} < \infty \quad {\rm and} \quad
\sup_{\bu \in \mathcal{U}}  \frac{(\rk{k}{k-1} \circ \cmap{k-1}) (\bu) \, \pref(\bu)}{\big(\gamma_k + \gk{k}(\bu)^2\big) \, \omega(\bu)} < \infty, \label{eq:dirt_ratio_bound}
\end{align}
and thus it can be shown (using induction) that the ratio $\pi_k / \hat{\pi}_k$ is bounded.
\end{remark} 

\begin{remark}
In some situations, the ratio function $(\rk{k+1}{k} \circ \cmap{k}) (\bu)$ may exhibit sharp boundary layers
if the uniform reference measure $\muni$ (with $R  = I$) is used. This can increase the complexity of the resulting TT decompositions.
Apart from carefully choosing the bridging measures, a partial remedy to the boundary layer is to use a reference measure with the density $\pref(\bu)$ decaying towards the boundary, such as the normal density truncated on a sufficiently large hypercube $[-\sigma, \sigma]^d$. The function $\pref(\bu)$ in \eqref{eq:qk_exact} and \eqref{eq:qk_approx} smoothens the previous approximation errors, which can improve the accuracy of TT approximations.
With a reference measure defined on a hypercube, the collocation techniques based on Chebyshev and Fourier bases (see Section~\ref{sec:sirt_cdf}) can be applied to construct and evaluate functional TT decompositions in DIRT.
\end{remark}


\subsubsection{DIRT error}

\newcommand{\epdfx}[1]{f_{\bX^{#1}}}
\newcommand{\apdfx}[1]{\hat{f}_{\hat{\bX}^{#1}}}

Based on assumptions on the TT error at each layer of the DIRT construction, here we first establish bounds on approximation errors of the DIRT-induced approximate density $\hat\pi_k$ for both the exact ratio approach and the approximate ratio approach. The error analysis also provides heuristics for choosing the constants $\gamma_j$. Then, the error bounds of $\sqrt{\hat\pi_k}$ leads to bounds on the TV distance, the Hellinger distance, and the $\chisq$-divergence of the $k$-th bridging measure $\nu_k$ from the pushforward measure $(\cmap{k})_\sharp \mref$.

\begin{theorem}[Exact ratio approach]\label{thm:dirt_exact}
At the $k$-th DIRT step ($k>0$), suppose the TT decomposition $\gk{k} \approx q_{k}$ and the constant $\gamma_k$ respectively satisfy
\[
\lpnormu{\gk{k} - q_k}{2} \leq \epsilon_k \quad {\rm and} \quad \gamma_k \leq \frac{1}{\omega(\mathcal{U})} \lpnormu{\gk{k} - q_k}{2}^2,
\]
where $q_{k}$ is defined in \eqref{eq:qk_exact}.
Then the unnormalised PDF $\hat{\pi}_k$ \eqref{eq:un_densityk} approximates the unnormalised density function of the $k$-th bridging measure $\pi_k$ with the error
\begin{align*}
\lpnormx{\sqrt{\pi_k}  - \sqrt{\hat\pi_k} }{2} \leq \sqrt{2\,\bar{z}_{k-1}} \epsilon_k,
\end{align*} 
where $\bar{z}_{k-1} = \prod_{j=0}^{k-1} \hat{z}_j$ is the normalising constant of the unnormalised PDF $\hat{\pi}_{k-1}$.
\end{theorem}

\begin{proof}
We start with the identity
\begin{align*}
\lpnormx{\sqrt{\pi_k}  - \sqrt{\hat\pi_k} }{2} & = \bigg(\int_\mathcal{X} \bigg( \Big(\frac{\pi_{k}(\bx)}{\hat\pi_{k-1}(\bx)}\Big)^\frac12 - \Big(\frac{\hat\pi_{k}(\bx)}{\hat\pi_{k-1}(\bx)}\Big)^\frac12 \bigg)^2 \hat\pi_{k-1}(\bx) \,\lambda(\bx)\,d\bx\bigg)^\frac12 \nonumber\\
& = \sqrt{\bar{z}_{k-1} } \bigg( \int_\mathcal{X} \bigg( \Big(\frac{\pi_{k}(\bx)}{\hat\pi_{k-1}(\bx)}\Big)^\frac12 - \Big(\frac{\hat\pi_{k}(\bx)}{\hat\pi_{k-1}(\bx)}\Big)^\frac12 \bigg)^2 \apdfx{k-1}(\bx)\,d\bx \bigg)^\frac12 .
\end{align*}
Applying the change of variables $\bu = \cmap{k-1}^{-1}(\bx)$ and Equations \eqref{eq:un_densityk} and \eqref{eq:qk_exact}, we have
\begin{align*}
\frac{\pi_{k}(\bx)}{\hat\pi_{k-1}(\bx)} & = \frac{(\pi_{k}\circ \cmap{k-1})(\bu)}{(\hat\pi_{k-1}\circ \cmap{k-1})(\bu)} = \frac{ q_k (\bu)^2 \,\omega (\bu)}{\pref(\bu)},\\
\frac{\hat\pi_{k}(\bx)}{\hat\pi_{k-1}(\bx)}  & = \frac{(\hat\pi_{k}\circ \cmap{k-1})(\bu)}{(\hat\pi_{k-1}\circ \cmap{k-1})(\bu)} =\frac{\big(\gamma_k + \gk{k}(\bu)^2\big)\,  \omega (\bu)}{\pref(\bu)}.
\end{align*}
In addition, Corollary \ref{coro:changek} implies 
\(
\apdfx{k-1}(\bx)\,d\bx = \pref(\bu)\,d\bu.
\)
Thus, we have
\begin{align*}
\lpnormx{\sqrt{\pi_k}  - \sqrt{\hat\pi_k} }{2} & = \sqrt{\bar{z}_{k-1}} \Big( \int_\mathcal{U} \Big( q_k (\bu) - \big(\gamma_k + \gk{k}(\bu)^2\big)^\frac12 \Big)^2 \omega(\bu) d\bu \Big)^\frac12 \\
& \leq \sqrt{\bar{z}_{k-1}} \Big( \int_\mathcal{U} \Big( q_k (\bu) - \gk{k}(\bu) \Big)^2 \omega(\bu) d\bu + \gamma_k \,\omega(\mathcal{U}) \Big)^\frac12 \\
& \leq \sqrt{2\,\bar{z}_{k-1}} \epsilon_k,
\end{align*}
where the second last inequality follows from the same proof of Proposition \ref{prop:sirt_l2}.
\end{proof}


\begin{theorem}[Approximate ratio approach]\label{thm:dirt_approx}
Suppose the initial TT decomposition $\gk{0} \approx \sqrt{\pi_0}$ and the constant $\gamma_0$ satisfy
\[
\lpnormu{\gk{0} - \sqrt{\pi_0}}{2} \leq \epsilon_0 \quad {\rm and} \quad \gamma_0 \leq \frac{1}{\lambda(\mathcal{X})} \lpnormu{\gk{0} - \sqrt{\pi_0} }{2}^2,
\]
respectively. At the $k$-th DIRT step ($k>0$), suppose further the TT decompositions $\gk{j} \approx \tilde{q}_{j}$  and the constants $\gamma_j$ satisfy
\begin{align}
\lpnormu{\gk{j} - \tilde{q}_j}{2} \leq \epsilon_j \quad \textrm{and} \quad
\gamma_j \leq \frac{1}{\omega(\mathcal{U})} \lpnormu{\gk{j} - \tilde{q}_j }{2}^2, \;\; \textrm{for} \;\; j = 1, \ldots, k,
\end{align}
respectively, where $\tilde{q}_{j}$ is defined in \eqref{eq:qk_approx}.
Then the unnormalised PDF of DIRT defined by \eqref{eq:un_densityk} approximates the $k$-th unnormalised bridging density function with the error
\begin{align*}
\lpnormx{\sqrt{\pi_k}  - \textstyle\sqrt{\hat\pi_k} }{2} \leq \sqrt{2\,c_{k,0}} \epsilon_0 +  \sum_{j=1}^{k} \sqrt{2\,c_{k,j}\bar{z}_{j-1}} \epsilon_j,
\end{align*} 
where $c_{k,j} = \sup_{\bx\in\mathcal{X}} \rk{k}{j}(x)$ is given in Assumption \ref{assum:ratio}, and $\bar{z}_{k} = \prod_{j=0}^{k} \hat{z}_j$.
\end{theorem}

\begin{proof}
The difference between $\sqrt{\pi_k}$ and $\sqrt{\hat\pi_k}$ can be written as
\begin{align*}
\sqrt{\pi_k} - \textstyle\sqrt{\hat\pi_k} 
= & \textstyle \left(\sqrt{\frac{\pi_k}{\pi_0}\pi_0}  - \sqrt{\frac{\pi_k}{\pi_0} \hat\pi_0} \right) + \sum_{j=1}^{k} \Big( \sqrt{\frac{\pi_k}{\pi_j} \frac{\pi_j}{\pi_{j-1}} \hat\pi_{j-1} } - \sqrt{\frac{\pi_k}{\pi_j} \frac{\hat\pi_{j}}{\hat\pi_{j-1}} \hat\pi_{j-1} } \Big) \\
= & \textstyle \sqrt{\rk{k}{0}}\, (\sqrt{\pi_0} - \textstyle\sqrt{\hat\pi_0} ) +
\sum_{j=1}^{k}\sqrt{\rk{k}{j}} \Big(\sqrt{\rk{j}{j-1}} - \sqrt{\frac{\hat\pi_{j}}{\hat\pi_{j-1}}} \Big)  \sqrt{\hat\pi_{j-1}}  . 
\end{align*}
Then, we have $\lpnormx{\sqrt{\pi_k}  - \sqrt{\hat\pi_k} }{2} \leq I_0 + \sum_{j=1}^{k} I_j$,
where
\begin{align*}
I_0  = \lpnormx{\sqrt{\rk{k}{0}} \, \big(\sqrt{\pi_0} - \textstyle\sqrt{\hat\pi_0} \big) }{2} \;\; {\rm and} \;\;
I_j  = \lpnormx{\sqrt{\rk{k}{j}} \Big(\sqrt{\rk{j}{j-1}} - \textstyle\sqrt{\frac{\hat\pi_{j}}{\hat\pi_{j-1}}} \Big)  \sqrt{\hat\pi_{j-1}} }{2}.
\end{align*}
Recalling that $c_{k,j} = \sup_{\bx\in\mathcal{X}} \rk{k}{j}(x)$ and applying Proposition \ref{prop:sirt_l2}, we have
\[
I_0 \leq \sqrt{c_{k,0}} \,\lpnormx{\sqrt{\pi_0} - \textstyle\sqrt{ \gamma_0 + \gk{0}^2} }{2} \leq \sqrt{2\,c_{k,0}} \epsilon_0 .
\]
Applying the change of variables $\bu = \cmap{j-1}^{-1}(\bx)$ and Corollary \ref{coro:changek} for each $j>0$, we have
\[
\rk{j}{j-1}(\bx) = (\rk{j}{j-1} \circ  \cmap{j-1}\big)(\bu) = \frac{\tilde{q}_j(\bu)^2\, \omega (\bu) }{\pref(\bu) } ,\quad
\frac{\hat\pi_{j}(\bx)}{\hat\pi_{j-1}(\bx)} = \frac{ \big( \gamma_j + \gk{j}(\bu) ^2\big)  \, \omega (\bu) }{\pref(\bu) } 
\]
and $\apdfx{j-1}(\bx)\,d\bx = \pref(\bu)\,d\bu$.
This leads to
\begin{align*}
I_j & \leq \sqrt{c_{k,j}}  \bigg( \!\int_\mathcal{X}\! \bigg(\rk{j}{j-1}(\bx)^\frac12 - \Big(\frac{\hat\pi_{j}(\bx)}{\hat\pi_{j-1}(\bx)} \Big)^\frac12 \bigg)^2 \hat{\pi}_{j-1}(\bx)\,\lambda(\bx)\,d\bx  \bigg)^\frac12\\
& = \sqrt{c_{k,j} \bar{z}_{j-1}}  \bigg( \!\int_\mathcal{X}\! \bigg(\rk{j}{j-1}(\bx)^\frac12 - \Big(\frac{\hat\pi_{j}(\bx)}{\hat\pi_{j-1}(\bx)} \Big)^\frac12 \bigg)^2 \apdfx{j-1}(\bx)\,d\bx \bigg)^\frac12 \\
& = \sqrt{c_{k,j}\bar{z}_{j-1}} \bigg( \int_\mathcal{U} \Big( \tilde{q}_j(\bu) - \big( \gamma_j + \gk{j}(\bu) ^2\big)^\frac12 \Big)^2 \omega(\bu) d\bu \bigg)^\frac12 \nonumber \\
& \leq \sqrt{c_{k,j}\bar{z}_{j-1}} \bigg( \int_\mathcal{U} \Big( \tilde{q}_j(\bu) -  \gk{j}(\bu) \Big)^2 \omega(\bu) d\bu + \gamma_j \,\omega(\mathcal{U}) \bigg)^\frac12 \nonumber \\
& \leq \sqrt{2\,c_{k,j}\bar{z}_{j-1}} \epsilon_j.
\end{align*}
Thus, the result follows. 
\end{proof}


\begin{remark}
At first glance, it appears that Theorem~\ref{thm:dirt_exact} gives smaller errors than Theorem~\ref{thm:dirt_approx}. However, this assumes that the two ratio functions in \eqref{eq:qk_exact} and \eqref{eq:qk_approx} are approximated with the \emph{same} TT error $\epsilon_k$.
Ideally this should also require the same number of degrees of freedom in TT cores.
In practice this may not be the case: the exact ratio~\eqref{eq:qk_exact_ratio} carries the previous approximation errors in the term $(\pi_k \circ \bar T_k)(\bu) \big/ ({\hat\pi}_k \circ \bar{T}_k)(\bu)$, which can have a complicated structure that is difficult to approximate in TT.
In contrast, the approximate ratio involves only the target densities. For example, if the bridging densities $\pi_k$ were introduced by tempering, the ratio $r_{k+1,k} = \pi^{\beta_{k+1}-\beta_{k}}$ is just another tempered density. For this reason, DIRT built by the approximate ratio approach may be more accurate in practice.
\end{remark}

\begin{corollary}
Given $\pi_k$ and $\hat\pi_k$ constructed using either the exact or the approximate ratio functions, we suppose the error of $\sqrt{\hat\pi_k}$ satisfies 
\[
\lpnormx{\sqrt{\pi_k}  - \textstyle\sqrt{\hat\pi_k} }{2} \leq e_k.
\]
Then, the Hellinger distance and the total variation distance between the $k$-th bridging measure $\nu_k$ and the pushforward  measure $(\cmap{k})_\sharp  \mref$ satisfy
\[
D_{\rm H}\big(\nu_k\| (\cmap{k})_\sharp  \mref \big) \leq \frac{\sqrt{2} e_k}{\sqrt{z_k} }\quad {\rm and} \quad D_{\rm TV}\big(\nu_k\| (\cmap{k})_\sharp  \mref \big) \leq  \frac{2 \, e_k}{ \sqrt{z_k} },
\]
respectively.
The $\chisq$-divergence of $\nu_k$ from $(\cmap{k})_\sharp  \mref$ satisfies 
\begin{equation*}
D_{\chisq}\big(\nu_k\| (\cmap{k})_\sharp  \mref \big) \leq \Big( \nu_k \big(\pi_k^2 \big/ \hat\pi_k^2\big)^\frac12 + \big((\cmap{k})_\sharp  \mref \big)\big(\pi_k^2 \big/ \hat\pi_k^2 \big)^\frac12 \Big)  \, \frac{2 \bar{z}_k}{z_k\sqrt{z_k}}\, e_k.
\end{equation*}
\end{corollary}

\begin{proof}
The results can be obtained by applying the same proofs of Theorem \ref{thm:sirt_h}, Corollary \ref{coro:sirt_tv}, and Corollary \ref{coro:sirt_chi}, respectively. Note that for the $\chisq$-divergence, the condition $\sup_{x\in \mathcal{X}} \pi_k(\bx)/\hat\pi_k(\bx) < \infty$ holds as discussed in Remark \ref{remark:dirt_bound_ratio}.
\end{proof}

\section{Debiasing}\label{sec:sampling}
Applying either SIRT or DIRT, one can obtain an approximate map $T: \mathcal{U} \mapsto \mathcal{X}$ that enables the simulation of a random variable $\hat{\bX} \sim T_\sharp\, \mref$ approximating the target random variable $\bX \sim \mtar$.
In a situation where SIRT and DIRT have high accuracy in approximating the target measure, one can approximate the expectation $\mtar(h)$ of a function of interest $h(\bx)$ directly, using the expectation of $h(\bx)$ over $T_\sharp \,\mref$, i.e., $(T_\sharp\mref)(h) \equiv \mref(h\circ T)$. The bias of the approximated expectation is proportional to the Hellinger distance $D_{\rm H}( \mtar \|  T_\sharp \mref)$ (see Proposition~\ref{prop:exp_hell}).
In addition, we can apply the approximate inverse Rosenblatt transport $T$ within the Metropolis-Hastings method and importance sampling to reduce the bias in computing $\mtar(h)$.
For the sake of completeness, here we discuss some debiasing strategies based on existing work.

\begin{algorithm}[htb]
\centering
\caption{IRT-MCMC}
\label{alg:metro_ind}
\begin{algorithmic}[1]
 \State Choose an initial state $\bX^{(0)} = \bx^\ast$ for the Markov chain.
   \For{$j=1,2,\ldots,N$} 
     \State Draw $\bU \sim \mref$ and compute the proposal candidate $\hat{\bX} = T(\bU)$ .
     \State Given the previous state  of the Markov chain $\bX^{(j-1)} = \bx$, with probability
     \begin{equation}\label{eq:mh_acc}
		\alpha(\bx, \hat{\bX}) = \min\left[1, \frac{f_{\bX}(\hat{\bX}) \, \hat{f}_{\hat{\bX}}(\bx) }{f_{\bX}(\bx) \, \hat{f}_{\hat{\bX}}(\hat{\bX})}\right], 
     \end{equation}
     \quad\; accept $\hat{\bX}$ by setting $\bX^{(j)} = \hat{\bX}$, otherwise set $\bX^{(j)} = \bx$. 
     \State Evaluate the function $H_j = h(\bX^{(j)})$.
   \EndFor
   \State Estimate $\mtar(h)$ by the sample average $\frac{1}{N}\sum_{j=1}^{N} H_j$.
\end{algorithmic}
\end{algorithm}

We first consider the IRT-MCMC (Algorithm~\ref{alg:metro_ind}), in which the approximate IRT is used as a proposal mechanism in the Metropolised independent sampler for constructing a Markov chain of random variables that converges to the target measure.
In the acceptance probability \eqref{eq:mh_acc}, $f_{\bX}(\cdot)$ is the PDF of the target measure $\mtar$ and $\hat{f}_{\hat{\bX}}(\cdot)$ is the PDF of $T_\sharp \mref$ that is defined by either the SIRT \eqref{eq:pdf_sirt} or the DIRT \eqref{eq:pdf_dirt}.
Following the result of Mengersen and Tweedie \cite{mengersen1996rates}, the bounds discussed in Remarks \ref{remark:sirt_bound_ratio} and \ref{remark:dirt_bound_ratio} can guarantee the uniform ergodicity of the Markov chain constructed by Algorithm~\ref{alg:metro_ind}.
In addition, the average rejection probability is bounded by $2\,D_{\rm TV}(\mtar \| T_\sharp  \mref )$, see Lemma 1 of \cite{dafs-tt-bayes-2019}. This provides an indicator on the performance of the Metropolised independent sampler. However, our bound on $D_{\rm TV}(\mtar \| T_\sharp  \mref )$ does not directly connect to the bound on the convergence rate of the Metropolised independent sampler, in which the use of acceptance/rejection may require a more precise control on the pointwise error, e.g., $\lpnormx{\tilde{g} - \sqrt{\pi}}{\infty}$, to assess the convergence rate of the sampler.

\begin{algorithm}[htb]
\centering
\caption{IRT-IS}
\label{alg:is}
\begin{algorithmic}[1]
   \For{$j=1,2,\ldots,N$} 
     \State Draw $\bU^{(j)} \sim \mref$ and compute the approximate target random variable $\hat{\bX}^\lowersup{(j)} = T(\bU^{(j)})$ .
     \State Evaluate the unnormalised weight  \vspace{-0.5em}
     \[
     W_j = w(\hat{\bX}^\lowersup{(j)}) := \frac{\ptar(\hat{\bX}^\lowersup{(j)}) \,\lambda(\hat{\bX}^\lowersup{(j)})}{\hat{f}_{\hat{\bX}}(\hat{\bX}^\lowersup{(j)}) } \vspace{-1em}
     \] 
     \quad\, and the function $\hat{H}_j = h(\hat{\bX}^\lowersup{(j)})$.  
   \EndFor
   \State Estimate the normalising constant using $\bar{z}_N = \frac{1}{N}\sum_{j=1}^{N} W_j $.
   \State Compute the sample average $\bar{h}_N = \frac{1}{N}\sum_{j=1}^{N} W_j \hat{H}_j$.\vspace{0.2em}
   \State Estimate $\mtar(h)$ by the ratio estimator $I_N = \bar{h}_N/\bar{z}_N$.
\end{algorithmic}
\end{algorithm}

One can also employ the approximate IRT built by either the SIRT or the DIRT as the biasing distribution in importance sampling, which leads to the IRT-IS algorithm (Algorithm~\ref{alg:is}).
Compared to IRT-MCMC, IRT-IS generates random variables $\hat{\bX}^\lowersup{(j)}$ from the approximate IRT and correct the bias using the weights $W_j$.
By avoiding the Markov chain, importance sampling offers several advantages over the Metropolised independent sampler: ({\romannumeral 1}) it can be easily parallelised; and ({\romannumeral 2}) variance reduction techniques such as antithetic variable and control variates (see \cite[Chapter~4]{robert2013monte} and references therein) and efficient high-dimensional quadrature methods such as quasi Monte Carlo~\cite{Kuo-QMC-2013} can be naturally applied within importance sampling.

The error bounds established in Sections \ref{sec:sirt} and \ref{sec:dirt} offer insights into the efficiency of IRT-IS. As discussed in \cite[Chapter~9]{mcbook}, for $N$ approximate target random variables, one can use the effective sample size (ESS)
\[
{\rm ESS}(N) = N\,\frac{(T_\sharp \mref)(w)^2 }{(T_\sharp \mref)(w^2)},
\]
where $w(\bx) = \ptar(\bx)/\hat\pi(\bx)$ for SIRT and $w(\bx) = \ptar(\bx)/\hat\pi_L(\bx)$ for DIRT, to measure the efficiency of importance sampling for representing the target measure $\mtar$.
\begin{lemma}
Given the $\chisq$-divergence of $\mtar$ from $T_\sharp \,\mref$, the ESS of Algorithm~\ref{alg:is} satisfies
\[
{\rm ESS}(N) = \frac{N}{1 + D_{\chisq}(\mtar\| T_\sharp\mref) }.
\]
\end{lemma}
\begin{proof}
Since we have $(T_\sharp \mref) ( \ptar \lambda \big/ \hat{f}_{\hat{\bX}} ) = z$, where $z$ is the normalising constant of the target density, the $\chisq$-divergence of $\mtar$ from $T_\sharp \,\mref$ satisfies
\[
D_{\chisq}(\mtar\| T_\sharp\mref) = \frac{(T_\sharp \mref)\big(( \ptar \lambda \big/ \hat{f}_{\hat{\bX}})^2\big)}{\big((T_\sharp \mref) ( \ptar \lambda \big/ \hat{f}_{\hat{\bX}}) \big)^2 } -1 = \frac{(T_\sharp \mref) (w^2)}{(T_\sharp \mref)(w)^2} - 1,
\]
in which one can choose $w$ to be the ratio $\ptar \lambda \big/ \hat{f}_{\hat{\bX}}$ multiplied by any nonzero constant. Thus, the result follows.
\end{proof}

The bounds discussed in Remarks \ref{remark:sirt_bound_ratio} and \ref{remark:dirt_bound_ratio} imply that $\bar{z}_N$ and $\bar{h}_N$ computed by Algorithm~\ref{alg:is} are unbiased estimators for the normalising constant $z$ and the expectation $(T_\sharp \mref)(w\,h)$, respectively. However the ratio estimator $I_N$ is only asymptotically unbiased such that
\[
\mathbb{P}\big[ \lim_{N\rightarrow \infty} I_N = \mtar(h)\big] = 1.
\]
For a finite sample size, $N < \infty$, the ratio estimator $I_N$ is a biased estimator of $\mtar(h)$. 
However, for sufficiently large sample size $N$, one can apply the Delta method (see \cite[Chapter~2]{mcbook} and references therein) to show that the mean square error (MSE) of $I_N$ yields the approximation
\begin{equation}\label{eq:mse}
{\rm MSE}(I_N) \equiv \mathbb{E}\big[\big(I_N - \mtar(h)\big)^2\big] \approx \frac{1}{N} \frac{ (T_\sharp \mref) \big( w^2 \big( h - \mtar(h) \big)^2\big)} {(T_\sharp \mref)(w)^2} . 
\end{equation}
Thus, for a sufficiently regular function $h$, the MSE of the ratio estimator $I_N$ can also be controlled by the $\chisq$-divergence $D_{\chisq}(\mtar\| T_\sharp\mref)$.

%% file: sec4_numerics.tex

\newcommand{\ness}{N / {\rm ESS}(N)}

\section{Numerical examples}\label{sec:examples}

We demonstrate the efficiency and various aspects of DIRT, which employs SIRT within each layer, using four Bayesian inference problems arising in dynamical systems and PDEs.
In all numerical examples, as efficiency measures of DIRT (or SIRT in the single layer case), we employ the estimated integrated autocorrelation time (IACT) for IRT-MCMC (Algorithm~\ref{alg:metro_ind}) and the ratio between the total number of samples and the estimated ESS, $\ness$, for IRT-IS (Algorithm~\ref{alg:is}). %
For both IACT and $\ness$, a lower value indicates a better sampling efficiency. The minimum value of both IACT and $\ness$ is $1$.
Since $\ness = 1 + D_{\chisq}(\nu_k\| T_\sharp\mref)$, it also measures directly the accuracy of DIRT for approximating the posteriors.
Matlab implementation of IRT methods and numerical examples is available at \url{https://github.com/dolgov/TT-IRT}.

\subsection{Predator and prey}\label{sec:pp}

The predator-prey model is a system of coupled ODEs frequently used to describe the dynamics of biological systems. The populations of predator (denoted by $Q$) and prey (denoted by $P$) change over time according to a pair of ODEs
\begin{equation}\label{eq:pp}
\left\{\begin{array}{ll}
\displaystyle \frac{d P}{dt} & = \displaystyle r P \Big(1 - \frac{P}{K}\Big) - s \Big(\frac{PQ }{\alpha + P}\Big), \vspace{0.3em} \\
\displaystyle \frac{d Q}{dt} & = \displaystyle u \Big(\frac{PQ }{\alpha + P}\Big) - v Q, 
\end{array}\right.
\end{equation}
with initial conditions $P(t=0) = P_0$ and $Q(t=0) = Q_0$.
The dynamical system is controlled by several parameters.
In the absence of the predator, the population of the prey evolves according to the logistic equation characterised by $r$ and $K$. In the absence of the prey, the population of the predator decreases exponentially with a rate $v$. In addition, the two populations have a nonlinear interaction characterised by $\alpha$, $s$, and $u$.
We often do not know the initial populations and the parameters $r$, $K$, $\alpha$, $s$, $u$, and $v$. This way, we need to estimate unknowns
\[
\bx = \left[ P_0, Q_0, r, K, \alpha, s, u, v \right]^\top,
\]
from observed populations of the predator and prey at time instances $t_i$ for $i = 1, \ldots, n_T$.

\subsubsection{Posterior density}
Let $\by \in \mathbb{R}^{2n_T}$ denote the observed populations of the predator and prey.
We define a forward model $G: \mathcal{X}\mapsto \mathbb{R}^{2n_T}$ in the form of $G(\bx) = [P(t_i), Q(t_i)]_{i=1}^{n_T}$ to represent the populations of the predator and prey computed at $\{t_i\}_{i =1}^{n_T}$ for a given parameters $\bx$.
Assuming independent and identically distributed (i.i.d.) normal noise in the observed data and assigning a prior density $\pi_0(\bx)$ to the unknown parameter, one can define the unnormalized posterior density
\[
\ptar(\bx) \propto \exp\Big(-\frac{1}{2\sigma^2}\|G(\bx) - \by\|^2_2\Big)\,\pi_0(\bx),
\]
where $\sigma$ is the standard deviation of the normally distributed noise. 
Synthetic observed data are used in this example. With $n_T = 13$ time instances $t_{i} = (i-1)\times 25/6$ and a given parameter $\bx_{\rm true} = [50, 5, 0.6, 100, 1.2, 25, 0.5, 0.3]^\top$,
we generate synthetic noisy data $\by = \by_{\rm true} + \eta$, where $\eta$ is a realization of the i.i.d. zero mean normally distributed noise with the standard derivation $\sigma = \sqrt{2}$.
A uniform prior density
$
\pi_0(\bx) = \prod_{k = 1}^{8} \mathbb{I}_{[a_k,b_k]}(\bx_k)
$
is specified to restrict the support of $\bx_k$ to the interval $[a_k, b_k]$, where
$
a = [30, 3, 0.36, 60, 0.72, 15, 0.3, 0.18]^\top
$
and
$
b = [80, 8, 0.96, 160, 1.92, 40, 0.8, 0.48]^\top.
$
To illustrate the behaviour of the posterior density, we plot the kernel density estimates of the marginal posterior densities in Figure~\ref{fig:pp:kde}.
Note that some of the parameters are significantly correlated, which makes the posterior density function difficult to explore by both MCMC and a straightforward TT approximation.

\begin{figure}[htb]
\centering
\includegraphics[width=\linewidth]{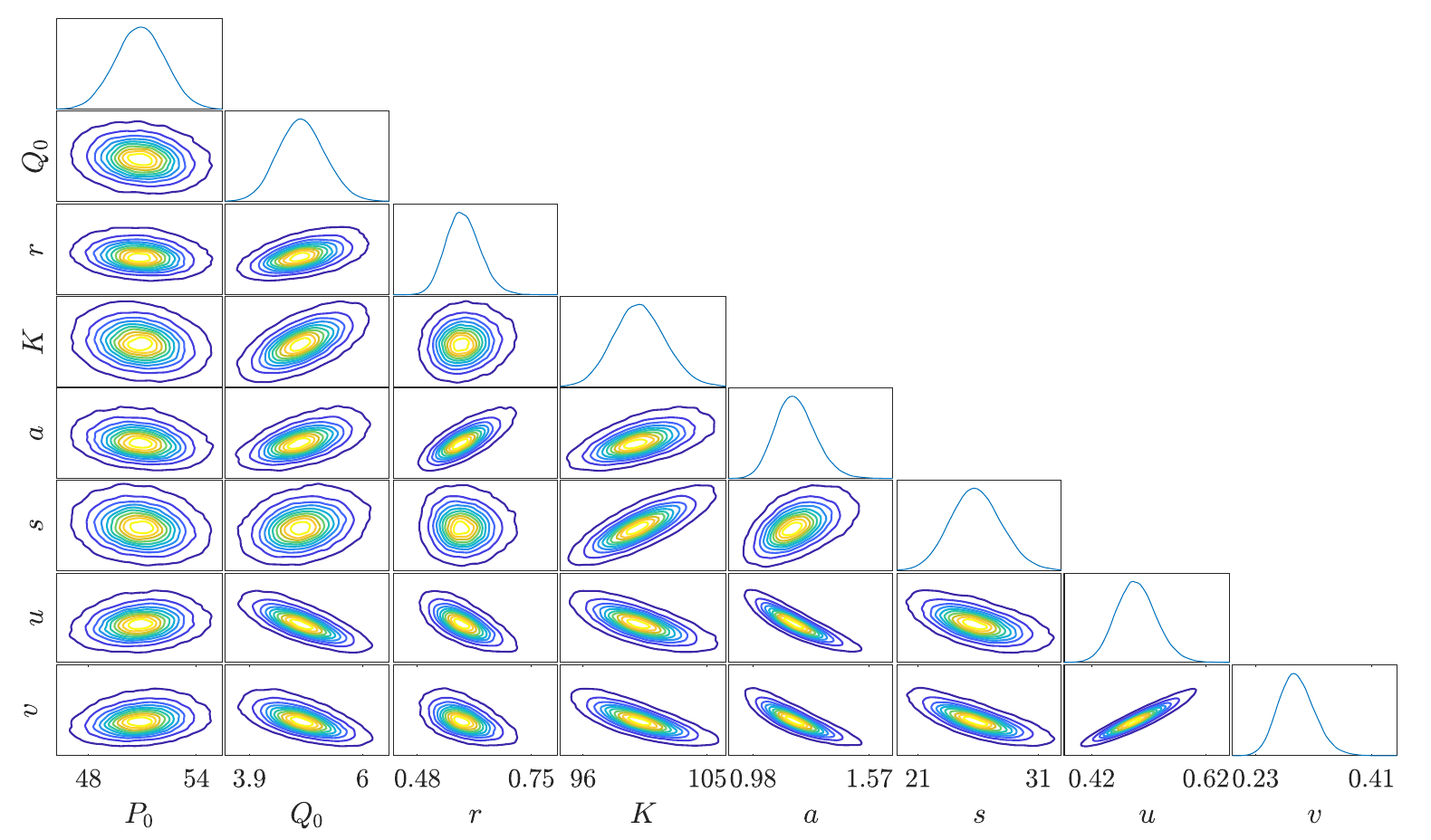}
\caption{Marginal posterior densities of the predator-prey model estimated from $10^6$ posterior samples.} \label{fig:pp:kde}
\end{figure}

\begin{figure}[htb]
\begin{tikzpicture}
\begin{axis}[%
xlabel=\texttt{Rho},
ylabel=IACT,
ymin=1,ymax=11,
xmin=0.9,xmax=4.1,
legend style={at={(0.01,0.005)},anchor=south west,fill=none},
legend cell align={left},
title=(a) uniform reference,
width=0.38\linewidth,
height=0.33\linewidth,
]

\addplot+[error bars/.cd,y dir=both,y explicit] coordinates{
(1, 5.49496)+-(0, 2.79262)
(2, 3.76391)+-(0, 2.8392 )
(3, 4.96062)+-(0, 3.43749)
(4, 3.46202)+-(0, 1.51225)
}; 

\addplot+[error bars/.cd,y dir=both,y explicit] coordinates{
(1, 9.59369)+-(0, 5.41685)
(2, 7.36183)+-(0, 3.82529)
(3, 8.77143)+-(0, 3.79918)
(4, 6.28285)+-(0, 1.92802)
}; 
\end{axis}
\end{tikzpicture}
\begin{tikzpicture}
\begin{axis}[%
xlabel=\texttt{Rho},
ylabel=IACT,
ymin=1,ymax=11,
xmin=0.9,xmax=4.1,
legend style={at={(0.99,1.05)},anchor=north east,fill=none},
legend cell align={left},
title=(b) truncated normal ref.,
width=0.38\linewidth,
height=0.33\linewidth,
]

\addplot+[error bars/.cd,y dir=both,y explicit] coordinates{
(1, 3.08216)+-(0, 0.872035)
(2, 2.69924)+-(0, 0.490011)
(3, 3.17509)+-(0, 1.02337 )
(4, 2.81417)+-(0, 0.340211)
}; \addlegendentry{$\mathtt{MaxIt}=3$};

\addplot+[error bars/.cd,y dir=both,y explicit] coordinates{
(1, 2.63931)+-(0, 0.34377)
(2, 2.9153 )+-(0, 0.56250)
(3, 4.3565 )+-(0, 2.41603)
(4, 3.83012)+-(0, 1.39494)
}; \addlegendentry{$\mathtt{MaxIt}=2$};

\addplot+[mark=diamond*,mark options={mark size=3},error bars/.cd,y dir=both,y explicit] coordinates{
(1, 2.67797)+-(0, 0.527912)
(2, 2.74325)+-(0, 0.459087)
(3, 3.44049)+-(0, 1.64924 )
(4, 2.70956)+-(0, 0.365264)
}; \addlegendentry{$\mathtt{MaxIt}=1$};
\end{axis}
\end{tikzpicture}
\begin{tikzpicture}
\begin{axis}[%
title=(c) number of density eval.,
xlabel=\texttt{Rho},
ylabel=$\times 10^3$,
y label style={at={(0.13,0.95)},rotate=-90},
ymin=15,ymax=48,
xmin=0.9,xmax=4.1,
legend style={at={(1.05,1.05)},anchor=north east,fill=none},
legend cell align={left},
width=0.38\linewidth,
height=0.33\linewidth,
]

\addplot+[error bars/.cd,y dir=both,y explicit] coordinates{
(1, 42784/1e3)
(2, 37600/1e3)
(3, 33952/1e3)
(4, 31840/1e3)
}; 

\addplot+[error bars/.cd,y dir=both,y explicit] coordinates{
(1, 30880/1e3)
(2, 29264/1e3)
(3, 28224/1e3)
(4, 27760/1e3)
}; 

\addplot+[mark=diamond*,mark options={mark size=3},error bars/.cd,y dir=both,y explicit] coordinates{
(1, 16752/1e3)
(2, 17056/1e3)
(3, 17552/1e3)
(4, 18240/1e3)
}; 
\end{axis}
\end{tikzpicture} 
\\

\begin{tikzpicture}
\begin{axis}[%
xlabel=\texttt{Rho},
ylabel=$\ness$,
ymin=1,ymax=11,
xmin=0.9,xmax=4.1,
legend style={at={(0.01,0.005)},anchor=south west,fill=none},
legend cell align={left},
title=(d) uniform reference,
width=0.38\linewidth,
height=0.33\linewidth,
]

\addplot+[error bars/.cd,y dir=both,y explicit] coordinates{
(1, 3.16609)+-(0, 2.05733 )
(2, 1.92633)+-(0, 0.963272)
(3, 2.05082)+-(0, 0.753237)
(4, 2.20485)+-(0, 1.29423 )
}; 

\addplot+[error bars/.cd,y dir=both,y explicit] coordinates{
(1, 11.9378)+-(0, 14.1539)
(2, 5.51376)+-(0, 5.11776)
(3, 4.47826)+-(0, 2.5528 )
(4, 5.65129)+-(0, 3.77855)
}; 
\end{axis}
\end{tikzpicture}
\begin{tikzpicture}
\begin{axis}[%
xlabel=\texttt{Rho},
ylabel=$\ness$,
ymin=1,ymax=11,
xmin=0.9,xmax=4.1,
legend style={at={(0.99,1.05)},anchor=north east,fill=none},
legend cell align={left},
title=(e) truncated normal ref.,
width=0.38\linewidth,
height=0.33\linewidth,
]
\addplot+[error bars/.cd,y dir=both,y explicit] coordinates{
(1, 2.16572)+-(0, 1.77117 )
(2, 1.94595)+-(0, 1.35591 )
(3, 1.72072)+-(0, 0.258144)
(4, 1.59266)+-(0, 0.129924)
}; 

\addplot+[error bars/.cd,y dir=both,y explicit] coordinates{
(1, 1.48506)+-(0, 0.137407)
(2, 1.69311)+-(0, 0.331761)
(3, 2.55416)+-(0, 1.39255 )
(4, 2.59482)+-(0, 2.64939 )
}; 

\addplot+[mark=diamond*,mark options={mark size=3},error bars/.cd,y dir=both,y explicit] coordinates{
(1, 1.47938)+-(0, 0.126075)
(2, 1.64249)+-(0, 0.294334)
(3, 1.9157 )+-(0, 0.686514)
(4, 1.57841)+-(0, 0.107509)
}; 
\end{axis}
\end{tikzpicture}

\caption{(a): IACT$\pm$standard deviation over $10$ runs with the uniform reference measure. (b): IACT$\pm$standard deviation with the truncated normal reference measure. (c): Number of density evaluations in TT-cross at each layer. (d): $\ness\pm$standard deviation over $10$ runs with the uniform reference measure. (e): $\ness\pm$standard deviation with the truncated normal reference measure. Initial TT rank is adjusted such that the maximal TT rank is $13$ in all tests, whereas enrichment ranks \texttt{Rho} and numbers of TT-cross iterations \texttt{MaxIt} are varied.
}
\label{fig:pp:kick}
\end{figure}
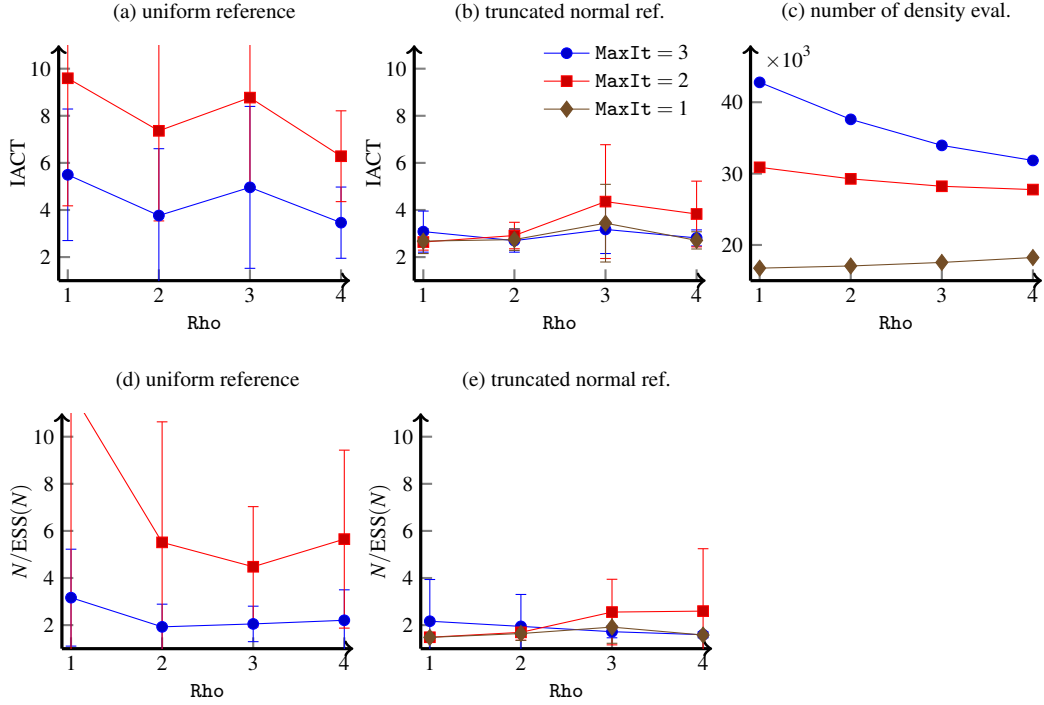

\begin{figure}[htb]
\begin{tikzpicture}
\begin{axis}[%
xlabel=$\mathtt{R}_{\max}$,
xmin=8.7,xmax=16.1,
ymin=1,ymax=6.9,
title=(a) IACT,
width=0.4\linewidth,
height=0.33\linewidth,
]

\addplot+[error bars/.cd,y dir=both,y explicit] coordinates{
(3+9, 6.30951)+-(0, 3.65508 )
(4+9, 3.64766)+-(0, 2.11751 )
(5+9, 2.95461)+-(0, 1.24639 )
(6+9, 2.56037)+-(0, 0.611001)
(7+9, 2.18981)+-(0, 0.721095)
}; 

\addplot+[error bars/.cd,y dir=both,y explicit] coordinates{
(9 , 6.30021)+-(0, 3.84185 )
(10, 5.02011)+-(0, 3.13215 )
(11, 3.44121)+-(0, 1.44929 )
(12, 2.99803)+-(0, 0.403069)
(13, 2.86486)+-(0, 1.10868 )
(14, 2.69505)+-(0, 0.24583 )
(15, 2.39004)+-(0, 0.275344)
(16, 2.38654)+-(0, 0.208617)
}; 

\addplot+[mark=diamond*,mark options={mark size=3},error bars/.cd,y dir=both,y explicit] coordinates{
(9 , 16.38  )+-(0, 12.4156)
(10, 4.64453)+-(0, 1.7727 )
(11, 4.13841)+-(0, 0.9621 )
(12, 3.40321)+-(0, 1.1824 )
(13, 2.86638)+-(0, 0.9173 )
(14, 3.0299 )+-(0, 0.85412)
(15, 2.2677 )+-(0, 0.233  )
(16, 2.24754)+-(0, 0.285  )
}; 
\end{axis}
\end{tikzpicture}
\begin{tikzpicture}
\begin{axis}[%
xlabel=$\mathtt{R}_{\max}$,
xmin=8.7,xmax=16.1,
ymin=1,ymax=6.9,
title=(b) $\ness$,
width=0.4\linewidth,
height=0.33\linewidth,
legend style={at={(0.99,0.99)},anchor=north east,fill=none},
legend cell align={left},
]
\addplot+[error bars/.cd,y dir=both,y explicit] coordinates{
(3+9, 2.83585)+-(0, 1.34011 )
(4+9, 1.67911)+-(0, 0.510679)
(5+9, 2.18784)+-(0, 2.00272 )
(6+9, 1.28077)+-(0, 0.13787 )
(7+9, 1.27141)+-(0, 0.157085)
}; \addlegendentry{uniform};

\addplot+[error bars/.cd,y dir=both,y explicit] coordinates{
(9 , 2.75371)+-(0, 1.75379 )
(10, 2.69817)+-(0, 2.14111 )
(11, 1.63366)+-(0, 0.208055)
(12, 1.79356)+-(0, 0.291916)
(13, 1.65096)+-(0, 0.665414)
(14, 1.57965)+-(0, 0.243106)
(15, 1.47021)+-(0, 0.12442 )
(16, 1.40903)+-(0, 0.046743)
}; \addlegendentry{trunc. normal};

\addplot+[mark=diamond*,mark options={mark size=3},error bars/.cd,y dir=both,y explicit] coordinates{
(9 , 6.41936)+-(0, 4.47445 )
(10, 6.30205)+-(0, 8.51012 )
(11, 3.12404)+-(0, 2.75038 )
(12, 1.8952 )+-(0, 0.414468)
(13, 1.77086)+-(0, 0.572376)
(14, 1.85824)+-(0, 0.644606)
(15, 1.43442)+-(0, 0.092144)
(16, 1.38684)+-(0, 0.097461)
};  \addlegendentry{$\substack{\mbox{trunc. normal}\\\mbox{\hspace{-1.3em}exact ratio}}$};
\end{axis}
\end{tikzpicture}
\begin{tikzpicture}
\begin{axis}[%
xlabel=$\mathtt{R}_{\max}$,
ylabel=$\times 10^3$,
y label style={at={(0.13,0.95)},rotate=-90},
ymin=5,ymax=58,
xmin=8.7,xmax=16.1,
legend style={at={(0.01,0.95)},anchor=north west,fill=none},
legend cell align={left},
title=(c) number of density eval.,
width=0.4\linewidth,
height=0.33\linewidth,
]

\addplot+[error bars/.cd,y dir=both,y explicit] coordinates{
(3+9, 28416/1e3)
(4+9, 33952/1e3)
(5+9, 40064/1e3)
(6+9, 46752/1e3)
(7+9, 54016/1e3)
}; 

\addplot+[error bars/.cd,y dir=both,y explicit] coordinates{
(9 , 8064 /1e3)
(10, 9920 /1e3)
(11, 11968/1e3)
(12, 14208/1e3)
(13, 16640/1e3)
(14, 19264/1e3)
(15, 22080/1e3)
(16, 25088/1e3)
}; 
\end{axis}
\end{tikzpicture}

\caption{IACT$\pm$standard deviation over 10 runs (a), $\ness\pm$standard deviation over 10 runs (b), and number of density evaluations in TT-cross at each layer (c) with varying maximum TT ranks $\mathtt{R}_{\max}$ and different reference measures. For the uniform reference, $\mathtt{MaxIt}=3$ and $\mathtt{Rho}=3$ are used. For the truncated normal reference, $\mathtt{MaxIt}=1$ and $\mathtt{Rho}=0$ are used.}
\label{fig:pp:Rmax}
\end{figure}

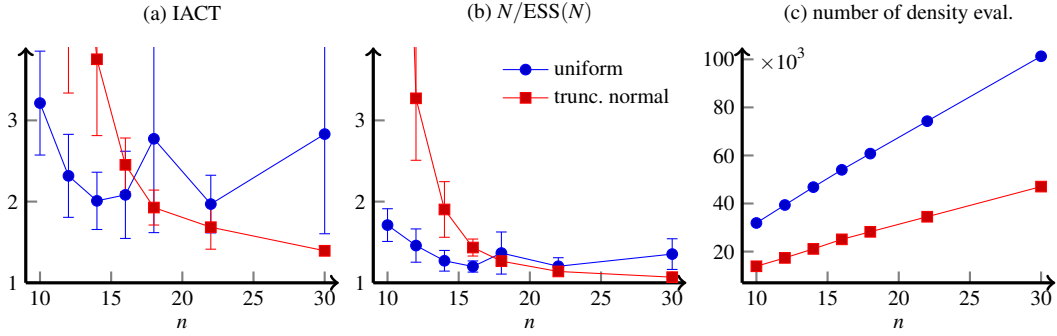
\begin{figure}[htb]
\begin{tikzpicture}
\begin{axis}[%
xlabel=$n$,
xmin=9,xmax=31,
title=(a) IACT,
ymin=1,ymax=3.9,
width=0.4\linewidth,
height=0.33\linewidth,
]

\addplot+[error bars/.cd,y dir=both,y explicit] coordinates{
(32-2, 2.8312 )+-(0, 1.22722 )
(24-2, 1.96943)+-(0, 0.35502 )
(20-2, 2.77324)+-(0, 1.15527 )
(18-2, 2.08316)+-(0, 0.536825)
(16-2, 2.00921)+-(0, 0.352602)
(14-2, 2.31719)+-(0, 0.51103 )
(12-2, 3.21303)+-(0, 0.639731)
}; 

\addplot+[error bars/.cd,y dir=both,y explicit] coordinates{
(32-2, 1.39474)+-(0, 0.061643)
(24-2, 1.68473)+-(0, 0.27081 )
(20-2, 1.92663)+-(0, 0.215249)
(18-2, 2.45273)+-(0, 0.330426)
(16-2, 3.75278)+-(0, 0.940035)
(14-2, 7.39083)+-(0, 4.05416 )
(12-2, 21.5178)+-(0, 8.56583 )
}; 
\end{axis}
\end{tikzpicture}
\begin{tikzpicture}
\begin{axis}[%
xlabel=$n$,
xmin=9,xmax=31,
title=(b) $\ness$,
ymin=1,ymax=3.9,
legend style={at={(0.99,0.99)},anchor=north east,fill=none},
legend cell align={left},
width=0.4\linewidth,
height=0.33\linewidth,
]

\addplot+[error bars/.cd,y dir=both,y explicit] coordinates{
(32-2, 1.35364)+-(0, 0.188768 )
(24-2, 1.20384)+-(0, 0.105071 )
(20-2, 1.36661)+-(0, 0.25922  )
(18-2, 1.20055)+-(0, 0.0686949)
(16-2, 1.27138)+-(0, 0.127336 )
(14-2, 1.45867)+-(0, 0.204943 )
(12-2, 1.71075)+-(0, 0.201387 )
}; \addlegendentry{uniform}; 

\addplot+[error bars/.cd,y dir=both,y explicit] coordinates{
(32-2, 1.06945)+-(0, 0.0120888)
(24-2, 1.13982)+-(0, 0.0334811)
(20-2, 1.26673)+-(0, 0.0571684)
(18-2, 1.43434)+-(0, 0.104595 )
(16-2, 1.9031 )+-(0, 0.342407 )
(14-2, 3.27159)+-(0, 0.762781 )
(12-2, 13.911 )+-(0, 4.96217  )
}; \addlegendentry{trunc. normal}; 
\end{axis}
\end{tikzpicture}
\begin{tikzpicture}
\begin{axis}[%
xlabel=$n$,
ylabel=$\times 10^3$,
y label style={at={(0.13,0.95)},rotate=-90},
ymin=7,ymax=105,
xmin=9,xmax=31,
legend style={at={(0.02,0.98)},anchor=north west,fill=none},
legend cell align={left},
title=(c) number of density eval.,
width=0.4\linewidth,
height=0.33\linewidth,
]

\addplot+[error bars/.cd,y dir=both,y explicit] coordinates{
(32-2, 101280/1e3)
(24-2, 74272 /1e3)
(20-2, 60768 /1e3)
(18-2, 54016 /1e3)
(16-2, 46816 /1e3)
(14-2, 39384 /1e3)
(12-2, 31900 /1e3)
};  

\addplot+[error bars/.cd,y dir=both,y explicit] coordinates{
(32-2, 47040/1e3)
(24-2, 34496/1e3)
(20-2, 28224/1e3)
(18-2, 25088/1e3)
(16-2, 21112/1e3)
(14-2, 17376/1e3)
(12-2, 13880/1e3)
}; 
\end{axis}
\end{tikzpicture}

\caption{IACT$\pm$standard deviation over 10 runs (a), $\ness\pm$standard deviation over 10 runs (b), and number of density evaluations in TT-cross at each layer (c) for varying numbers of collocation points $n$ and different reference measures. For the uniform reference, $\mathtt{MaxIt}=3$ and $\mathtt{Rho}=3$ are used. For the truncated normal reference, $\mathtt{MaxIt}=1$ and $\mathtt{Rho}=0$ are used.}
\label{fig:pp:n}
\end{figure}

\subsubsection{Numerical results} 
We use $L = 8$ bridging measures in the construction of DIRT by tempering the unnormalised posterior density with $\pi_k(\bx) = \pi(\bx)^{\beta_k}$, starting from $\beta_0=10^{-4}$ and following by $\beta_{k+1} = \sqrt{10}\cdot \beta_k$. 
This way, $\beta_L=1$ gives the target probability density.
We consider two reference measures: the uniform reference measure $\muni$ and the truncated normal reference measure $\mref_{\rm TG}$ with the density 
\[
\pref(\bu) \propto \prod_{k=1}^{8} \mathbb{I}_{[-4, 4]} (\bu_k) \exp(-\|\bu_k\|_2^2/2).
\]
Note that at layer $0$, the ratio function is just the tempered density $\pi_0(\bx)$
in the original domain $\bx_k \in [a_k, b_k]$.
We employ the piecewise-linear basis functions with $n$ equally spaced interior collocation points for both reference measures.
In addition, we tune TT-cross (Algorithm~\ref{alg:tt-cross}) using three parameters: the initial TT rank \texttt{R0}, enrichment TT ranks $\rho_1=\cdots=\rho_{d-1}=\mathtt{Rho}$, and the maximum number of TT-cross iterations \texttt{MaxIt}.
Those define uniquely the maximum TT rank $\mathtt{R}_{\max} = \mathtt{R0} + \mathtt{Rho}\cdot \mathtt{MaxIt}$.

Firstly, we vary one tuning variable at a time and investigate its impact on the efficiency and computational cost of the DIRT. 
We take the number of posterior density function evaluations in TT-cross in each DIRT layer to measure the computational cost for building DIRT. 

In Figure~\ref{fig:pp:kick}, we vary the enrichment rank \texttt{Rho} and the number of TT-cross iterations \texttt{MaxIt}. 
The initial TT rank \texttt{R0} is adjusted such that the maximum TT rank is $13$ in all cases. We set the number of collocation points to be $n = 16$. All the DIRTs are constructed using the approximate ratio~\eqref{eq:qk_approx}. 
With each \texttt{Rho} and \texttt{MaxIt}, we repeat the IRT-MCMC and IRT-IS for 10 experiments and report the estimated mean and standard deviation of the efficiency indicators.
For the uniform reference (Figure~\ref{fig:pp:kick} (a) and (d)), carrying out $\texttt{MaxIt}=1$ iteration gives very inaccurate results with $\mbox{IACT}>10$ and $\ness > 10$. 
Increasing the number of TT-cross iterations for the uniform reference measure significantly improves the DIRT accuracy.
Since the ratio function varies considerably from layer to layer, TT-cross needs at least 3 iterations and a nontrivial enrichment to adapt the approximation to the new function. This comes at the expense of tripling the number of density evaluations, in addition to those needed to compute the enrichment, as shown in Figure~\ref{fig:pp:kick} (c).
In contrast, using the truncated normal reference measure (Figure~\ref{fig:pp:kick} (b) and (e)) can significantly improve the efficiency in this example. With only one TT-cross iteration, it can reduce the final IACT to below 4 and $\ness$ to below 3.
\begin{remark}
At levels $k > 0$ of the DIRT construction, the ratio functions may have a similar shape (see~Figure~\ref{fig:dirt-cartoon}). Thus, one can take the TT of the ratio function at the previous level $k>0$ as the initial guess for building TT at level $k+1$. This initialization provides good index sets in TT-cross,
such that only one TT-cross iteration is sufficient with the truncated normal reference measure.
\end{remark}

In Figure~\ref{fig:pp:Rmax}, we vary the maximum TT rank $\mathtt{R}_{\max}$. With the uniform reference, we set $\texttt{Rho}=3$ and $\texttt{MaxIt}=3$. 
With the truncated normal reference, we set $\texttt{Rho}=0$ and $\texttt{MaxIt}=1$, which makes the number of density evaluations equal to the number of degrees of freedom in the TT decomposition, $n_1 r_1 + \sum_{k=2}^{d-1} n_k r_{k-1} r_k + r_{d-1} n_d = (d-2)n\mathtt{R}_{\max}^2 + 2n\mathtt{R}_{\max}$.
We observe that the two reference measures give eventually comparable IACTs and ESSs with increasing $\mathtt{R}_{\max}$. However, the truncated normal reference achieves this with much fewer density evaluations.

In Figure~\ref{fig:pp:Rmax}, we compare also the approximate ratio~\eqref{eq:qk_approx} used in all experiments with the exact ratio~\eqref{eq:qk_exact}. The diamond shaped markers in Figure~\ref{fig:pp:Rmax} (a) and (b) show IACTs and ESSs obtained by the exact ratio approach.
In this example, it gives worse results with a larger IACT and $\ness$ for the truncated normal reference measure with lower $\mathtt{R}_{\max}$ values, and does not lead to any meaningful results for the uniform reference measure.

In Figure~\ref{fig:pp:n}, we vary the number of collocation points $n$ used in each dimension. The truncated normal reference starts with a larger error since $n=10$ points cannot resolve the rather large reference domain $[-4,4]$. With increasing $n$, the IACT obtained using the truncated normal reference decays rapidly.
In comparison, the IACT obtained using the uniform reference exhibits a spike and does not show rapid decay with increasing $n$. This may be caused by the boundary layers in the ratio function.
Similar trends are observed in the reported $\ness$.
Again, the truncated normal reference requires significantly fewer density evaluations to achieve the same level of accuracy compared to the uniform reference in this experiment.

Next, we benchmark DIRT with the truncated normal reference, $\mathtt{MaxIt}=1$, $\mathtt{Rho}=0$, $n=16$, and $\mathtt{R0}=\mathtt{R}_{\max}=13$ against other sampling algorithms, including the Delayed Rejection Adaptive Metropolis (DRAM)~\cite{Haario-DRAM-2006}, the Stein variational Newton (SVN)~\cite{detommaso-SVN-2018}, and the Hierarchical Invertible Neural Transport (HINT)~\cite{Detommaso-HINT-2019}.

DRAM is initialized with the covariance matrix $5\mI$, adaptation scale $2.4/\sqrt{d}$, adaptation interval $10$ and delayed rejection scale $2$. These parameters are commonly recommended in general case.

For this example, SVN is sensitive to the choice of the step size and to the initial distribution of particles. We choose the step size to be $2\cdot 10^{-2}$ and generate the initial particle set from the normal distribution $\mathcal{N}(\bx_{\rm true}, (2\cdot 10^{-2}\bx_{\rm true})^2)$, which gives a reasonable balance between the stability and the rate of convergence. We carry out $23$ Newton iterations in SVN to approach stationarity.

HINT is an autoregressive normalising flow estimator for the joint probability density $\pi(\by,\bx)$.
Since both the prior random variable $\bX$ with the density $\pi_0(\bx)$ and the noise random variable $\eta$ can be directly simulated, drawing samples from the joint distribution is easy.
One can first draw a sample $\bX$ from the prior, and then simulate the corresponding data sample $\bY = G(\bX) + \eta$ by generating a noise random variable $\eta$.
Drawn a set of independent and identically distributed samples $(\bY^{(i)},\bX^{(i)})_{i=1}^{N}$ from the joint measure, HINT computes a triangular invertible map $(\bu_Y, \bu_X ) = S_{\theta}(\by,\bx)$ from the (joint) target measure to the reference measure by minimizing the maximum likelihood
$$
L(\theta) = \sum_{i=1}^{N} \frac{1}{2}\|S_{\theta}(\bY^{(i)},\bX^{(i)}) \|_2^2 - \log |\nabla S_{\theta}(\bY^{(i)},\bX^{(i)})|
$$
over the parameter $\theta$ that defines the neural networks $S_\theta$.
Given observed data $\by$, this allows one to define a conditional map $\bx = T_{\theta}(\bu_X; \by) :=  \big(S_{\theta}^{\bX}\big)^{-1}(\by, \bu_X)$ that maps from the reference measure of $\bU_X$ to the posterior measure conditioned on data $\by$.

We simulate each method $M = 10$ times with $N$ samples produced in each simulation,
denoted by $\{\bx^{(\ell,j)}\}_{j=1}^{N}$, where $\ell=1,\ldots,M$ indexes the simulations.
For each simulation, we compute the empirical posterior covariance matrix $\mC^\ell = \frac{1}{N}\sum_{j=1}^{N} (\bx^{(\ell,j)} - \bar{\bx}^{\ell})(\bx^{(\ell,j)} - \bar{\bx}^{\ell})^\top$, where $\bar{\bx}^{\ell} = \frac{1}{N}\sum_{j=1}^{N} \bx^{(\ell,j)}$ is the empirical posterior mean.
Then, we use the average deviation of covariance matrices to benchmark the sampling performance of different sampling algorithms.
Here we employ the F{\"o}rstner--Moonen distance~\cite{forstner2003metric} over the cone of symmetric and positive definite (SPD) matrices,
\[
d_{\rm FM}(\mA, \mB) = \sum_{i=1}^{d} \ln^2\big(\lambda_i(\mA, \mB)\Big),
\]
where $\lambda_i(\mA, \mB)$ denotes the $i$-th generalised eigenvalue of the pair of SPD matrices $(\mA, \mB)$, to measure the deviation.
This way, averaging the F{\"o}rstner--Moonen distance between the $\ell$-th empirical covariance matrix and the average covariance matrix over all $M$ simulations,
\begin{equation}
 \mathcal{E}_C = \frac{1}{M} \sum_{\ell=1}^{M} d_{\rm FM}(\mC^\ell, \bar{\mC}), \quad \text{where} \quad \bar{\mC} = \frac{1}{M}\sum_{\ell=1}^{M} \mC^\ell,
 \label{eq:pp:err_C}
\end{equation}
provides an estimated deviation of empirical covariance matrices computed by a given algorithm.

In Figure~\ref{fig:pp:bench}, we plot the covariance deviations~\eqref{eq:pp:err_C} obtained by IRT-MCMC, DRAM, SVN and HINT versus the total number of target density function evaluations and the total CPU time needed by each algorithm.
Here the reported total numbers of density evaluations and CPU times include the construction of DIRT in each simulation experiment.
The 10 independent simulations are run in parallel on a workstation with a Intel Xeon E5-2640v4 CPU at 2.4GHz.
We can notice that DIRT produces estimated covariance matrices with smallest deviations in almost all tests.
Moreover, DIRT is computationally more efficient in terms of the CPU time, because the evaluation of DIRT can take advantage of vector instructions.

\begin{figure}[htb]
\centering
\begin{tikzpicture}
\begin{axis}[%
xlabel=$N_{total}$,
title=covariance error \eqref{eq:pp:err_C},
xmode=log,
ymode=log,
xmin=2e3,xmax=2e6,
legend style={at={(0,0)},anchor=south west,fill=none,font=\tiny},
legend cell align={left},
width=0.4\linewidth,
height=0.33\linewidth,
]

\addplot+[line width=1pt,error bars/.cd,y dir=both,y explicit] coordinates{
(16640*9 + 1e3, 0.246274   )+-(0, 0.0830641  )
(16640*9 + 1e4, 0.0186566  )+-(0, 0.00682567 )
(16640*9 + 1e5, 0.00184439 )+-(0, 0.000583848)
(16640*9 + 1e6, 0.000247861)+-(0, 8.00401e-05)
}; \addlegendentry{IRT-MCMC};

\addplot+[line width=1pt,error bars/.cd,y dir=both,y explicit] coordinates{
(18416.3    , 0.223314  )+-(0, 0.0720316 )
(187127     , 0.0197326 )+-(0, 0.00738233)
(1.87527e+06, 0.00368555)+-(0, 0.00236237)
}; \addlegendentry{DRAM};

\addplot+[line width=1pt,mark=diamond*,mark options={mark size=3pt},error bars/.cd,y dir=both,y explicit] coordinates{
(1e2*23, 1.32325 )+-(0, 1.52    )
(1e3*23, 0.159928)+-(0, 0.072684)
(1e4*23, 0.103588)+-(0, 0.128889)
}; \addlegendentry{SVN};

\addplot+[line width=1pt,mark=triangle*,mark options={mark size=3pt},error bars/.cd,y dir=both,y explicit] coordinates{
(1e6,1e0)
}; \addlegendentry{HINT$^*$};
\end{axis}
\end{tikzpicture}
\qquad
\begin{tikzpicture}
\begin{axis}[%
xlabel=CPU time (sec.),
title=covariance error \eqref{eq:pp:err_C},
xmode=log,
ymode=log,
xmin=9e0,xmax=5e3,
legend style={at={(0.99,0.7)},anchor=east,fill=none},
legend cell align={left},
width=0.4\linewidth,
height=0.33\linewidth,
]

\addplot+[line width=1pt,error bars/.cd,y dir=both,y explicit] coordinates{
(13.3531  +   0.13557, 0.246274   )+-(0, 0.0830641  )
(13.9874  +   1.25176, 0.0186566  )+-(0, 0.00682567 )
(13.8229  +   16.7719, 0.00184439 )+-(0, 0.000583848)
(13.5395  +   182.411, 0.000247861)+-(0, 8.00401e-05)
}; 

\addplot+[line width=1pt,error bars/.cd,y dir=both,y explicit] coordinates{
(46.5561, 0.223314  )+-(0, 0.0720316 )
(469.654, 0.0197326 )+-(0, 0.00738233)
(4684.91, 0.00368555)+-(0, 0.00236237)
}; 

\addplot+[line width=1pt,mark=diamond*,mark options={mark size=3pt},error bars/.cd,y dir=both,y explicit] coordinates{
(18.0352, 1.32325 )+-(0, 1.52    )
(188.971, 0.159928)+-(0, 0.072684)
(4278.12, 0.103588)+-(0, 0.128889)
}; 

\addplot+[line width=1pt,mark=triangle*,mark options={mark size=3pt},error bars/.cd,y dir=both,y explicit] coordinates{
(4e3,1e0)
};
\end{axis}
\end{tikzpicture}

\caption{Estimated deviation of empirical covariance matrices~\eqref{eq:pp:err_C} computed by IRT-MCMC, DRAM, SVN and HINT for different total numbers of density evaluations (left) and CPU times (right). {\footnotesize (* the marker in the figure is only indicative, and the actual results are larger)}}
\label{fig:pp:bench}
\end{figure}
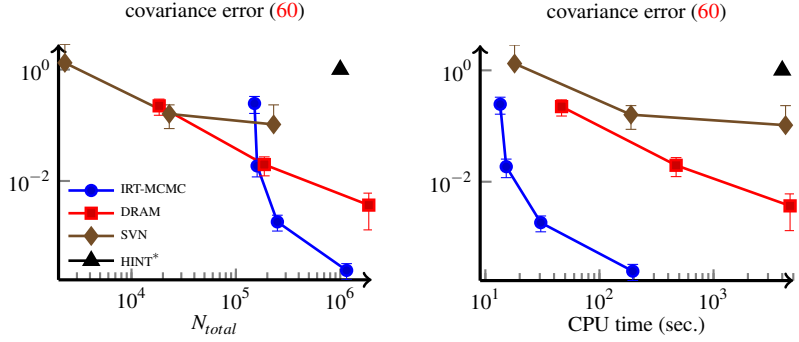

In this example, HINT gives the worst results since the joint density estimation from samples is a much higher dimensional problem compared to the posterior density approximation.
In particular, the dimension of $\by$ in the predator-prey model is $2n_T=26$, so the total dimensionality of the problem increases to $34$.
This required us to construct HINT networks with $8$ blocks containing $200 \times 200$ weights each.
This totalled to 2~691~308 trainable parameters in the entire HINT.
We trained the networks using 5~000~000 training samples for 50 epochs, consisting of 500 batches of 10~000 samples each.
The other (e.g. ADAM) parameters are left unchanged from~\cite{Detommaso-HINT-2019}.
The training took 5.8 hours on a NVidia GeForce GTX 1650 Max-Q GPU card.
To avoid disproportionate scaling of axes in Figure~\ref{fig:pp:bench} compared to other methods, we put just an indicative marker for HINT.
The actual error~\eqref{eq:pp:err_C} with 100~000 test samples taken directly from HINT was 16.7.
Using the test samples as proposals in the MCMC rejection against the exact posterior gives a rejection rate of 97\% and IACT of 127, and the covariance matrix computed from the rejected samples gives the F{\"o}rstner--Moonen distance of 0.1.
This indicates that the data-driven joint density estimation should be more applicable for lower-dimensional data,
whereas if one is only interested in the posterior, the function approximation methods seem to be a better choice.

\subsection{Lorenz-96}
This is a widely used benchmark model in atmospheric circulations. We consider a Lorenz-96 model that is specified by the system of ODEs
\begin{align}\label{eq:lorenz}
 \frac{dP_i(t)}{dt} & = (P_{i+1} - P_{i-2})P_{i-1} - P_i + 8,\quad \text{for} \quad i=1,\ldots,d,
\end{align}
with periodic boundary conditions
and an unknown initial condition $P_i(0) = \bx_i$ for $i=1,\ldots,d$. The state dimension is set to $d=40$.
Observing noisy states with even indices at the final time $T=0.1$, we aim to infer the initial state $\bx$ in this example.
This way, we have observed data $\by \in \mathbb{R}^{\frac{d}2}$ and can define a forward model $G: \mathcal{X} \mapsto \mathbb{R}^{d/2}$ in the form of $G(\bx) = [P_{2k}(T)]_{k=1}^{d/2}$ to represent simulated observables for a given initial condition $\bx$.

\begin{figure}[htb]
\centering
\begin{tikzpicture}
\begin{axis}[%
width=0.8\linewidth,
height=0.4\linewidth,
xlabel={coordinate},
ymin=-1,ymax=3,
xmin=1,xmax=40,
font=\small,
]

\addplot[no marks,blue,dashed,line width=2pt] coordinates{
(1,     0.9969)
(2,     0.8645)
(3,     0.9701)
(4,     0.9528)
(5,     0.9836)
(6,     0.8136)
(7,     0.9602)
(8,     0.8241)
(9,     0.9638)
(10,    0.9787)
(11,    0.9913)
(12,    0.9393)
(13,    0.9814)
(14,    1.1973)
(15,    1.0119)
(16,    1.1731)
(17,    1.0122)
(18,    1.0343)
(19,    0.9979)
(20,    1.1652)
(21,    1.0204)
(22,    0.9172)
(23,    0.9844)
(24,    1.0228)
(25,    0.9909)
(26,    0.9210)
(27,    0.9787)
(28,    1.0590)
(29,    0.9960)
(30,    0.9723)
(31,    0.9838)
(32,    1.1697)
(33,    1.0093)
(34,    0.8526)
(35,    0.9706)
(36,    0.9164)
(37,    0.9769)
(38,    1.0843)
(39,    1.0150)
(40,    1.0903)
};

\addplot[no marks,color=white,name path=minus] coordinates{
(1,    -0.9716)
(2,     0.3782)
(3,    -1.0037)
(4,     0.5649)
(5,    -0.9958)
(6,     0.4269)
(7,    -1.0154)
(8,     0.4416)
(9,    -1.0085)
(10,    0.5928)
(11,   -0.9900)
(12,    0.5530)
(13,   -0.9973)
(14,    0.8088)
(15,   -0.9647)
(16,    0.7811)
(17,   -0.9559)
(18,    0.6414)
(19,   -0.9802)
(20,    0.7760)
(21,   -0.9494)
(22,    0.5250)
(23,   -0.9902)
(24,    0.6365)
(25,   -0.9835)
(26,    0.5326)
(27,   -0.9981)
(28,    0.6713)
(29,   -0.9786)
(30,    0.5832)
(31,   -0.9953)
(32,    0.7823)
(33,   -0.9598)
(34,    0.4639)
(35,   -0.9922)
(36,    0.5319)
(37,   -1.0030)
(38,    0.6992)
(39,   -0.9423)
(40,    0.7019)
};

\addplot[no marks,color=white,name path=plus] coordinates{
(1,     2.9654)
(2,     1.3507)
(3,     2.9438)
(4,     1.3407)
(5,     2.9630)
(6,     1.2003)
(7,     2.9359)
(8,     1.2066)
(9,     2.9362)
(10,    1.3646)
(11,    2.9726)
(12,    1.3256)
(13,    2.9601)
(14,    1.5858)
(15,    2.9884)
(16,    1.5650)
(17,    2.9804)
(18,    1.4272)
(19,    2.9760)
(20,    1.5544)
(21,    2.9902)
(22,    1.3094)
(23,    2.9591)
(24,    1.4090)
(25,    2.9653)
(26,    1.3094)
(27,    2.9555)
(28,    1.4466)
(29,    2.9707)
(30,    1.3613)
(31,    2.9630)
(32,    1.5571)
(33,    2.9784)
(34,    1.2413)
(35,    2.9335)
(36,    1.3009)
(37,    2.9569)
(38,    1.4694)
(39,    2.9724)
(40,    1.4787)
};

\addplot[blue!16!white] fill between[of = minus and plus];

\addplot[no marks,red,solid,line width=1pt] coordinates{
(1,     0.9916)
(2,     1.0039)
(3,     0.9947)
(4,     0.9942)
(5,     0.9990)
(6,     0.9904)
(7,     1.0068)
(8,     1.0057)
(9,     1.0001)
(10,    0.9986)
(11,    1.0051)
(12,    0.9940)
(13,    1.0070)
(14,    0.9924)
(15,    0.9961)
(16,    1.0021)
(17,    1.0069)
(18,    0.9998)
(19,    1.0055)
(20,    0.9945)
(21,    1.0068)
(22,    1.0039)
(23,    0.9824)
(24,    0.9950)
(25,    1.0052)
(26,    0.9904)
(27,    1.0084)
(28,    1.0038)
(29,    0.9885)
(30,    0.9999)
(31,    1.0057)
(32,    1.0135)
(33,    1.0053)
(34,    0.9949)
(35,    1.0005)
(36,    1.0039)
(37,    1.0093)
(38,    1.0048)
(39,    1.0138)
(40,    0.9954)
};
\end{axis}
\end{tikzpicture}
\caption{Lorenz-96 model, true initial state (solid red) and posterior expectation, mean (blue dashed) $\pm$ 2 standard deviations (shaded area).}
\label{fig:lorenz}
\end{figure}
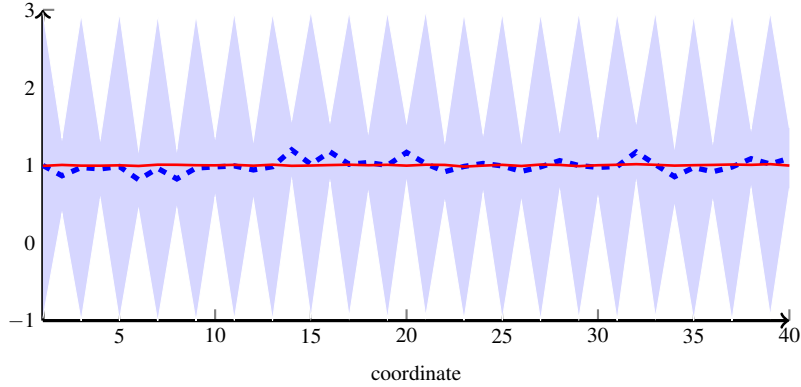

Assuming i.i.d. normal noise in the observed data and assigning a truncated normal prior density to the initial condition, we have the unnormalized posterior density
\[
 \ptar(\bx) = \exp\Big(-\frac{1}{2\sigma^2}\|G(\bx) - \by\|_2^2\Big)\, \prod_{k=1}^{d} \Big( \mathbb{I}_{[-10, 10]}(\bx_k) \exp\big(-\frac12 (\bx_k-1)^2\big)\Big).
\]
We use a synthetic data set $\by = G(\bx_{\rm true}) + \eta$, where $\bx_{\rm true}$ is drawn from $\sim\mathcal{N}(1, 10^{-4}\mI_{d})$,
and $\eta$ is a realisation of the i.i.d. zero mean normal noise with the standard deviation $\sigma = 10^{-1}$.

For the TT-cross approximations, we use the truncated normal reference measure on $[-3,3]^d$, piecewise linear basis functions with $n=15$ interior collocation points, $\texttt{MaxIt}=1$ TT-cross iteration, and TT ranks $\mathtt{R}_{\max}=15$.
DIRT is built with the tempered density
\[
 \ptar_k(\bx) = \exp\Big(-\frac{\beta_k}{2\sigma^2}\|G(\bx) - \by\|_2^2\Big) \cdot \prod_{k=1}^{d} \Big( \mathbb{I}_{[-10, 10]}(\bx_k) \exp\big(-\frac{\beta_k^{0.25}}2 (\bx_k-1)^2\big)\Big),
\]
with $\beta_0=10^{-4}$ and $\beta_{k+1}=\sqrt{10}\cdot \beta_k$. This way, we need $L=8$ layers to reach the posterior density. 
A weaker tempering of the prior is used to reduce its impact on the intermediate levels. This allows most of the intermediate DIRT levels to be used to bridge the more complicated likelihood.
This setup requires a total of $1.2 \times 10^6$ density evaluations in TT-cross at all layers, and provides an average ESS of $N/1.55$ in IRT-IS and an average IACT of $2.6$ in IRT-MCMC.

Using the posterior density, we can quantify the uncertainty of the inferred initial state and make predictions of the terminal state. The predicted initial state is shown in Figure~\ref{fig:lorenz}. Note that the chaotic regime of Lorenz-96 makes it difficult to predict the unobserved odd coordinates. Nevertheless, DIRT demonstrates high numerical and sampling efficiency in approximating this complicated posterior.

\subsection{Elliptic PDE}\label{sec:elliptic}
In the third example, we apply both SIRT and DIRT to the classical inverse problem governed by the stochastic diffusion equation
\begin{equation}
-\nabla \cdot \big(\kappa_d(s; \bx) \nabla u(s)\big) = 0 \quad \mbox{on} \quad s \in D:= (0,1)^2,
\label{eq:pdeproblem}
\end{equation}
with Dirichlet boundary conditions $u|_{s_1=0}=1$ and $u|_{s_1=1}=0$ on the left and right boundaries, and homogeneous Neumann conditions on other boundaries.
The goal is to infer the unknown diffusion coefficient $\kappa_d(s; \bx)$ from incomplete observations of the potential function $u(s)$. 
Here we adopt the same setup used in \cite{dafs-tt-bayes-2019,scheichl-qmc-bayes-2017}. 

\subsubsection{Posterior density}
The unknown diffusion coefficient $\kappa_d(s; \bx)$ 
is parametrized by a $d$-dimensional random variable $\bx$.
We take each of the parameters $\bx_k$, $k=1,\ldots,d,$ to be uniformly distributed on $[-\sqrt{3},\sqrt{3}]$.
Then, for any $\bx \in [-\sqrt{3},\sqrt{3}]^d$ and $s = (s_1,s_2) \in D$,  the logarithm of the diffusion coefficient at $s$ is defined by the following expansion
\begin{equation}
\ln \kappa_d(s; \bx) = \sum_{k=1}^{d} \bx_k \, \sqrt{\eta_k} \, \cos(2\pi \rho_1(k) s_1) \cos(2\pi \rho_2(k) s_2),
\label{eq:kle_art}
\end{equation}
where 
\begin{align*}
\eta_k = \frac{k^{-(\nu+1)}}{\sum_{k=1}^{d} k^{-(\nu+1)}}, \quad
\rho_1(k) = k - \frac{\tau(k)^2 + \tau(k)}{2},  \quad {\rm and} \quad
\rho_2(k) = \tau(k)-\rho_1(k),
\end{align*}
with $\tau(k) = \lfloor \frac12 (\sqrt{1+k/2}  - 1 )\rfloor$.
To discretise the PDE in \eqref{eq:pdeproblem}, we tessellate the spatial domain $D$ with a uniform Cartesian grid with mesh size $h$. Then, we replace the infinite dimensional solution $u \in V \equiv H^1(D)$ by the continuous, piecewise bilinear finite element (FE) approximation $u_h \in V_h$ associated with the discretisation grid. To find $u_h$, we solve the resulting Galerkin system using a sparse direct solver.
A fixed discretisation with $d=11$, $h=2^{-6}$, and $\nu=2$ is used in this example. 

The observed data $\by \in \mathbb{R}^{m}$ consist of $m$ local averages of the potential function $u(s)$ over subdomains $D_{i} \subset D$, $i=1,\ldots,m$.
To simulate the observable model outputs, we define the forward model $G^h:\mathcal{X}\mapsto \mathbb{R}^{m}$ with 
$$
G^h_i(\bx) = \frac{1}{|D_{i}|}\int_{D_i} u_h(s;\bx) ds, \quad i=1,\ldots,m\,.
$$
The subdomains $D_{i}$ are squares with side length $2/(\sqrt{m}+1)$ centred at the interior vertices of a uniform Cartesian grid on $D=[0,1]^2$ with grid size $1/(\sqrt{m}+1)$, which form an overlapping partition of $D$.
Synthetic data for these $m$ local averages are produced from the ``true'' parameter $\bx_{\rm true} = (1.5,\ldots,1.5)$ by adding i.i.d. zero mean normally distributed noise with the standard deviation $\sigma$.
This way, we have the unnormalized posterior density
$$
\ptar(\bx) = \exp\Big(-\frac1{{2\sigma^2}}\big\| G^h(\bx) - \by\big\|^2_2\Big) \, \prod_{k=1}^{d} \Big( \mathbb{I}_{[-\sqrt{3}, \sqrt{3}]}(\bx_k)\Big).
$$

\subsubsection{Numerical results}
In this example, we compare the impact of different tempering schemes, different numbers of measurements, and different measurement noise levels on DIRT. We also compare different basis functions used in the DIRT construction.
In all experiments, we feed $N=2^{16}$ independent samples generated by DIRT to both IRT-MCMC and IRT-IS. 

In Figure~\ref{fig:ell:ml}, we compare DIRT with three different tempering sequences $\beta = \left[\beta_0,\ldots,\beta_L\right]$,
varying the grid size $n$ and the TT ranks $\mathtt{R}_{\max}$.
Note that with $L=0$ we have the single-layer SIRT.
The reported number of density function evaluations is a sum of the numbers of evaluations in TT-cross at all layers.
We use the truncated normal reference measure on $[-4,4]^d$ with both piecewise linear and Fourier bases for the multilayer DIRT.

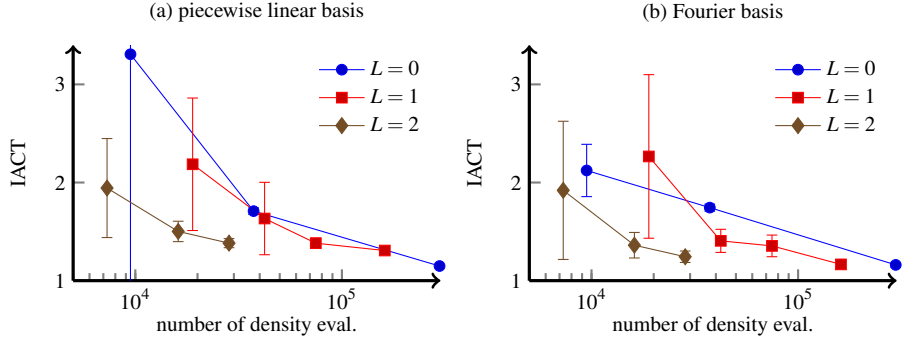
\begin{figure}[htb]
\centering
\begin{tikzpicture}
\begin{axis}[%
 xlabel=number of density eval.,
 ylabel=IACT,
 y label style={at={(-0.15,0.5)}},
 title=(a) piecewise linear basis,
 xmode=log,
 ymin=1,ymax=3.4,
 xmin=5e3,xmax=3e5,
 width=0.45\linewidth,
 height=0.33\linewidth,
 nodes near coords,point meta=explicit symbolic,
 every node near coord/.append style={anchor=south west,font=\small,color=black},
 ]
 \addplot+[error bars/.cd,y dir=both,y explicit] coordinates{
 (9472  , 3.30706)+-(0, 3.42258  )
 (37376 , 1.70847)+-(0, 0.0331986)
 (296960, 1.14902)+-(0, 0.0176813)
 }; \addlegendentry{$L=0$};

 \addplot+[error bars/.cd,y dir=both,y explicit] coordinates{
 (18944, 2.18601)+-(0, 0.675067 )
 (42240, 1.63324)+-(0, 0.369048 )
 (74752, 1.38177)+-(0, 0.051353 )
 (161280,1.30699)+-(0, 0.0284209)
 }; \addlegendentry{$L=1$};

 \addplot+[mark=diamond*,mark options={mark size=3pt},error bars/.cd,y dir=both,y explicit] coordinates{
 (7296,  1.94373)+-(0, 0.50446  )
 (16128, 1.50164)+-(0, 0.103885 )
 (28416, 1.38227)+-(0, 0.0464364)
 }; \addlegendentry{$L=2$};
\end{axis}
\end{tikzpicture}
\begin{tikzpicture}
\begin{axis}[%
 xlabel=number of density eval.,
 ylabel=IACT,
 y label style={at={(-0.15,0.5)}},
 title=(b) Fourier basis,
 xmode=log,
 ymin=1,ymax=3.4,
 xmin=5e3,xmax=3e5,
 width=0.45\linewidth,
 height=0.33\linewidth,
 nodes near coords,point meta=explicit symbolic,
 every node near coord/.append style={anchor=south west,font=\small,color=black},
 ]
 \addplot+[error bars/.cd,y dir=both,y explicit] coordinates{
 (9472  , 2.12292)+-(0, 0.266536 )
 (37376 , 1.74427)+-(0, 0.0356088)
 (296960, 1.1605 )+-(0, 0.0232985)
 }; \addlegendentry{$L=0$};

 \addplot+[error bars/.cd,y dir=both,y explicit] coordinates{
 (18944, 2.26577)+-(0, 0.833482 )
 (42240, 1.40564)+-(0, 0.117714 )
 (74752, 1.35416)+-(0, 0.110273 )
 (161280,1.16564)+-(0, 0.0472221)
 }; \addlegendentry{$L=1$};

 \addplot+[mark=diamond*,mark options={mark size=3pt},error bars/.cd,y dir=both,y explicit] coordinates{
 (7296,  1.92027)+-(0, 0.704082 )
 (16128, 1.36119)+-(0, 0.130157 )
 (28416, 1.24363)+-(0, 0.0582417)
 }; \addlegendentry{$L=2$};
\end{axis}
\end{tikzpicture}
\caption{Elliptic PDE with $\sigma^2=10^{-2}$ and $m=3^2$. IACT vs. number of density evaluations in DIRT for different number of layers: $L=0$ ($\beta=1$),  $L=1$ ($\beta=\{0.1,1\}$) and $L=2$ ($\beta=\{0.1,\sqrt{0.1},1\}$).
Note $\mathtt{R}_{\max}$ varying from 8 to 32 for $L=0$, but only from 4 to 8 for $L=2$.
}
\label{fig:ell:ml}
\end{figure}

With the noise variance $\sigma^2=10^{-2}$ and a rather small data size $m=3^2$, the posterior density is relatively simple to characterise, and hence can be tackled directly using the single-layer SIRT (see the case $L=0$ in Figure~\ref{fig:ell:ml} and \cite{dafs-tt-bayes-2019}). However, the multilayer DIRT uses much smaller number of collocation points and TT ranks for producing an approximate posterior density with the same accuracy.
Here the 3-layer DIRT needs only 10\% of the density evaluations required for the single-layer counterpart.

\begin{figure}[htb]
\centering
\begin{tikzpicture}
 \begin{axis}[%
    xmode=log,
    ymode=normal,
    xlabel=$m$,
    title=(a) numbers of density eval.,
    xmin=7,xmax=1100,
    ylabel=$\times 10^3$,
    y label style={at={(0.13,0.95)},rotate=-90},
    legend style={at={(0.01,0.99)},anchor=north west,fill=none},
    legend cell align={left},
    width=0.4\linewidth,
    height=0.33\linewidth,
    ]
  \addplot+[line width=2pt] coordinates{%
   (3^2 , 63360 /1e3)
   (5^2 , 84480 /1e3)
   (7^2 , 84480 /1e3)
   (11^2, 105600/1e3)
   (15^2, 105600/1e3)
   (23^2, 126720/1e3)
   (31^2, 126720/1e3)
   };
  \addplot+[line width=1pt] coordinates{%
   (3^2 , 63360 /1e3)
   (5^2 , 84480 /1e3)
   (7^2 , 84480 /1e3)
   (11^2, 105600/1e3)
   (15^2, 105600/1e3)
   (23^2, 126720/1e3)
   (31^2, 126720/1e3)
   };
 \end{axis}
\end{tikzpicture}
\hspace{-12pt}
\begin{tikzpicture}
 \begin{axis}[%
    xmode=log,
    ymode=normal,
    xlabel=$m$,
    title=(b) $\ness$,
    xmin=7,xmax=1100,
    width=0.4\linewidth,
    height=0.33\linewidth,
    legend style={at={(0.01,0.99)},anchor=north west,fill=none},
    legend cell align={left},
 ]
 \addplot+[line width=1pt,error bars/.cd,y dir=both,y explicit,] coordinates{
  (3 ^2, 1.06979)+-(0, 0.00132477)
  (5^2 , 1.14526)+-(0, 0.0025336 )
  (7 ^2, 1.16562)+-(0, 0.00716722)
  (11^2, 1.26355)+-(0, 0.00716935)
  (15^2, 1.30243)+-(0, 0.00949426)
  (23^2, 1.39299)+-(0, 0.00797188)
  (31^2, 1.43036)+-(0, 0.022603  )
 }; \addlegendentry{piecewise linear};
 \addplot+[line width=2pt,error bars/.cd,y dir=both,y explicit,] coordinates{
  (3 ^2, 1.00717)+-(0, 0.00193883)
  (5^2 , 1.01064)+-(0, 0.00222291)
  (7 ^2, 1.02153)+-(0, 0.00482665)
  (11^2, 1.02775)+-(0, 0.00772593)
  (15^2, 1.03095)+-(0, 0.00680918)
  (23^2, 1.02027)+-(0, 0.00288953)
  (31^2, 1.08657)+-(0, 0.0190964 )
 }; \addlegendentry{Fourier};

\end{axis}
\end{tikzpicture}
\hspace{-12pt}
\begin{tikzpicture}
 \begin{axis}[%
    xmode=log,
    ymode=normal,
    xlabel=$m$,
    title=(c) IACT,
    xmin=7,xmax=1100,
    width=0.4\linewidth,
    height=0.33\linewidth,
 ]
 \addplot+[line width=1pt,error bars/.cd,y dir=both,y explicit,] coordinates{
  (3 ^2, 1.32474)+-(0, 0.0181528)
  (5^2 , 1.51505)+-(0, 0.0260164)
  (7 ^2, 1.54861)+-(0, 0.0361176)
  (11^2, 1.7701 )+-(0, 0.0382799)
  (15^2, 1.81671)+-(0, 0.0599083)
  (23^2, 2.10944)+-(0, 0.0438052)
  (31^2, 2.18685)+-(0, 0.105047 )
 }; 
 \addplot+[line width=2pt,error bars/.cd,y dir=both,y explicit,] coordinates{
  (3 ^2, 1.09261)+-(0, 0.0148774)
  (5^2 , 1.11218)+-(0, 0.0178975)
  (7 ^2, 1.1838 )+-(0, 0.0447108)
  (11^2, 1.19464)+-(0, 0.0284256)
  (15^2, 1.18868)+-(0, 0.05525  )
  (23^2, 1.18967)+-(0, 0.0435337)
  (31^2, 1.42214)+-(0, 0.0793362)
 }; 

\end{axis}
\end{tikzpicture}
\caption{Elliptic PDE with varying  numbers of measurements $m$. (a): Total Number of density evaluations in all layers; (b): reciprocal sample size; and (c): IACT. Tempering is carried out with $\beta_0 = 4^{-\lceil \log_4 m \rceil}$, $\beta_{k+1} = 4 \cdot \beta_k$. TT-cross parameters: $n=16$, $\mathtt{R}_{\max}=\mathtt{R}_0=12$, and $\mathtt{MaxIt}=1$.}
\label{fig:ell:m0}
\end{figure}
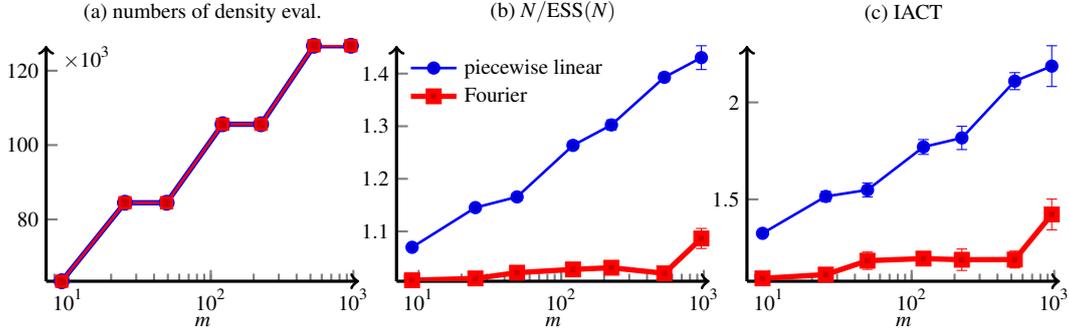

Next, we test the multilayer DIRT on more difficult posterior densities, with larger numbers of measurements and smaller observation noise. We set the number of collocation points to be $n=16$, maximum TT-cross iteration to be $\mathtt{MaxIt}=1$, and maximum TT rank to be $\mathtt{R}_{\max}=12$.
In Figure~\ref{fig:ell:m0}, we fix $\sigma^2 = 10^{-2}$ and vary the number of measurements.
Since halving the measurement grid size $2/(\sqrt{m}+1)$ corresponds to multiplying $m$ by approximately a factor of $4$, we use a different tempering strategy, starting with $\beta_0 = 4^{-\lceil \log_4 m \rceil}$, and setting $\beta_{k+1}=4\cdot\beta_k$ for next layers.
This way, the number of layers grows proportionally to $\log m$, and the number of density evaluations in TT-cross for fixed TT ranks is also proportional to $\log m$, which can be confirmed by Figure~\ref{fig:ell:m0} (a).
Here we can see that the Fourier basis is significantly more accurate than the piecewise-linear basis for the same grid size.
With the linear basis, both IACT and $\ness$ grow logarithmically in the number of measurements.
With the Fourier basis, the IACT stays almost constant below $1.5$ and the $\ness$ stays almost constant below $1.1$, increasing slightly only for the most difficult case with $m = 31^2$ (Figure~\ref{fig:ell:m0} (b) and (c)).
With increasing number of measurements, the likelihood becomes more concentrated. This makes it more challenging to characterise the posterior using prior-based approaches such as QMC~\cite{scheichl-qmc-bayes-2017} or single-layer TT approximation.
For example, even with a much larger number of collocation points $n=65$ and $5$ iterations of TT-cross (giving a maximal TT rank of $41$), we still can not produce reasonable results for $m=15^2$ with the single-layer SIRT.

We carry out an additional test with decreasing noise variance $\sigma^2$.
In Figure~\ref{fig:ell:sigma_n}, we fix $m=15^2$ and vary $\sigma^2$ from $10^{-1}$ to $10^{-5}$.
In this experiment, fixing TT ranks becomes insufficient for representing posterior densities with low observation noise.
In particular, the piecewise linear basis does not have sufficient accuracy for the case of the smallest noise variance.
In contrast, the Fourier basis can still retain low IACT and $\ness$ for low noise variance cases, where IACT and $\ness$ grow proportionally to $\log \sigma$.
Together with the log-scaling of the number of evaluations, the effective complexity of the entire IRT-MCMC and IRT-IS schemes becomes \emph{poly-logarithmic} in the variance.
Although the Fourier basis is computationally more expensive to evaluate than the piecewise-linear basis, with a factor of $2.5$ in the worst case scenario in this experiment, this additional computational effort is well compensated by a much higher accuracy. This makes DIRT a viable approach for a range of concentrated distributions.

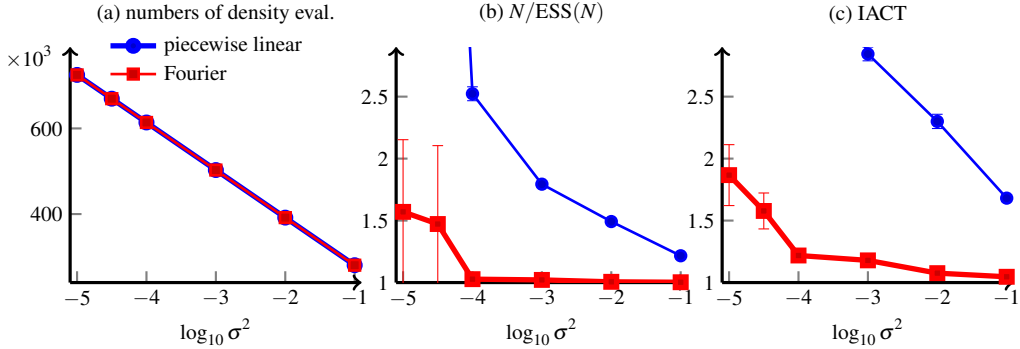
\begin{figure}[htb]
\centering
\begin{tikzpicture}
 \begin{axis}[%
    xmode=normal,
    ymode=normal,
    xlabel=$\log_{10}\sigma^2$,
    title=(a) numbers of density eval.,
    legend style={at={(0.1,1.09)},anchor=north west,fill=none},
    legend cell align={left},
    ymin=240,ymax=790,
    xmin=-5.1,xmax=-0.9,
    ylabel=$\times 10^3$,
    y label style={at={(-0.14,0.96)},rotate=-90},
    width=0.38\linewidth,
    height=0.33\linewidth,
    ]
  \addplot+[line width=2pt] coordinates{%
   (-1.0, 280576/1e3)
   (-2.0, 391680/1e3)
   (-3.0, 502784/1e3)
   (-4.0, 613888/1e3)
   (-4.5, 669440/1e3)
   (-5.0, 724992/1e3)
   }; \addlegendentry{piecewise linear};
  \addplot+[line width=1pt] coordinates{%
   (-1.0, 280576/1e3)
   (-2.0, 391680/1e3)
   (-3.0, 502784/1e3)
   (-4.0, 613888/1e3)
   (-4.5, 669440/1e3)
   (-5.0, 724992/1e3)
   };\addlegendentry{Fourier};
 \end{axis}
\end{tikzpicture}
\hspace{-12pt}
\begin{tikzpicture}
 \begin{axis}[%
    xmode=normal,
    ymode=normal,
    xlabel=$\log_{10}\sigma^2$,
    title=(b) $\ness$,
    ymin=1,ymax=2.9,
    xmin=-5.1,xmax=-0.9,
    width=0.38\linewidth,
    height=0.33\linewidth,
    legend style={at={(0.99,0.99)},anchor=north east,fill=none},
    legend cell align={left},
 ]

 \addplot+[line width=1pt,error bars/.cd,y dir=both,y explicit,] coordinates{
   (-1.0, 1.21617)+-(0, 0.00145605)
   (-2.0, 1.49223)+-(0, 0.00399947)
   (-3.0, 1.79368)+-(0, 0.0118506 )
   (-4.0, 2.52323)+-(0, 0.0561081 )
   (-4.5, 7.67189)+-(0, 2.51237   )
 };

 \addplot+[line width=2pt,error bars/.cd,y dir=both,y explicit,] coordinates{
   (-1.0, 1.00266)+-(0, 0.000282278)
   (-2.0, 1.00605)+-(0, 0.000496591)
   (-3.0, 1.02145)+-(0, 0.00899243 )
   (-4.0, 1.02746)+-(0, 0.00683004 )
   (-4.5, 1.4723 )+-(0, 0.632422   )
   (-5.0, 1.57029)+-(0, 0.582327   )
 }; 

\end{axis}
\end{tikzpicture}
\hspace{-12pt}
\begin{tikzpicture}
 \begin{axis}[%
    xmode=normal,
    ymode=normal,
    xlabel=$\log_{10}\sigma^2$,
    title=(c) IACT,
    ymin=1,ymax=2.9,
    xmin=-5.1,xmax=-0.9,
    width=0.38\linewidth,
    height=0.33\linewidth,
 ]
 \addplot+[line width=1pt,error bars/.cd,y dir=both,y explicit,] coordinates{
   (-1.0, 1.68102)+-(0, 0.0286865 )
   (-2.0, 2.30019)+-(0, 0.0574817 )
   (-3.0, 2.84429)+-(0, 0.0554225 )
   (-4.0, 4.28874)+-(0, 0.313905  )
   (-4.5, 25.2584)+-(0, 20.4251   )
 }; 
 \addplot+[line width=2pt,error bars/.cd,y dir=both,y explicit,] coordinates{
   (-1.0, 1.04664)+-(0, 0.0068179)
   (-2.0, 1.07601)+-(0, 0.0184154)
   (-3.0, 1.17903)+-(0, 0.0268432)
   (-4.0, 1.21841)+-(0, 0.0351191)
   (-4.5, 1.57773)+-(0, 0.144928 )
   (-5.0, 1.86681)+-(0, 0.245747 )
 }; 

\end{axis}
\end{tikzpicture}
\caption{Elliptic PDE with varying noise variances~$\sigma^2$. (a): Total numbers of density evaluations in all layers; (b): reciprocal sample size; (c) IACT. Tempering is carried out with $\beta_0 = 0.1\sigma^2$, $\beta_{k+1} = \sqrt{10} \cdot \beta_k$. TT parameters: $n=16$, TT rank $20$, one TT-cross iteration.}
\label{fig:ell:sigma_n}
\end{figure}

\subsection{Parabolic PDE}\label{sec:heat}

In the fourth example, we consider an inverse problem of identifying the diffusion coefficient of a two-dimensional parabolic PDE from point observations of its solution.
In the problem domain $D = [0, 3]\times [0, 1]$, with boundary $\partial D$, we model the time-varying potential function $p(s,t)$ for given diffusion coefficient field $\kappa_d(s)$ and forcing function $f(s,t)$ using the heat equation
\begin{equation}\label{eq:heat}
\frac{\partial p(s,t)}{\partial t} = \nabla \cdot \left( \kappa_d(s; \bx) \nabla p(s,t) \right) + f(s,t), \quad s \in D, \; t \in [0, T],
\end{equation}
where $T = 10$. Parabolic PDEs of this type are widely used in modeling groundwater flow, optical diffusion tomography, the diffusion of thermal energy, and numerous other common scenarios for inverse problems.
Let $\partial D_{\text n} = \{ s \in \partial D \,|\, s_2 = 0\}  \cup  \{ s \in \partial D \,|\, s_2 = 1\}$ denote the top and bottom boundaries, and $\partial D_{\text d} = \{ s \in \partial \Omega \,|\, s_1 = 0\} \cup \{ s \in \partial \Omega \,|\, s_1 = 3\}$ denote the left and right boundaries. 
For $t \geq 0$, we impose the mixed boundary condition:
\[
p(s,t) = 0, \forall s \in \partial D_{\text d}, \quad \textrm{and} \quad (\kappa_d(s; \theta) \nabla p(s,t) ) \cdot \vec{n}(s) = 0, \forall x \in \partial D_{\text n},
\] 
where $\vec{n}(s)$ is the outward normal vector on the boundary. We also impose a zero initial condition, i.e., $p(s,0) = 0, \forall s \in D$, and let the potential field be driven by a time-invariant forcing function
\[
f(s, t) = c\,\Big( \exp\big(-\frac{1}{2 r^2} \| s - a\|^2 \big) - \exp\big(-\frac{1}{2 r^2} \| s - b\|^2 \big) \Big), \forall t \geq 0,
\]
with $r = 0.05$, which is the superposition of two normal-shaped sink/source terms centered at $a = (0.5, 0.5)$ and $b = (2.5, 0.5)$, scaled by a constant $c = 5\pi\times 10^{-5}$.

\subsubsection{Posterior density}
The logarithm of the diffusion coefficient, $\ln \kappa_d(s; \bx)$, is endowed with the process convolution prior \cite{higdon2002space},
\begin{equation}
\ln \kappa_d(s; \bx) = \ln \bar{\kappa} +  \sum_{k=1}^{d} \bx_k \exp\Big(-\frac12 \|s - s^{(k)}\|^2\Big),
\label{eq:gp_conv}
\end{equation}
where $d = 27$, $\ln \bar{\kappa} = -5$, each coefficient $\bx_k$ follows a standard normal prior $\mathcal{N}(0, 1)$ (which can be truncated to $[-5,5]$ with sufficient accuracy), and $s^{(k)}, k = 1, \ldots, d$ are centers of the kernel functions (shown as blue crosses in Figure \ref{fig:heat_setup} (a)).
Similarly to the previous example, the potential function $p(s,t)$ in \eqref{eq:heat} is approximated by $p_h(s,t)$ using the finite element method with piecewise bilinear basis functions and implicit Euler time integration.

\begin{figure}[htb]
\includegraphics[width=\linewidth]{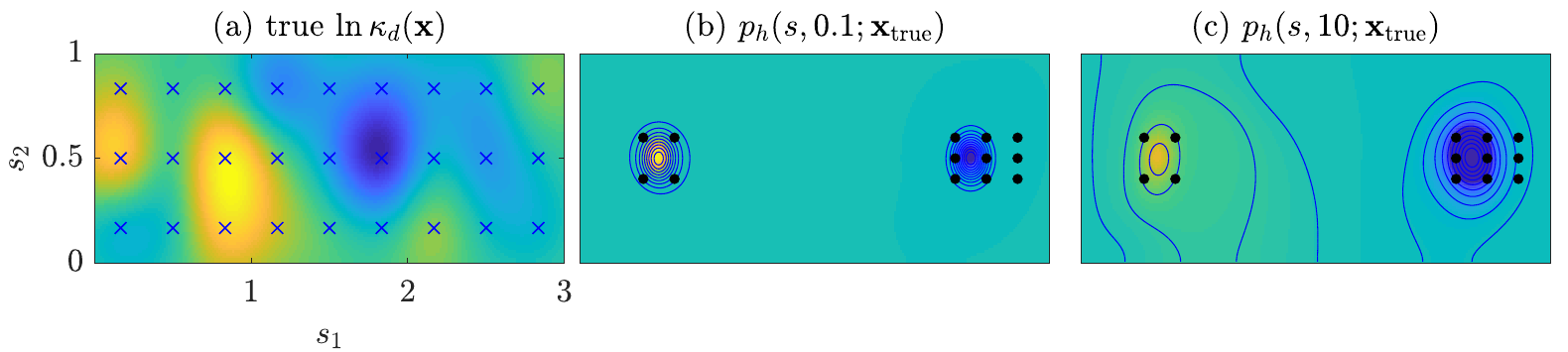}
\caption{Setup of the parabolic example. (a): Logarithm of the ``true'' diffusion coefficient and $d = 27$ centers of the process convolution prior (blue crosses); (b): the potential function $p_h(s,t; \bx_{\rm true})$ at $t = 0.1$, computed with $h = 1/80$; and (c): the potential function $p_h(s,t; \bx_{\rm true})$ at $t = 10$. Black dots in (b) and (c) are locations of measurements.}
\label{fig:heat_setup}
\end{figure}

The observed data $\by \in \mathbb{R}^{m\times n_T}$ consist of the time-varying potential function $p(s,t)$ measured at $m = 13$ locations (shown as black dots in Figure \ref{fig:heat_setup} (b) and (c)) at $n_T = 10$ discrete time points equally spaced between $t=1$ and $t=10$.
To simulate the observable model outputs, we define the forward model $G^h:\mathcal{X}\mapsto \mathbb{R}^{m\times n_T}$ with 
$$
G^h_{i,j}(\bx) = p_h(s_i, t_j;\bx), \quad i=1,\ldots,m, \quad j = 1, \ldots, n_T.
$$
Using a ``true'' parameter $\bx_{\rm true}$ drawn from the prior distribution and a forward model with $h = 1/80$, synthetic data $\by \in \mathbb{R}^{m\times n_T}$ are produced by adding i.i.d. normal noise with zero mean and the standard deviation $\sigma = 1.65\times 10^{-2}$ to $G^h(\bx_{\rm true})$. The corresponding $\ln \kappa_d(s; \bx_{\rm true})$ and the simulated potential function at several time snapshots are shown in Figure~\ref{fig:heat_setup}. The standard deviation $\sigma = 1.65\times 10^{-2}$ corresponds to a signal-to-noise ratio of $10$.
This way, we have the unnormalized posterior density
$$
\ptar(\bx) = \exp\Big(-\frac1{{2\sigma^2}}\big\| G^h(\bx) - \by\big\|_F^2\Big) \, \prod_{k=1}^{d}\Bigl(\mathbb{I}_{[-5,5]}(\bx_k) \exp \big( - \frac12 \bx_k^2\big)  \Bigr).
$$

\subsubsection{Numerical results}
To construct DIRT, we employ a geometric grading in $ \beta$, refining towards $1$,
\[
 \log_{10}\beta_k \in \{-5, -4, -3, -2.5, -2, -1.5, -1, -0.75, -0.5, -0.25, 0\}.
\]
The posterior is very concentrated in this example, so we employ separate tempering of prior and likelihood in the bridging densities,
\[
\ptar_k(\bx) = \exp\Big(-\frac{\beta_k}{{2\sigma^2}}\left\| G^h(\bx) - \by\right\|_F^2\Big) \, \, \prod_{k=1}^{d}\Bigl(\mathbb{I}_{[-5,5]}(\bx_k) \exp \big( - \frac{\beta_k^{0.01}}2  \bx_k^2 \big)  \Bigr).
\]
in which a weakly tempered prior is used.
We use a truncated normal reference measure on the domain $(-4,4]^d$ with the Fourier basis to build DIRT. In TT-cross, a maximum iteration $\mathtt{MaxIt}=1$ without enrichment ($\mathtt{Rho}=0$) is used.
The number of collocation points in each dimension is set to be $n=16$ and the TT ranks are chosen to be $\mathtt{R0}=\mathtt{R}_{\max}=\mathtt{R}_k$, where
\[
 \mathtt{R}_k \in \{15, 15, 15, 15, 15, 15, 13, 9, 9, 8, 7\}
\]
at the $k$-th layer of DIRT.

The PDE in \eqref{eq:heat} is computationally expensive to solve.
Here our goal is to explore the posterior density defined by a forward model, $G^{h_f}$, with refined grid size $h_f = 1/80$. A coarse forward model, $G^{h_c}$ with $h_c = 1/20$, and an intermediate forward model, $G^{h_m}$ with $h_m = 1/40$, are used in defining the bridge densities to speed-up the DIRT construction. 
This multilevel construction shares similarities with the multi-fidelity preconditioning strategy of \cite{peherstorfer2019transport}, except that DIRT is based on TT rather than optimisation and our multilevel models are blended into the bridging densities.
In numerical experiments, we consider the CPU time of solving the coarse model evaluation as one work unit. The CPU times for evaluating the intermediate model and the fine model are about $12.5$ work units and $160$ work units, respectively.

In the first experiment, we employ the coarse forward model, $G^{h_c}$, to compare the sampling performance of DIRT with that of DRAM.
The results are reported in Figure~\ref{fig:samples_heat} (a), where the number of independent samples is calculated as the length of the Markov chain divided by the estimated IACT.
The estimated IACTs for DRAM and DIRT are about $132$ and $3.04$, respectively, and the importance sampling with DIRT produces $\mbox{ESS}=N/1.5$.
For DRAM, we exclude the burn-in samples in the number of work units, whereas the number of work units for the DIRT includes the construction cost of DIRT (993392 density evaluations). 
In this experiment, despite the high construction cost, DIRT can generate a Markov chain with almost independent samples, which is significantly more efficient than DRAM. Furthermore, the construction cost of DIRT will be less significant if one needs to generate more posterior samples, as shown in Figure~\ref{fig:samples_heat}~(a).

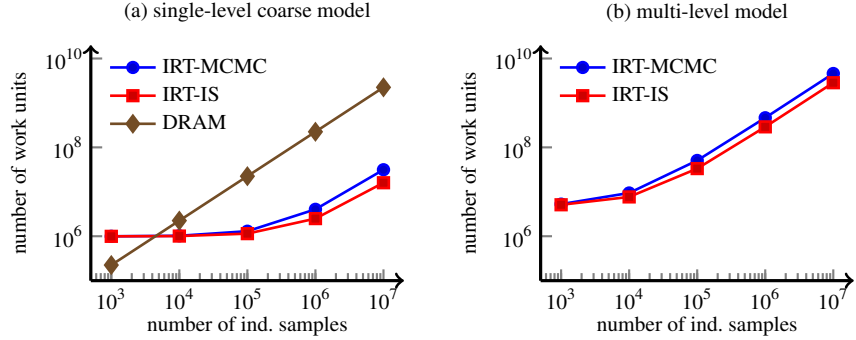
\begin{figure}[htb]
\centering
\begin{tikzpicture}
 \begin{axis}[%
    xmode=log,
    ymode=log,
    ylabel=number of work units,
    xlabel=number of ind. samples,
    title=(a) single-level coarse model,
    legend style={at={(0.03,0.99)},anchor=north west,fill=none},
    legend cell align={left},
    ymin=1e5,ymax=2e10,
    xmin=0.5e3,xmax=2e7,
    y label style={at={(-0.23,0.5)}},
    width=0.4\linewidth,
    height=0.33\linewidth,
    ]
  \addplot+[line width=1pt] coordinates{%
   (1e3, 993392+3.04e3)
   (1e4, 993392+3.04e4)
   (1e5, 993392+3.04e5)
   (1e6, 993392+3.04e6)
   (1e7, 993392+3.04e7)
   };  \addlegendentry{IRT-MCMC};
  \addplot+[line width=1pt] coordinates{%
   (1e3, 993392+1.5e3)
   (1e4, 993392+1.5e4)
   (1e5, 993392+1.5e5)
   (1e6, 993392+1.5e6)
   (1e7, 993392+1.5e7)
   };  \addlegendentry{IRT-IS};
  \addplot+[mark=diamond*,mark options={mark size=3pt},line width=1pt] coordinates{%
   (1e3, 224.4e3)
   (1e4, 224.4e4)
   (1e5, 224.4e5)
   (1e6, 224.4e6)
   (1e7, 224.4e7)
   };  \addlegendentry{DRAM};
 \end{axis}
\end{tikzpicture}
\qquad
\begin{tikzpicture}
 \begin{axis}[%
    xmode=log,
    ymode=log,
    ylabel=number of work units,
    xlabel=number of ind. samples,
    title=(b) multi-level model,
    legend style={at={(0.03,0.99)},anchor=north west,fill=none},
    legend cell align={left},
    ymin=1e5,ymax=2e10,
    xmin=0.5e3,xmax=2e7,
    y label style={at={(-0.23,0.5)}},
    width=0.4\linewidth,
    height=0.33\linewidth,
    ]
  \addplot+[line width=1pt] coordinates{%
   (1e3, 915024+58544*12.5+19824*160+160*2.87e3)
   (1e4, 915024+58544*12.5+19824*160+160*2.87e4)
   (1e5, 915024+58544*12.5+19824*160+160*2.87e5)
   (1e6, 915024+58544*12.5+19824*160+160*2.87e6)
   (1e7, 915024+58544*12.5+19824*160+160*2.87e7)
   };  \addlegendentry{IRT-MCMC};
  \addplot+[line width=1pt] coordinates{%
   (1e3, 915024+58544*12.5+19824*160+160*1.78e3)
   (1e4, 915024+58544*12.5+19824*160+160*1.78e4)
   (1e5, 915024+58544*12.5+19824*160+160*1.78e5)
   (1e6, 915024+58544*12.5+19824*160+160*1.78e6)
   (1e7, 915024+58544*12.5+19824*160+160*1.78e7)
   };  \addlegendentry{IRT-IS};
 \end{axis}
\end{tikzpicture}
\caption{Number of independent samples computed by IRT-MCMC, IRT-IS, and DRAM versus the total computational cost. (a): Comparison using a single level coarse model. (b) Comparison using the multilevel model in the DIRT construction. In both plots, the construction costs of IRT-MCMC and IRT-IS are included, whereas the burn-in cost of DRAM is not included. The work unit is the computational cost of one coarse model evaluation. }
\label{fig:samples_heat}
\end{figure}

In the second experiment, we demonstrate the construction of DIRT using not only the bridge densities with different temperatures, but also the forward models with different grid resolutions.
For initial temperatures such that $\beta_k<10^{-0.5}$, we use the coarse forward model $G^{h_c}$. For $\beta_k=10^{-0.5}$ and $\beta_k=10^{-0.25}$, we use the intermediate forward model $G^{h_m}$.
For $\beta_k=1$ we use the fine forward model $G^{h_f}$, so that the fine model is used to define the target posterior density.
We need $915024$, $58544$, and, $19824$ evaluations of the coarse, intermediate, and fine models, respectively, to construct DIRT. 
Once the DIRT is constructed, Algorithm~\ref{alg:metro_ind} generates a Markov chain with IACT $2.87$ that samples the posterior defined by the fine model. Again, the importance sampling is more efficient with $\mbox{ESS}=N/1.78$. The number of independent samples versus the number of work units is reported in Figure~\ref{fig:samples_heat}~(b).
In this experiment, it is computationally infeasible to apply DRAM directly (or any MCMC in general) to sample the posterior defined by the fine model.
In contrast, the evaluation of DIRT and the corresponding posterior densities
can be embarrassingly parallelised, which can further accelerate the posterior inference using high-performance computers.
The IRT-IS algorithm can bypass the construction of Markov chains, which makes it suitable to be integrated into multilevel Monte Carlo or multilevel quasi Monte Carlo estimators to improve the convergence rate of the computation of posterior expectations.
We leave this as a future research question.

\section{Conclusion}

We have enabled functional tensor decompositions of complicated and concentrated continuous probability density functions that suffer from impractically large tensor ranks when approximated directly. Instead, we build an adaptive sequential change of coordinates that drives the target density towards a product function. This change of variables is realised by the composition of order-preserving SIRTs computed from functional TT decompositions of ratios of bridging densities.
Each of the ratio functions recovers one scale of correlations of the target density, and hence it can be approximated with fixed TT ranks. Together with the triangular structure of the Rosenblatt transport, this makes the total complexity linear in the number of variables.

This deep composition of the inverse Rosenblatt transports shares similarities with deep neural networks with nonlinear activation functions. However, DIRT has several advantages.
\begin{itemize}
\item Each DIRT layer, defined by the bridging densities, can be associated with the scale of noise or observation function. Any prior knowledge of model hierarchies can improve the selection of bridging densities. In contrast, the influence of a particular fully-connected layer in a neural network is difficult to predict or understand.
\item DIRT layers can be computed independently. As soon as the layer is approximated up to the desired accuracy,
it can be saved and never recomputed again. This enables a simple interactive construction, where the tuning parameters can be set layer per layer. Neural networks require optimisation of all layers simultaneously.
\item The construction of each DIRT layer is powered by efficient TT-cross algorithms, which can converge much faster than the stochastic gradient descent used by neural networks in many cases. The dense linear algebra operations used by TT decompositions can take full advantage of modern CPU and GPU vectorisations, whereas an embarrassing parallelism with respect to target density evaluations is well scalable to modern high performance computers.
\end{itemize}

This work opens many potential applications and further enhancements of DIRT. For example, the transport maps defined by DIRT can be naturally extended to approximate the optimal biasing density in importance sampling, which can be valuable for solving rare event simulations. 
In Section \ref{sec:heat}, we offered some preliminary investigation on constructing DIRT using multilevel models. The multilevel idea can be further integrated with DIRT to improve the convergence rate of the importance sampling estimator. 
For problems involving extremely high-dimensional or infinite-dimensional random variables, DIRT can be combined with the likelihood-informed subspace (LIS) \cite{cui2014likelihood,cui2021unified,spantini2015optimal,zahm2018certified} to characterise the highly non-Gaussian effective random variable dimensions identified by LIS. 
In addition, for sequential Bayesian inference, we can apply DIRT to iteratively characterise the filtered posterior measures changing over time, where the evolution of the random states and time-dependent observations naturally define a sequence of bridging measures.
%

%% file: appendix.tex

%
\section{Appendices}

\subsection{Appendix A: construction of functional TT by cross interpolation}\label{sec:tt_cross}
Here we recall an alternating iteration algorithm for constructing the TT decomposition of a multivariate function $h : \mathcal{X} \mapsto \R$.
We seek an TT of the form
\begin{equation}
h(\bx) \approx \tilde{h}(\bx) = 
\sum_{\alpha_0=1}^{r_0}\sum_{\alpha_1 = 1}^{r_1} \cdots \sum_{\alpha_d = 1}^{r_d} \ttcorez{\bx}{\mH} \cdots \ttcorej{k}{\bx}{\mH} \cdots \ttcorej{d}{\bx}{\mH} \label{eq:FTT_r},
\end{equation}
with $r_0 = r_d = 1$. 
Each univariate function $\ttcorej{k}{\bx}{\mH}: \mathcal{X}_k \mapsto \R$ is represented as a linear combination of a set of
$n_k$ basis functions $\{\phi_{k}^{(i)}(\bx_k)\}_{i = 1}^{n_k}$, that is,
\begin{equation}
\ttcorej{k}{\bx}{\mH} = \sum_{i =1}^{n_k}\phi_{k}^{(i)}(\bx_k) \,\tA_k [\alpha_{k-1}, i, \alpha_k]\label{eq:FTT_basis},
\end{equation}
where $\tA_k \in \R^{r_{k-1}\times n_k \times r_k}$ is (the coefficient tensor of) the $k$-th TT core. 
The number of degrees of freedom in the TT decomposition, that is, in the tensors $\{\tA_k\}_{k=1}^{d}$, is linear in $d$ provided the TT ranks $r_0,\ldots,r_d$ are bounded.
For the numerical efficiency it is essential that the TT cores can be constructed using a similar number of evaluations of $h(\bx)$. This can be achieved using cross interpolation methods.
The following definition is used to construct cross interpolations. 

\begin{definition}\label{def:interpolation-basis}
For each variable $\bx_k$, we consider a set of \emph{interpolation basis functions} that can be represented by a vector-valued function 
\[
\bphi_k(\bx_k^{}) = \big[\phi_k^{(1)}(\bx_k^{}), \cdots, \phi\vphantom{X}_k^{(n_k)}(\bx_k^{})\big]\in\R^{1\times n_k}, 
\]
and a set of collocation points $\mathtt{X}_k^{} = \{\bx_k^{(i)}\}_{i=1}^{n_k}$ such that
the $n_k \times n_k$ dimensional Vandermonde matrix
\begin{equation}\label{eq:phi2d}
\bphi_k(\mathtt{X}_k^{})[i,j] \equiv \phi_k^{(j)}(\bx_k^{(i)}),
\end{equation}
is an identity matrix. 
A typical construction is the (piecewise) Lagrange basis functions defined by a point set $\mathtt{X}_k$. 
We can also construct the interpolation basis from other basis functions of a separable Hilbert space, denoted by 
\[
\bpsi_k(\bx_k^{}) = \big[\psi_k^{(1)}(\bx_k^{}), \cdots, \psi_k^{(n_k)}(\bx_k^{})\big]\in\R^{1\times n_k},
\] 
and a point set $\mathtt{X}_k$ with a nonsingular Vandermonde matrix by setting
\[
\bphi_k(\bx_k^{}) = \bpsi_k(\bx_k^{}) \bpsi_k(\mathtt{X}_k^{})^{-1}.
\]
Specifically, if $\bpsi_k(\bx_k^{})$ is a set of $\lambda_k$-orthogonal functions and $\mathtt{X}_k$ are the roots of the function $\psi_{k}^{(n_k+1)}(\bx_k^{})$, we recover the pseudo-spectral methods and have
\[
 \bpsi_k(\mathtt{X}_k^{})^{-1} = \bpsi_k(\mathtt{X}_k^{})^\top \rm{diag}({\bs \omega}_k),
\]
where the vector ${\bs \omega}_k \in \mathbb{R}^{n_k}$ contains quadrature weights associated with $\mathtt{X}_k$ and $\rm{diag} (\cdot)$ brings a vector into a diagonal matrix.

Furthermore, we define the mass matrix $\mM_k = \int \bphi\vphantom{\bx}_k(\bx_k)^\top \bphi\vphantom{\bx}_k(\bx_k) \lambda_k(\bx_k) d\bx_k $. 
We let $\chol_k \in \mathbb{R}^{n_k \times n_k}$ be the Cholesky factor of the mass matrix, i.e.,  $\chol_{k}^{} \chol_{k}^\top = \mM_k^{}$.
For an interpolation basis constructed from $\lambda_k$-orthogonal functions and the roots of $\psi_{k}^{(n_k+1)}(\bx_k^{})$, we have $\mM_k = \rm{diag} ({\bs \omega}_k)$.
\end{definition}

\subsubsection{Two dimensional case}

Consider first the TT decomposition of a bivariate function 
\begin{equation}\label{eq:h2d-factor}
h(\bx_1,\bx_2)\approx \tilde h(\bx_1,\bx_2) = \sum_{\alpha=1}^{r} \mH_1^{(\alpha)}(\bx_1) \mH_2^{(\alpha)}(\bx_2),
\end{equation}
where the rank-$r$ cores
\[
\mH_1(\bx_1) \equiv \bphi_1(\bx_1)\mA_1 \in \mathbb{R}^{1\times r} \quad {\rm and} \quad \mH_2(\bx_2) \equiv \mA_2 \bphi_2(\bx_2)^\top  \in \mathbb{R}^{r \times 1} 
\] 
are specified by basis functions $\bphi_1(\bx_1)\in \mathbb{R}^{1 \times n_1} $ and $\bphi_2(\bx_2)\in \mathbb{R}^{1 \times n_2} $, and the corresponding coefficient matrices $\mA_1 \in \mathbb{R}^{n_1 \times r}$ and $\mA_2\in \mathbb{R}^{r \times n_2}$, respectively.
We aim to recover $\mA_1$ and $\mA_2$ such that the $L^2$ norm of the error 
\[
\lpnormx{h(\bx_1, \bx_2) - \mH_1(\bx_1) \mH_2(\bx_2)}{2}, 
\]
is minimised.
Note that with interpolation bases, the matrices $\mA_1$ and $\mA_2$ are also pointwise evaluations of the functions $\mH_1(\bx_1)$ and $\mH_2(\bx_2)$ at collocation points $\mathtt{X}_1$ and $\mathtt{X}_2$, respectively. 
This way, $h(\bx_1, \bx_2)$ yields a discrete approximation 
\begin{equation}
h(\bx_1, \bx_2) \approx \bphi_1(\bx_1)  h(\mathtt{X}_1,\mathtt{X}_2) \bphi_2(\bx_2)^\top,
\end{equation}
where $h(\mathtt{X}_1,\mathtt{X}_2)\equiv [h(\bx_1^{(i)}, \bx_2^{(j)})] \in \R^{n_1 \times n_2}$ for $\bx_1^{(i)}\in\mathtt{X}_1$ and $\bx_2^{(j)}\in\mathtt{X}_2$ is the matrix of nodal values of $h(\bx_1, \bx_2)$ similarly to~\eqref{eq:phi2d}.
This way, the $L^2$ error of the continuous factorisation yields a discrete approximation
\begin{align}\label{eq:h2d-discrete-norm}
\lpnormx{h(\bx_1, \bx_2) - \mH_1(\bx_1) \mH_2(\bx_2)}{2} & \approx \lpnormx{ \bphi_1(\bx_1) \big( h(\mathtt{X}_1,\mathtt{X}_2) ^\top - \mA_1 \mA_2 \big) \bphi_2(\bx_2)^\top }{2} \nonumber \\
& = \big\| \chol_1^\top \big( h(\mathtt{X}_1,\mathtt{X}_2) - \mA_1 \mA_2 \big) \chol_2 \big\|_F.
\end{align}
Thus, we can recover the matrices $\mA_1$ and $\mA_2$ by solving some low-rank matrix decomposition of $h(\mathtt{X}_1,\mathtt{X}_2)$.
However, assembling the matrix $h(\mathtt{X}_1,\mathtt{X}_2)$ requires evaluating the function $h(\bx_1,\bx_2)$ at the Cartesian union of the collocation points $\mathtt{X}_1 \times \mathtt{X}_2$, which can be computationally prohibitive in the generalisation to $d > 2$.

Instead, we can use some \emph{interpolation point sets} $\overline{\mathtt{X}}_1 \subset \mathtt{X}_1$ and $\overline{\mathtt{X}}_2\subset \mathtt{X}_2$ of cardinality $\#\overline{\mathtt{X}}_1 = \#\overline{\mathtt{X}}_2 = r$ such that the matrix $h(\overline{\mathtt{X}}_1, \overline{\mathtt{X}}_2)\in \R^{r \times r}$ is nonsingular, and rank-$r$ \emph{interpolation cores}
\[
\mG_1(\bx_1) \equiv \bphi_1(\bx_1) \mB_1 \in \mathbb{R}^{1\times r} \quad  {\rm and} \quad \mG_2(\bx_2) \equiv \mB_2 \bphi_2(\bx_2)^\top \in \mathbb{R}^{r\times 1}, 
\] 
with $\mB_1 \in \mathbb{R}^{n_1 \times r}$ and $\mB_2\in \mathbb{R}^{r \times n_2}$, to approximate $h(\bx_1, \bx_2)$ by interpolation. 
The interpolation cores satisfy the property that $\mG_1(\overline{\mathtt{X}}_1)$ and $\mG_2(\overline{\mathtt{X}}_2)$ are identity matrices.
This yields interpolated approximations to  $h(\bx_1, \bx_2) $, for example, 
\begin{align*}
h(\bx_1, \bx_2)  \approx \mG_1(\bx_1) h(\overline{\mathtt{X}}_1, \bx_2) \quad {\rm and} \quad h(\bx_1, \bx_2)  \approx \mG_1(\bx_1) h(\overline{\mathtt{X}}_1, \overline{\mathtt{X}}_2) \mG_2(\bx_2).
\end{align*}
This way, the goal becomes identifying the optimal point sets $(\overline{\mathtt{X}}_1, \overline{\mathtt{X}}_2)$ and the cores $(\mG_1, \mG_2)$ that minimise the interpolated rank-$r$ factorisation error
\begin{equation}\label{eq:h2d-cross-problem}
\lpnormx{h(\bx_1, \bx_2) -  \mG_1(\bx_1) h(\overline{\mathtt{X}}_1, \overline{\mathtt{X}}_2) \mG_2(\bx_2) }{2} .
\end{equation}
In practice, an \emph{alternating direction} strategy can be employed to solve the above nonlinear minimisation problem via a sequence of subproblems at a lower computational cost compared to that of the full matrix decomposition induced by~\eqref{eq:h2d-discrete-norm}.
For example, we start from some initial guess of $\mB_2$ and $\overline{\mathtt{X}}_2$ to solve for $\mB_1$ and $\overline{\mathtt{X}}_1$ via  the minimisation problem
\begin{equation}\label{eq:h2d-cross-sub-1}
\mB_1, \overline{\mathtt{X}}_1 = \argmin_{\mB_1', \overline{\mathtt{X}}\vphantom{\mC}'_1 } \lpnormx{ h(\bx_1^{}, \overline{\mathtt{X}}_2^{}) \mB_2^{} \bphi_2(\bx_2^{})^\top  -  \bphi_1(\bx_1^{}) \mB_1'  h(\overline{\mathtt{X}}\vphantom{\mB}'_1, \overline{\mathtt{X}}_2^{}) \mB_2^{} \bphi_2(\bx_2^{})^\top }{2},
\end{equation}
then we use the updated $\mB_1$ and $\overline{\mathtt{X}}_1$ to renew $\mB_2$ and $\overline{\mathtt{X}}_2$ via
\begin{equation*}\label{eq:h2d-cross-sub-2}
\mB_2, \overline{\mathtt{X}}_2 = \argmin_{\mB_2', \overline{\mathtt{X}}\vphantom{\mB}'_2 } \lpnormx{ \bphi_1(\bx_1^{}) \mB_1^{}   h(\overline{\mathtt{X}}_1^{}, \bx_2^{})  -  \bphi_1(\bx_1^{}) \mB_1^{}  h(\overline{\mathtt{X}}^{}_1, \overline{\mathtt{X}}\vphantom{\mB}'_2) \mB_2' \bphi_2(\bx_2^{})^\top }{2},
\end{equation*}
and repeat until convergence.
Given the collocation points $\overline{\mathtt{X}}_1$ and $\overline{\mathtt{X}}_2$, the coefficient matrices $\mB_1^{}$ and $\mB_2^{}$ satisfy a simple quadratic optimisation, and can be computed from
\begin{align}\label{eq:h2d-cross}
\mB_1 h(\overline{\mathtt{X}}_1, \overline{\mathtt{X}}_2)= h(\mathtt{X}_1, \overline{\mathtt{X}}_2)  \quad {\rm and} \quad h(\overline{\mathtt{X}}_1, \overline{\mathtt{X}}_2) \mB_2  = h(\overline{\mathtt{X}}_1, \mathtt{X}_2),
\end{align}
respectively. Solving \eqref{eq:h2d-cross} only requires $(n_1+n_2-r)r$ evaluations of $h(\bx_1,\bx_2)$.

In~\eqref{eq:h2d-cross}, one needs to find the interpolation point sets $\overline{\mathtt{X}}_1$ and $\overline{\mathtt{X}}_2$ so that the resulting interpolation operator is an optimal approximation to the projection operator that spans the same linear subspace. However, finding the optimal interpolation point sets is an NP-hard problem.
In practice, accurate quasi-optimal solutions can be obtained by greedy algorithms such as the (discrete) empirical interpolation~\cite{barrault2004empirical,chaturantabut2010nonlinear} or the maximum volume (\emph{MaxVol})~\cite{goreinov2010find,goreinov1997theory,goreinov1997pseudo} methods.
Here we outline the procedure of the \emph{MaxVol} algorithm~\cite{goreinov2010find} for solving~\eqref{eq:h2d-cross-sub-1}, which can be equivalently expressed as the problem of searching for an index set $\mathtt{I}\subset \{1,\ldots,n\}$ of cardinality $r$ such that the norm of $\mB = \mH \overline{\mH}\vphantom{\mH}^{-1} \in\R^{n \times r}$ is minimized.
Here $\overline{\mH} = \mH[\mathtt{I},:]\in\R^{r\times r}$ is the submatrix of a given matrix $\mH$ in the MATLAB notation.
For example, one can have $\mH = h(\mathtt{X}_1, \overline{\mathtt{X}}_2)$, and then the interpolation point set $\overline{\mathtt{X}}_1$ is given by $\mathtt{I}$ and the coefficient matrix is set by $\mB_1 = \mB$. 

\begin{algorithm}[htb]
\caption{MaxVol}
\label{alg:maxvol}
\centering
\begin{algorithmic}[1]
 \State Choose an initial set $\mathtt{I}$ and a stopping threshold $\delta>0$.
 \While{$\max_{i,j} |\mB[i,j]| > 1+\delta$}
   \State Let $i_\star, j_\star = \arg\max_{i,j} |\mB[i,j]|$.
   \State Replace the index $\mathtt{I}[j_\star]$ in the set by $i_\star$.
   \State Recompute $\overline{\mH} = \mH[\mathtt{I},:]$ and $\mB = \mH \overline{\mH}\vphantom{\mH}^{-1}$.
 \EndWhile
\end{algorithmic}
\end{algorithm}

Given an initial index set, which can be chosen as the $r$ dominant pivots from Gaussian elimination, \emph{MaxVol} proceeds as Algorithm~\ref{alg:maxvol}.
Note that $\mB[\mathtt{I},:] \in \mathbb{R}^{r \times r}$ is an identity matrix by construction. \emph{MaxVol} ensures that no other row is more ``important'' by searching for a \emph{dominant} submatrix $\overline{\mH}$ such that $|\mB[i,j]|\le 1+\delta$, which is a proxy to the  \emph{maximum volume} submatrix $\overline{\mH}_\star$ such that $|\det(\overline{\mH}_\star)| = \max_{\mathtt{I}} |\det(\mH[\mathtt{I},:])|$.
The update of $\mB$ can be computed efficiently via the Sherman--Morrison--Woodbury formula~\cite{goreinov2010find} with a total cost of $\mathcal{O}(nr^2)$ per iteration.

For the numerical stability it is beneficial to compute the thin generalised QR factorization $\mH \mR = h(\mathtt{X}_1, \overline{\mathtt{X}}_2) $, where the matrix $\mH$ has $\mM_1$-orthonormal columns. This way, the set of functions $\bphi_1(\bx_1) \mH$ forms a $\lambda_1$-orthogonal basis. 
The factorisation $\mH \mR = h(\mathtt{X}_1, \overline{\mathtt{X}}_2) $ can be obtained by the thin QR factorization $\mQ \mR = \chol_1^\top h(\mathtt{X}_1, \overline{\mathtt{X}}_2) $ and $\mH =  \chol_1^{-\top} \mQ$.
Then, one can apply \emph{MaxVol} to  $\mH$, which is the evaluation of $\bphi_1(\bx_1) \mH$ at $\mathtt{X}_1$, 
to select the index set $\mathtt{I}$, and thus the interpolation points $\overline{\mathtt{X}}_1 \subset \mathtt{X}_1$. 
We have
\(
\mH[\mathtt{I},:] = h(\overline{\mathtt{X}}_1, \overline{\mathtt{X}}_2) \mR^{-1},
\)
which yields
\(
\mB = \mH \mH[\mathtt{I},:]^{-1} =  h(\mathtt{X}_1, \overline{\mathtt{X}}_2) h(\overline{\mathtt{X}}_1, \overline{\mathtt{X}}_2)^{-1} .
\)
Thus, we can set the coefficient matrix as $\mB_1 = \mB$ and define the interpolation core $\mG_1(\bx_1) = \bphi_1(\bx_1) \mB_1$ such that $\mG_1(\mathtt{X}_1)$ is an identity matrix. 

We can obtain $(\mB_1, \overline{\mathtt{X}}_1)$ and $(\mB_2, \overline{\mathtt{X}}_2)$ by applying \emph{MaxVol} within alternating iterations. Then, we can set $\mA_1 = \mB_1$ and $\mA_2 = h(\overline{\mathtt{X}}_1, \overline{\mathtt{X}}_2) \mB_2$ to recover the factorisation in the form of~\eqref{eq:h2d-factor}.

\subsubsection{Multi-dimensional case}

The TT-cross algorithm~\cite{oseledets2010tt} recursively extends~\eqref{eq:h2d-cross} to $d > 2$.
In the first step, we assume that a reduced point set $\overline{\mathtt{X}}_{>1} = \{(\bx_2^{(\alpha_1)},\ldots,\bx_d^{(\alpha_1)})\}$ of $r_1$ points is given. We can for example draw it from some tractable reference measure.
We compute an analogue of the first equation in~\eqref{eq:h2d-cross}
\[
\mA_1 h(\overline{\mathtt{X}}_1,\overline{\mathtt{X}}_{>1}) = h(\mathtt{X}_1,\overline{\mathtt{X}}_{>1})\in\R^{n_1 \times r_1},
\]
where 
$
 h(\mathtt{X}_1, \overline{\mathtt{X}}_{>1}) = [h(\bx_1^{(i_1)}, \ \bx_{2}^{(\alpha_{1})},\ldots,\bx_{d}^{(\alpha_{1})} ) ]
$
is a matrix filled with the function $h(\bx)$ evaluated at the ``reduced'' set of points $\mathtt{X}_1 \times \overline{\mathtt{X}}_{>1}$.
Now we apply \emph{MaxVol} to compute reduced subsets $\mathtt{I}_{1} \subset \{1,\ldots,n_1\}$ and $\overline{\mathtt{X}}_{<2} = \mathtt{X}_1(\mathtt{I}_{1}) \subset \mathtt{X}_1$.
Similarly to the matrix $\mB_1$ in the two dimensional case, we let the actual TT core be the ``stabilized'' matrix $\mA_1 = \mH_1 \mH_1[\mathtt{I}_1,:]^{-1}$, where $\mH_1 \mR_1 = h(\mathtt{X}_1,\overline{\mathtt{X}}_{>1})$ is the generalised QR decomposition.

In the $k$-th step, we assume reduced point sets
$\overline{\mathtt{X}}_{<k} = \{\bx_1^{(\alpha_{k-1})},\ldots,\bx_{k-1}^{(\alpha_{k-1})}\}$ and $\overline{\mathtt{X}}_{>k} = \{\bx_{k+1}^{(\alpha_{k})},\ldots,\bx_{d}^{(\alpha_{k})}\}$ are given.
We can compute a third order tensor
\begin{equation}\label{eq:hk-cross}
 \tH_k:= \left[h(\bx_1^{(\alpha_{k-1})},\ldots,\bx_{k-1}^{(\alpha_{k-1})}, \ \bx_k^{(i_k)}, \ \bx_{k+1}^{(\alpha_{k})},\ldots,\bx_{d}^{(\alpha_{k})} )\right] \in \R^{r_{k-1} \times n_k \times r_k},
\end{equation}
which consists of evaluations of $h(\bx)$ at the Cartesian union of the sets $\overline{\mathtt{X}}_{<k} \times \mathtt{X}_k \times \overline{\mathtt{X}}_{>k}$.
We let $\overline{\mathtt{X}}_{<1} = \overline{\mathtt{X}}_{>d} = \emptyset$ to enable the notation for all $k$.
We can \emph{unfold} $\tH_k$ into matrices of the form
\begin{align}\label{eq:hk-unfold}
 \mH_k^{(\rm L)} \in \R^{(r_{k-1}n_k) \times r_k}, & \quad {\rm and} \quad \mH_k^{(\rm R)} \in \R^{r_{k-1} \times (n_k r_k)}, 
\end{align}
such that
\begin{align*}
\tH_k[\alpha_{k-1},i_k,\alpha_k] & =  \mH_k^{(\rm L)}[\alpha_{k-1}+(i_k-1) r_{k-1}, \alpha_k]
  = \mH_k^{(\rm R)}[\alpha_{k-1}, i_k+ (\alpha_k-1)n_k].
\end{align*}
The union of the indices $\alpha_{k-1}$ and $i_k$ corresponds to the union of the point sets $\mathtt{X}_{\le k}:=\overline{\mathtt{X}}_{<k} \times \mathtt{X}_k$.
Therefore, we can apply \emph{MaxVol} to $\mH_k^{(\rm L)}$ (or a generalised QR factor thereof) to obtain a discrete set $\mathtt{I}_k \subset \{1,\ldots,r_{k-1}n_k\}$,
and take a subsample of $\mathtt{X}_{\le k}$ for the next recursion step, $\overline{\mathtt{X}}_{<k+1} = \mathtt{X}_{\le k}(\mathtt{I}_k)$.
Similarly for the $k$th TT core we define
\begin{align}
\mB_k^{(\rm L)} & = \mH_k^{(\rm L)} \mH_k^{(\rm L)}[\mathtt{I}_k,:]^{-1}, \label{eq:Bk-cross} \\
\tA_k[\alpha_{k-1},i_k,\alpha_k] & = \mB_k^{(\rm L)}[\alpha_{k-1}+(i_k-1) r_{k-1}, \alpha_k]. \label{eq:Ak-cross}
\end{align}

If the function $h(\bx)$ admits an exact TT decomposition, and the initial point sets were chosen such that all $\mH_k^{(\rm L)}$ are full-rank, the recursion defined above reconstructs the decomposition exactly.
However, in practice the initial point sets can be a poor interpolation sets. In this case we can \emph{refine} them by carrying out several iterations. Having computed $\tA_d$, we reverse the recursion and iterate backwards, computing discrete sets $\mathtt{J}_k \subset \{1,\ldots,n_kr_k\}$ via \emph{MaxVol}  applied to $(\mH_k^{(\rm R)})^\top$, and setting $\overline{\mathtt{X}}_{>k-1} = \mathtt{X}_{\ge k}(\mathtt{J}_k)$, where $\mathtt{X}_{\ge k} = \mathtt{X}_{k} \times \overline{\mathtt{X}}_{>k}$.

The second key ingredient is the adaptation of TT ranks.
The TT ranks can be easily reduced. For example, it is sufficient to replace the generalised QR factorization of $\mH_k^{(\rm L)}$ or $(\mH_k^{(\rm R)})^\top$ by a generalised SVD, where the singular values below the desired threshold are truncated.
To increase the TT ranks, we can apply \emph{oversampling}. Using the forward iteration (with $k$ increasing) as an example, we can compute the tensor $\tH_k \in \R^{r_{k-1} \times n_k \times (r_k+\rho_k)}$ on the enriched point set $\overline{\mathtt{X}}_{<k} \times \mathtt{X}_k \times (\overline{\mathtt{X}}_{>k} \cup \widetilde{\mathtt{X}}_{>k})$, where $\widetilde{\mathtt{X}}_{>k} = \{(\bx_{k+1}^{(\alpha_{k})},\ldots,\bx_{d}^{(\alpha_{k})})\}_{\alpha_k=1}^{\rho_k}$ are auxiliary points.
These auxiliary points can be sampled at random~\cite{Os-mvk2-2011}, or more accurately, from a surrogate of the error~\cite{dolgov2014alternating}.
In the latter case, we carry out a second TT-cross to approximate the error $h(\bx) - \tilde h(\bx)$ by a TT decomposition with TT ranks $\rho_1,\ldots,\rho_{d-1}$, and take the \emph{MaxVol} points of the error as $\widetilde{\mathtt{X}}_{>k}$.
This \emph{enrichment} of the solution with error or residual information has proven to accelerate the convergence drastically even for small expansion ranks $\rho_k$ when applied to solving linear systems~\cite{dolgov2014alternating}.
The pseudocode of the TT-cross is provided in Algorithm~\ref{alg:tt-cross}.

\begin{algorithm}[htb]
\centering
\caption{TT-cross}
\label{alg:tt-cross}
\begin{algorithmic}[1]
 \State Choose initial sets $\overline{\mathtt{X}}_{<k}$, $k=2,\ldots,d$, stopping threshold $\delta>0$, enrichment TT ranks $\rho_k$.
 \While{first iteration or $\|\tilde h(\bx) - \tilde h_{\rm prev}(\bx) \|>\delta \|\tilde h(\bx)\|$}
   \For{$k=d,d-1,\ldots,2$} \Comment{backward iteration}
     \State Sample $\tH_k$ as shown in~\eqref{eq:hk-cross}, optionally expanding $\overline{\mathtt{X}}_{<k}$ to $\overline{\mathtt{X}}_{<k} \cup \widetilde{\mathtt{X}}_{<k}$.
     \State Compute $\mathtt{J}_k$ from \emph{MaxVol} on $(\mH_k^{(\rm R)})^\top$ or its SVD factor, let $\overline{\mathtt{X}}_{>k-1} = \mathtt{X}_{\ge k}(\mathtt{J}_k)$.
   \EndFor
   \For{$k=1,2,\ldots,d-1$} \Comment{forward iteration}
     \State Sample $\tH_k$ as shown in~\eqref{eq:hk-cross}, optionally expanding $\overline{\mathtt{X}}_{>k}$ to $\overline{\mathtt{X}}_{>k} \cup \widetilde{\mathtt{X}}_{>k}$.
     \State Compute $\mathtt{I}_k$ from \emph{MaxVol} on $\mH_k^{(\rm L)}$ or its SVD factor, let $\overline{\mathtt{X}}_{<k+1} = \mathtt{X}_{\le k}(\mathtt{I}_k)$.
     \State Reconstruct TT cores as shown in~\eqref{eq:Bk-cross}--\eqref{eq:Ak-cross}.
   \EndFor
   \State Sample the last TT core $\tA_d = \tH_d$ as shown in~\eqref{eq:hk-cross}.
 \EndWhile
\end{algorithmic}
\end{algorithm}

The tensor in~\eqref{eq:hk-cross} suggests that the TT-cross requires $\sum_{k=1}^{d} r_{k-1} n_k r_k$ evaluations of $h(\bx)$ per iteration, which is proportional to the number of unknowns in the TT cores.
To enhance the robustness (at the expense of a larger number of function evaluations), one may oversample $\mathtt{X}_k$ beyond $n_k$ basis functions, and use the \emph{rectangular MaxVol} algorithm~\cite{mo-rectmaxvol-2018} to oversample $\mathtt{I}_k,\mathtt{J}_{k+1}$ beyond $r_k$ indices. In this case, the matrix inverse in~\eqref{eq:Bk-cross} is replaced by a pseudoinverse.
For our DIRT framework, the standard \emph{MaxVol} equipped with the error enrichment is sufficiently robust to factorise the ratio functions, so we proceed with Algorithm~\ref{alg:tt-cross}.

\newcommand{\tmpnew}[2]{\sum_{{#1}_{#2-1}=1}^{r_{#2-1}}\!\!\!\! \mG_{<#2}^{({#1}_{#2-1})}\!(\bx_{<#2}) \sqbasis_{#2}^{({#1}_{#2-1},{\ell}_{#2})}\!(\bx_{#2})}

\subsection{Appendix B: proof of Proposition \ref{prop:sirt_recur}} \label{appen:prop:sirt_recur}
Recall the marginal function 
\begin{align}
\hat{\ptar}_{ \leq k}(\bx_{\leq k}) 
= \gamma  \prod_{i = k+1}^d  \lambda_i(\mathcal{X}_i) + \sum_{\ell_{k}=1}^{r_{k}} \Big(\mG_{1}(\bx_{1}) \cdots \mG_{\km}(\bx_{\km}) \, \sqbasis_{k}^{(\,:\,,\ell_{k})}(\bx_{k}) \Big)^2,
\end{align}
where $\sqbasis_{k}(\bx_{k}) : \mathcal{X}_{k} \mapsto \R^{ r_{k-1} \times r_{k} }$ is given by~\eqref{eq:L_k} with a coefficient tensor $\tB_k \in \R^{r_{k-1}\times n_k \times r_k}$.
The next marginal function $\hat{\ptar}_{ <k}(\bx_{<k})$ is defined by
\begin{align}
& \hspace{-12pt}\hat{\ptar}_{ < k}(\bx_{< k}) \nonumber \\
& = \int_{\mathcal{X}_k} \bigg(  \gamma \prod_{i = k+1}^d \lambda_i(\mathcal{X}_i)  + \sum_{\ell_{k}=1}^{r_{k}} \Big(\mG_{1}(\bx_{1}) \cdots \mG_{\km}(\bx_{\km}) \, \sqbasis_{k}^{(\,:\,,\ell_{k})}(\bx_{k}) \Big)^2 \bigg) \lambda_k(\bx_k)\,d\bx_k \nonumber \\
& =  \gamma \prod_{i = k}^d \lambda_i(\mathcal{X}_i) + \sum_{\ell_{k}=1}^{r_{k}} \int_{\mathcal{X}_k} \!\!\Big(\mG_{1}(\bx_{1}) \cdots \mG_{\km}(\bx_{\km}) \, \sqbasis_{k}^{(\,:\,,\ell_{k})}(\bx_{k}) \Big)^2 \,\lambda_k(\bx_k)\,d\bx_k. \label{eq:sirt_proof_0}
\end{align}
The second term of \eqref{eq:sirt_proof_0} can be expressed as
\begin{align}
& \sum_{\ell_{k}=1}^{r_{k}} \int_{\mathcal{X}_k} \!\!\Big(\mG_{1}(\bx_{1}) \cdots \mG_{\km}(\bx_{\km}) \, \sqbasis_{k}^{(\,:\,,\ell_{k})}(\bx_{k}) \Big)^2 \,\lambda_k(\bx_k)\,d\bx_k \nonumber \\
& = \sum_{\ell_{k}=1}^{r_{k}} \int_{\mathcal{X}_k} \tmpnew{\alpha}{k} \tmpnew{\beta}{k} \lambda_k(\bx_k)d\bx_k \nonumber\\
& = \sum_{\alpha_{k-1}=1}^{r_{k-1}} \sum_{\beta_{k-1}=1}^{r_{k-1}} \mG_{< k}^{(\alpha_{k-1})}(\bx_{< k})\,\mG_{< k}^{(\beta_{k-1})}(\bx_{< k}) \, \overline{\mM}_{k}[\alpha_{k-1}, \beta_{k-1}], \label{eq:sirt_proof_1}
\end{align}
where the symmetric matrix $\overline{\mM}_{k} \in \R^{ r_{k-1} \times r_{k-1} }$ is given by
\begin{align}\label{eq:mass_funcL}
\overline{\mM}_{k}[\alpha_{k-1}, \beta_{k-1}] = \sum_{\ell_{k}=1}^{r_{k}} \int_{\mathcal{X}_k} \sqbasis_{k}^{(\alpha_{k-1},\ell_{k})}(\bx_{k})\, \sqbasis_{k}^{(\beta_{k-1},\ell_{k})}(\bx_{k})\,\lambda_k(\bx_k)\,d\bx_k,
\end{align}
and $\mG_{< k}$ is defined in~\eqref{eq:G_<k}.
Plugging the expression~\eqref{eq:L_k} of $\sqbasis_k(\bx_k)$ into \eqref{eq:mass_funcL}, we obtain
\begin{align}
& \overline{\mM}_{k}[\alpha_{k-1}, \beta_{k-1}] \nonumber \\
& = \sum_{\ell_{k}=1}^{r_{k}} \int_{\mathcal{X}_k}\!\! \Big(\sum_{i =1}^{n_k}\phi_{k}^{(i)}(\bx_k) \, \tB_k [\alpha_{k-1}, i, \ell_k]\Big)\Big(\sum_{j =1}^{n_k}\phi_{k}^{(j)}(\bx_k) \, \tB_k [\beta_{k-1}, j, \ell_k]\Big)\,\lambda_k(\bx_k)\,d\bx_k \nonumber \\
& = \sum_{\ell_{k}=1}^{r_{k}} \sum_{i =1}^{n_k} \sum_{j =1}^{n_k} \tB_k [\alpha_{k-1}, i, \ell_k] \, \tB_k [\beta_{k-1}, j, \ell_k] \int_{\mathcal{X}_k}\!\!  \phi_{k}^{(i)}(\bx_k) \phi_{k}^{(j)}(\bx_k)\,\lambda_k(\bx_k)\,d\bx_k \nonumber \\
& = \sum_{\ell_{k}=1}^{r_{k}} \sum_{i =1}^{n_k} \sum_{j =1}^{n_k} \tB_k [\alpha_{k-1}, i, \ell_k] \, \tB_k [\beta_{k-1}, j, \ell_k] \, \mM_k[i,j], \label{eq:mass_funcL2}
\end{align}
where $\mM_k^{} \in \R^{n_k \times n_k} $ is the symmetric positive definite mass matrix defined in~\eqref{eq:mass-k}. 
Given the Cholesky decomposition $\chol_k \chol_k^\top = \mM_k$, we have
\[
\mM_k[i,j] = \sum_{\tau = 1}^{n_k} \chol_k[i, \tau]\, \chol_k[j, \tau].
\] 
Substituting the above identity into \eqref{eq:mass_funcL2}, we have
\begin{align}
& \hspace{-12pt} \overline{\mM}_{k}[\alpha_{k-1}, \beta_{k-1}] \nonumber \\
& = \sum_{\ell_{k}=1}^{r_{k}} \sum_{i =1}^{n_k} \sum_{j =1}^{n_k} \sum_{\tau = 1}^{n_k} \tB_k [\alpha_{k-1}, i, \ell_k] \, \tB_k [\beta_{k-1}, j, \ell_k] \, \chol_k[i, \tau]\, \chol_k[j, \tau] \nonumber \\
& = \sum_{\ell_{k}=1}^{r_{k}} \sum_{\tau = 1}^{n_k} \Big(\sum_{i =1}^{n_k} \tB_k [\alpha_{k-1}, i, \ell_k]\, \chol_k[i, \tau]\Big)\Big(\sum_{j =1}^{n_k} \tB_k [\beta_{k-1}, j, \ell_k]\, \chol_k[j, \tau]\Big).
\end{align}
Denoting $\tC_k[\alpha_{k-1}, \tau, \ell_{k}] = \sum_{i = 1}^{n_k} \tB_k[\alpha_{k-1}, i, \ell_{k}] \, \chol_k[i, \tau]$ and unfolding $\tC_k$ along the first coordinate similarly to~\eqref{eq:hk-unfold} to obtain a matrix $\mC_k^{(\rm R)} \in \R^{r_{k-1} \times (n_k  r_k)} $, we have
\begin{align*}
& \hspace{-12pt} \overline{\mM}_{k}[\alpha_{k-1}, \beta_{k-1}] \\
& = \sum_{\ell_{k}=1}^{r_{k}} \sum_{\tau = 1}^{n_k} \tC_k[\alpha_{k-1}, \tau, \ell_{k}]\, \tC_k[\beta_{k-1}, \tau, \ell_{k}] = \sum_{\iota = 1}^{n_kr_k}\mC_k^{(\rm R)}[\alpha_{k-1}, \iota] \, \mC_k^{(\rm R)}[\beta_{k-1}, \iota].
\end{align*}
Equivalently, we have $\overline{\mM}_{k} = \mC_k^{(\rm R)} \big(\mC_k^{(\rm R)}\big)^\top$. This way, computing the thin QR decomposition
\begin{align*}
\mQ_k \mR_k^{} = \big( \mC_k^{(\rm R)} \big)^\top,
\end{align*}
we obtain the Cholesky decomposition $\mR_k^\top \mR_k^{} = \overline{\mM}_{k}$ where $\mR_k \in \R^{r_{k-1} \times r_{k-1}}$ is upper-triangular. 
Substituting the identity
\[
\overline{\mM}_{k}[\alpha_{k-1}, \beta_{k-1}] = \sum_{\ell_{k-1} = 1}^{r_{k-1}} \mR_k[\ell_{k-1}, \alpha_{k-1}]\, \mR_k[\ell_{k-1}, \beta_{k-1}],
\] 
into \eqref{eq:sirt_proof_1}, 
the second term of \eqref{eq:sirt_proof_0} is defined by
\begin{align}
& \hspace{-12pt} \sum_{\ell_{k}=1}^{r_{k}} \int_{\mathcal{X}_k} \!\!\Big(\mG_{1}(\bx_{1}) \cdots \mG_{\km}(\bx_{\km}) \, \sqbasis_{k}^{(\,:\,,\ell_{k})}(\bx_{k}) \Big)^2 \,\lambda_k(\bx_k)\,d\bx_k \nonumber \\
& = \sum_{\ell_{k-1} = 1}^{r_{k-1}}  \sum_{\alpha_{k-1}=1}^{r_{k-1}} \sum_{\beta_{k-1}=1}^{r_{k-1}} \mG_{< k}^{(\alpha_{k-1})}(\bx_{< k})\,\mG_{< k}^{(\beta_{k-1})}(\bx_{< k}) \, \mR_k[\ell_{k-1}, \alpha_{k-1}]\, \mR_k[\ell_{k-1}, \beta_{k-1}] \nonumber \\
& = \sum_{\ell_{k-1} = 1}^{r_{k-1}} \Big( \sum_{\alpha_{k-1}=1}^{r_{k-1}} \mG_{< k}^{(\alpha_{k-1})}(\bx_{< k})\, \mR_k[\ell_{k-1}, \alpha_{k-1}]\Big)^2 \nonumber \\
& = \sum_{\ell_{k-1} = 1}^{r_{k-1}} \Big(\mG_{1}(\bx_{1}) \cdots \mG_{k-2}(\bx_{k-2}) \sum_{\alpha_{k-1}=1}^{r_{k-1}}  \mG_{k-1}^{(\,:\,,\, \alpha_{k-1})}(\bx_{k-1}) \, \mR_k[\ell_{k-1}, \alpha_{k-1}] \Big)^2,
\end{align}
where the last line follows from the identity in~\eqref{eq:G_<k}.
Following the recursive definition of $\sqbasis_{k}(\bx_{k})$ in Proposition \ref{prop:sirt_recur}, we have 
\begin{equation}\label{eq:ftt_group2}
\sqbasis_{k-1}^{(\,:\,,\ell_{k-1})}(\bx_{k-1}) = \sum_{\alpha_{k-1}=1}^{r_{k-1}}  \mG_{k-1}^{(\,:\,,\, \alpha_{k-1})}(\bx_{k-1}) \, \mR_k[\ell_{k-1}, \alpha_{k-1}].
\end{equation}
Substituting the definition of the TT core 
\[
\ttcorej{k-1}{\bx}{\mG} = \sum_{i=1}^{n_{k-1}}\phi_{k-1}^{(i)}(\bx_{k-1}) \tA_{k-1} [\alpha_{k-2}, i, \alpha_{k-1}]
\]
into \eqref{eq:ftt_group2}, we have
\[
\sqbasis_{k-1}^{(\alpha_{k-2},\ell_{k-1})}(\bx_{k-1}) =  \sum_{i =1}^{n_{k-1}}\phi_{k-1}^{(i)}(\bx_{k-1}) \Big( \sum_{\alpha_{k-1}=1}^{r_{k-1}}  \tA_{k-1} [\alpha_{k-2}, i, \alpha_{k-1}] \, \mR_k[\ell_{k-1}, \alpha_{k-1}] \Big).
\]
Thus, defining the coefficient tensor 
\[
\tB_{k-1}[\alpha_{k-2},i, \ell_{k-1}] =  \sum_{\alpha_{k-1} = 1}^{r_{k-1}} \tA_{k-1}[\alpha_{k-2},i, \alpha_{k-1}]\, \mR_k[\ell_{k-1},\alpha_{k-1}],
\]
we obtain the result of Proposition \ref{prop:sirt_recur}:
\[
\sqbasis_{k-1}^{(\alpha_{k-2},\ell_{k-1} )}(\bx_{k-1}) = \sum_{i = 1}^{n_{k-1}} \phi_{k-1}^{(i)}(\bx_{k-1})\, \tB_{k-1}[\alpha_{k-2},i, \ell_{k-1}].
\]

By setting index $k = 1$ and repeating the above procedure, we can obtain the normalising constant $\hat{z} = \int_{\mathcal{X}_1} \hat{\ptar}_{ \leq 1}(\bx_{1}) \lambda_1(\bx_1) d\bx_1 = \gamma \prod_{i = 1}^d \lambda_i(\mathcal{X}_i) + \mR_1^2$, where $\mR_1 \in \mathbb{R}$ as the unfolded $\tC_1$ along the first coordinate is a row vector $\mC_1^{(\rm R)} \in \R^{1 \times (n_1 r_1)} $.